\documentclass[twoside,11pt]{article}

\usepackage{blindtext}

\usepackage[preprint]{jmlr2e}

\usepackage{soul}           %
\usepackage{mathrsfs}       %
\usepackage{dsfont}
\usepackage{amsmath}
\usepackage{bbm}
\usepackage{amsfonts}
\usepackage{amssymb}
\usepackage{subcaption}
\usepackage{bm,array}
\usepackage{hyperref}
\usepackage{numprint}

\usepackage{multirow}

\newcommand{\N}{\mathbb{N}}
\newcommand{\R}{\mathbb{R}}
\renewcommand{\P}{\mathbb{P}}

\usepackage[normalem]{ulem}

\newcommand{\measSpace}{\Omega}

\newcommand{\probaSet}{\mathcal P(\measSpace)}

\newcommand{\PRD}{\text{PRD}}
\newcommand{\dPRD}{\partial\PRD}
\newcommand{\supp}{\mathrm{supp}}
\newcommand{\cosupp}{\mathrm{cosupp}}
\newcommand{\fpr}{\mathrm{fpr}}
\newcommand{\fnr}{\mathrm{fnr}}
\newcommand{\F}{\mathscr{F}}
\newcommand{\E}{\mathbb{E}}
\newcommand{\1}{\mathds{1}}
\newcommand{\BKNN}{B_{\text{kNN}}}
\newcommand{\X}{{\mathcal X}}
\newcommand{\Y}{{\mathcal Y}}

\newcommand{\Var}{\mathrm{Var}}

\newcommand{\TV}[1]{\Vert #1 \Vert_{\mathrm{TV}}}

\DeclareMathOperator*{\argmin}{arg\,min}

\usepackage[dvipsnames]{xcolor}
\definecolor{color_IPR}{RGB}{23,116,180}
\definecolor{color_KNN}{RGB}{255,136,32}
\definecolor{color_PARZ}{RGB}{20,160,20}
\definecolor{color_KDE}{RGB}{20,160,20}
\definecolor{color_COV}{RGB}{214,27,28}
\definecolor{color_GT}{RGB}{110,24,189}

\newcommand{\knn}{\text{\normalfont\scshape \bfseries kNN}}
\newcommand{\ipr}{\text{\normalfont\scshape \bfseries iPR}}
\newcommand{\parzen}{\text{\normalfont\scshape \bfseries KDE}}
\newcommand{\kde}{\text{\normalfont\scshape \bfseries KDE}}
\newcommand{\cov}{\text{\normalfont\scshape \bfseries Cov}}

\newlength{\myheight} 
\newlength{\mywidth} %
\newlength{\mywidthhere} %

\usepackage{lastpage}

\ShortHeadings{A NEW PERSPECTIVE ON PRECISION AND RECALL FOR GENERATIVE MODELS}{}
\firstpageno{1}

\begin{document}
\title{A New Perspective on Precision and Recall for Generative Models}

\author{\name Benjamin Sykes \email benjamin.sykes@unicaen.fr\\
    \addr NORMANDIE UNIV, UNICAEN, ENSICAEN, CNRS, GREYC, 14000 CAEN, FRANCE\\
    \AND
    \name Loïc Simon \email loic.simon@ensicaen.fr\\
    \addr NORMANDIE UNIV, UNICAEN, ENSICAEN, CNRS, GREYC, 14000 CAEN, FRANCE\\
    \AND
    \name Julien Rabin \email julien.rabin@ensicaen.fr\\
    \addr NORMANDIE UNIV, UNICAEN, ENSICAEN, CNRS, GREYC, 14000 CAEN, FRANCE\\
    \AND
    \name Jalal Fadili \email jalal.fadili@ensicaen.fr\\
    \addr NORMANDIE UNIV, UNICAEN, ENSICAEN, CNRS, GREYC, 14000 CAEN, FRANCE\\
    }
\editor{None}

\maketitle

\begin{abstract}%
With the recent success of generative models in image and text, the question of their evaluation has recently gained a lot of attention.
While most methods from the state of the art rely on scalar metrics, the introduction of Precision and Recall (PR) for generative model has opened up a new avenue of research. 
The associated PR curve allows for a richer analysis, but their estimation poses several challenges.
In this paper, we present a new framework for estimating entire PR curves based on a binary classification standpoint. 
We conduct a thorough statistical analysis of the proposed estimates. As a byproduct, we obtain a minimax upper bound on the PR estimation risk.
We also show that our framework extends several landmark PR metrics of the literature which by design are restrained to the extreme values of the curve. 
Finally, we study the different behaviors of the curves obtained experimentally in various settings.
\end{abstract}%


\section{Introduction}
\label{sec:introduction}

\subsection{Problem Statement}
\label{sec:problem-statement}
In this paper, we consider metrics designed to evaluate the faithfulness of a generative model to the distribution it is trained to capture. The problem intrinsically consists in evaluating the “closeness” of a target distribution, hereafter denoted by $P$ and a generated one, denoted by $Q$. 
This problem in itself is a challenging one as generative models have recently been able to create ever more complex data in high dimension.
Additionally, the notion of quality in generative models is inherently subjective — unlike in supervised learning, there is no unique correct output or ground-truth for one generated data.
The notion of ``closeness'' stated above still remains to be precisely defined and several methods have been proposed to deal with this question, ranging from human feedback to the use of pre-trained encoder as a learned metric.
We will first briefly see how different methods have emerged with the will to tackle generative model evaluation and which refreshing point of view we present in this field.

\subsection{Context}
\label{sec:intro-SOTA}

In the early days of training deep generative models for still images, several scalar metrics were defined to objectively capture a notion of quality on generated images. While these metrics were appealing at first to compare two generative models, they were omitting a large quantity of information about the distributions under study.
To alleviate these drawbacks, \cite{sajjadi_AssessingGenerativeModels_2018} defined a new and formal framework to evaluate generative models. This framework allows to capture the quality of a generative model in terms of \textbf{fidelity}: \textit{``how close the generated data are from the target data?''} and \textbf{diversity}: \textit{``how much of the target space do the generated images cover?''}. 
In this framework, both information are captured through a Precision and Recall curve which is inherently different from those of binary classification. The intuitive description of the metric is described in Figure~\ref{fig:PR}.

\begin{figure}[!t]
    \centering
         \includegraphics[width=0.48\textwidth]{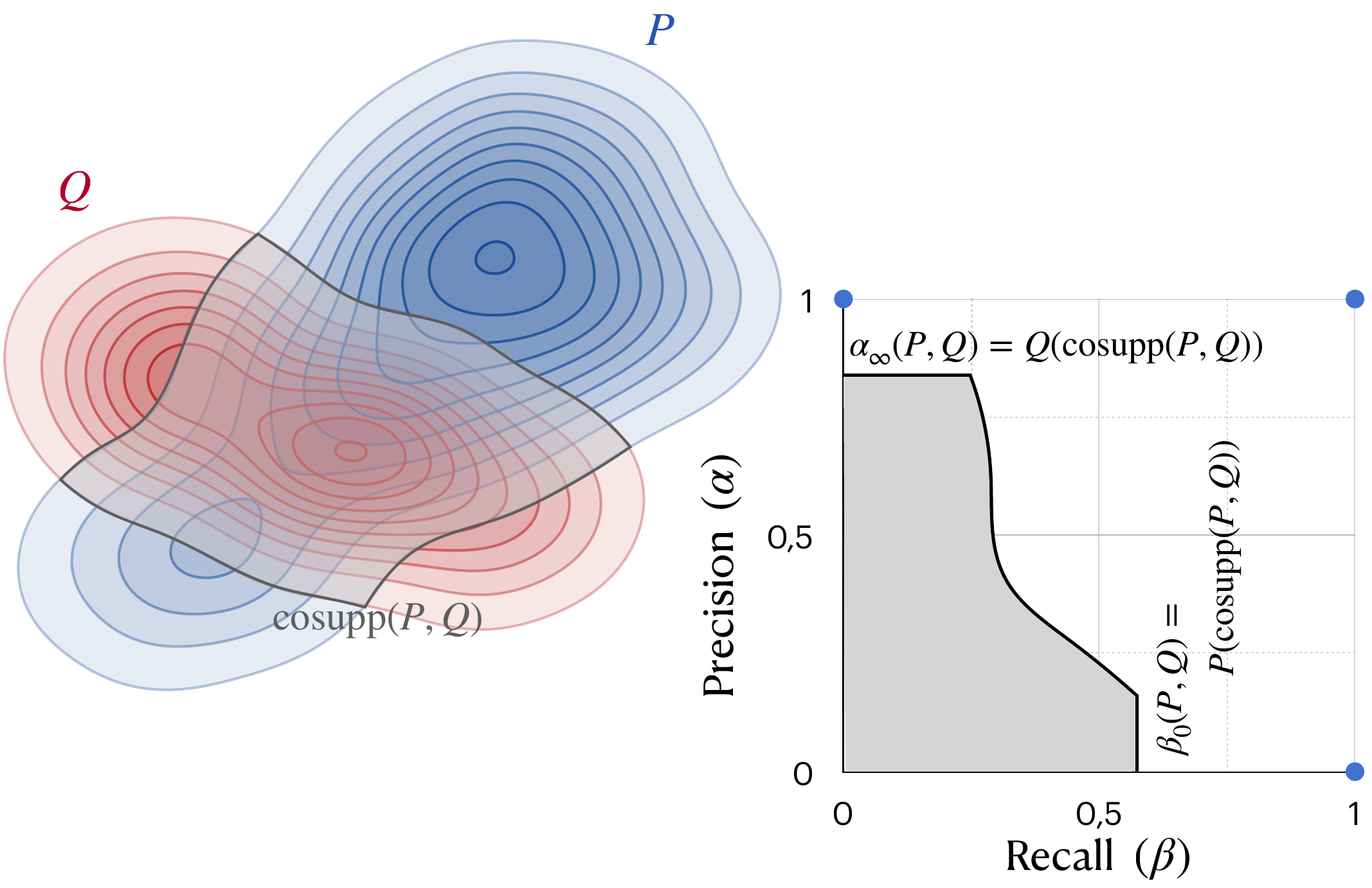}
    \caption{
    Left: two illustrative distributions $P$ and $Q$ (example borrowed from \citet{kynkaanniemi_ImprovedPrecisionRecall_2019}) ---
    Right: the PR-curve is the frontier of the shaded area composed of all admissible PR pairs $(\beta,\alpha)$. In essence, these pairs represent the mass of $P$ and $Q$ that one can recover by selecting a subset of the common support (gray area on the left). More precisely, by selecting regions of high likelihood of $P$, one trades Precision ($\alpha$) in favor of Recall ($\beta$).  The extreme values $\beta_0(P,Q)$ and  $\alpha_\infty(P,Q)$ embody the respective masses of the entire common support.
    }
    \label{fig:PR}
\end{figure}

While PR-curves are strongly theoretically grounded, practitioners have preferred scalar metrics of Precision and Recall over the entire curve. 
Whereas the initial formulation by \cite{sajjadi_AssessingGenerativeModels_2018} defined clusters based on samples using K-Means, \cite{kynkaanniemi_ImprovedPrecisionRecall_2019} defined scalar Precision and Recall metrics based on kNN support estimation.
\cite{naeem_ReliableFidelityDiversity_2020} proposed other metrics to evaluate fidelity and diversity, still based on kNN yet in a different manner. One key contribution was that these metrics were less sensitive to outliers. 
In \cite{khayatkhoei_EmergentAsymmetryPrecision_2023}, the authors proposed a new metric constructed in order to be robust in high dimension. This metric summarizes as the minimum between the two previous metrics.
Finally, \cite{park_ProbabilisticPrecisionRecall_2023} created novel Precision and Recall metrics estimating not simply the support of the distribution under study but estimating the probability density in a Parzen like framework.
A more detailed overview of the state of the art will be provided later on in Section~\ref{sec:link-literature}.

On the one hand, this effort is justified by practical arguments which entail simplicity as well as an intuitive understanding of the extreme values.
On the other hand, restricting to extreme values only allows for a very shallow comparison.
For instance, extreme Precision and Recall scores may both attain their saturation level of $1$ for entirely different distributions (as long as both distributions share the same support). 
On the contrary, Precision and Recall curves do capture the fundamental differences between two distributions in the transition zone between the two extreme values. 
Another concern with published PR metrics relates to their statistical consistency which is either provably lacking or at best only partially studied. 
A notable exception is \citep{pillutla_MAUVEScoresGenerative_2023}, in which a formal and theoretical framework is presented in the context of divergence frontiers introduced by \citep{djolonga_PrecisionRecallCurvesUsing_2020}. 
The authors propose several scalar summary metrics for such performance curves, and %
conduct a theoretical analysis to provide estimation error bounds for $f$-divergences. %
It is also worth mentioning \citep{kim_TopPRRobustSupport_2023} in which the authors propose a metric called topological Precision and Recall (topP\&R) along with its consistency analysis. However this latter is limited to extreme PR metrics and relies on strong assumptions (no tail in the distributions -- see their assumption A3) to obtain meaningful results.

\subsection{Contributions and Link to Prior Works}
\label{sec:intro-contributions}
In this work, we propose a novel approach for evaluating generative models with Precision and Recall, through the prism of binary classification. We also exploit the link between Precision and Recall and the Total Variation (TV) distance in order to make a fine grained statistical analysis on several of our proposed metrics. 
This analysis includes a non asymptotic convergence analysis for our KDE based methods which formally unveils the elephant in the room that is the curse of dimensionality.

In section \ref{sec:link-literature} we express how our novel framework actually coincides with the methods from the relevant literature on the extreme values. In particular, we detail how our approach restricted to these extreme values can be viewed as a dual approach of common metrics. In section \ref{sec:experiments} we first lead experiments on a controlled setting in which we can compare our proposed methods to the known ground truth. We then conduct experiments on real data and compare them to the results of articles from the same literature. Finally, in order to come to a conclusion concerning these ``real world'' experiments, we construct a hybrid experiment in which ground truth can effectively be computed and a qualitative conclusion can be made on our extended metrics.

\section{A New Framework for Evaluating Generative Models}
\label{sec:new-framework}

In this section, we formally present our novel framework for evaluating generative models. This framework stems from the dual approach proposed in \cite{simon_RevisitingPrecisionRecall_2019} related to binary classification. 
Before presenting this formulation, let us first introduce the notations.

\subsection{Notations}
\label{sec:notations}
\begin{itemize}
    \item $\langle \cdot , \cdot \rangle$ is the standard scalar product on $\R^d$ and $\|\cdot\|$ the associated norm;
    \item $\measSpace \subset \R^d$: a measurable space;
    \item $\probaSet$: set of probability measures on $\measSpace$; 
    \item $P,Q\in\probaSet$ (e.g. real and generated distributions);
    \item $\X$ (resp. $\Y$): finite training set of $N$ samples from $P$ (resp. $Q$) ($\# \X= \# \Y = N$);
    \item $\X'$ (resp. $\Y'$): finite testing set of $N$ samples from $P$ (resp. $Q$) ($\# \X'= \# \Y' = N$);
    \item $\bar P, \bar Q$: empirical distributions formed by the $N$ training samples;
    \item $\hat P, \hat Q$: estimates of the probability $P$ and $Q$; %
\end{itemize}

We also recall the following definitions.
\begin{definition}[TV norm] The total variation (TV) norm of a signed measure $\mu$ over $\measSpace$ is
    $$\TV{\mu} := \frac12 \sup_{A} \left( \mu(A) - \mu(A^c) \right)$$
    where $A$ is a measurable set and $A^c$ is the complement of $A$.\\
    The TV distance between $P, Q \in probaSet$ is
    $$
    \TV{P-Q} = \\ \int_{dP\ge dQ} d(P-Q) = \sup_{A} \left( P(A) - Q(A) \right).
    $$
     When the probability measures have densities w.r.t to the Lebesgue measure, \emph{i.e.,} $dP(x) = p(x)dx$ and $dQ(x) = q(x)dx$, the TV distance can be expressed as the half $L^1(\Omega)$ norm, 
     $$
     \TV{p-q} = \tfrac12 \|p-q\|_1 = \tfrac12 \int_{\Omega} |p(x)-q(x)| dx .
     $$
\end{definition}

\begin{definition}[Subgaussianity]
\label{def:subgaussianity}
A probability distribution $P$ on $\R^d$ is said to be subgaussian with constant $K$ iff $\forall \theta\in\R^d$ 
\begin{align}
\E_{X\sim P}\left[\exp\left(\left\langle \theta, X-\E[X] \right\rangle\right)\right]\leq&  \exp\left(\frac 12 K^2\Vert\theta\Vert^2\right) .
\end{align}
We will say that $X\sim P$ is a $K$-subgaussian random vector.
Besides the smallest value of $K^2$ for which $X$ is $K$-subgaussian is denoted $\Var_G(X)$ and is called the subgaussian variance of $X$ (with the convention $\Var_G(X)=+\infty$ iff $X$ is not subgaussian).
\end{definition}

\begin{definition}
\label{def:sobolev}
The Sobolev space $W^{1,1}$ on $\measSpace\subset\R^d$ is defined by:

\begin{equation*}
    W^{1,1}(\Omega) := \left\{ p \in L^{1}(\measSpace):\nabla p \in L^{1}(\Omega) \right\}
\end{equation*}
where $\nabla p$ is {\em a priori} defined in the distributional sense. This space is equipped with the semi-norm 
$$
\Vert p\Vert_{W^{1,1}(\measSpace)}:=\int_\measSpace \|\nabla p(x)\| dx .
$$
\end{definition}

\subsection{PR Curve Starting Point}

\paragraph{The initial framework}
As initially described in \cite{sajjadi_AssessingGenerativeModels_2018} the PR set between $P$ and $Q$ is defined as the set of non-negative pairs $(\alpha, \beta)$ such that $\exists \mu\in\probaSet$ verifying both $P\geq \beta \mu$ and $Q\geq \alpha \mu$.
This set is denoted as $\PRD(P,Q)$.
In a nutshell, the two conditions translate the fact that the ``probe'' distribution $\mu$ can simultaneously ``extract''  some mass $\beta$ from $P$ and $\alpha$ from $Q$.
Note that the PR set is included within $[0,1]^2$ and it is quite regular: it is topologically closed and also a lower set for the component-wise order.
The latter means that $\forall \alpha'\leq\alpha,\forall\beta'\leq\beta$, $(\alpha, \beta)\in \PRD(P,Q)\implies (\alpha', \beta')\in \PRD(P,Q)$. As a result, this set is fully characterized by its (upper-right) Pareto frontier denoted by $\dPRD(P,Q)$ which can be parameterized as $\dPRD(P,Q) = \{(\alpha_\lambda,\beta_\lambda), \lambda \in\bar\R^+\}$ with 
\begin{equation}
\label{eq:primalPR}
\begin{split}
\alpha_\lambda &= (\lambda P \wedge Q)(\Omega) = \frac{\lambda+1}2 - \TV{\lambda P -Q}\\
\beta_\lambda &= \frac{\alpha_\lambda}{\lambda}
\end{split}
\end{equation}
where $\wedge$ is the minimum operator between two measures. The link with TV can directly be adapted from \cite[Lemma~4.4.1]{verine_QualityDiversityGenerative_2024}. Note that since Recall can be deduced from Precision as $\beta_\lambda=\frac{\alpha_\lambda}{\lambda}$, we will often solely focus on the Precision $\alpha_\lambda$ value.

The information captured by this whole curve encompasses both extreme Precision and Recall values corresponding to $\alpha_\infty$ and $\beta_0$ which play a central role in the literature starting from \cite{kynkaanniemi_ImprovedPrecisionRecall_2019}. In addition, it also describes how similarly the mass is distributed within the common support of $P$ and $Q$ (see \citep{siry_TheoreticalEquivalenceSeveral_2022} for details). 
This curve can also be characterized in a \textbf{dual form, based on a specific two-sample classification problem} as developed in \cite{simon_RevisitingPrecisionRecall_2019}.

\paragraph{Dual approach description of PR curves}

This dual approach allows to express the Precision and Recall pairs $(\alpha_\lambda,\beta_\lambda)$ for a given $\lambda$ as a minimization problem over the hypothesis class $\F$ which is composed of all binary classifiers on $\Omega$. The expression under minimization is a linear combination of the false negative rate (FNR) and the false positive rate (FPR) of the binary classifier $f$, that is to say the probability that a sample $\Y\sim Q$ was classified as a sample from $P$ for FNR (resp. $\X\sim P$ classified as coming from $Q$ for FPR). With these notations set once and for all, the dual approach gives:
\begin{equation}
\label{eq:dualPR}
\begin{split}
    \alpha_{\lambda}(P,Q) &= \min_{f \in \F} \left\{ \lambda \fpr(f) + \fnr(f) \right\} \\
\end{split}
\end{equation}
where the false positive and negative rates write: 
\begin{equation}
    \label{eq:fpr}
    \fpr(f) = \int (1-f) dP \quad\text{ and }\quad \fnr(f)=\int f dQ .
\end{equation}

\subsection{A New Framework for Estimating $(\alpha_\lambda,\beta_\lambda)$}
\label{subsec:new-framework}
While \eqref{eq:dualPR} provides a rigorous alternative definition of Precision, it is  intractable to estimate in practice as minimization occurs on the the entire set of binary classifiers $\F$. 
To tackle this issue, \cite{simon_RevisitingPrecisionRecall_2019} propose to approximate the optimal Bayes classifier $f^*$ through the optimization over a parameterized family $\F_\theta$ of deep neural networks.
While this method does yield relevant Precision and Recall curves, it is computationally intensive since it requires training one classifier per point on the curve. 
In the proposed approach, we propose to replace the neural network parameterization by a family of non-parametric classifiers, \emph{e.g.} based on nearest-neighbors. For a specific classification method $M$, the corresponding family takes the form:
\begin{align}
\label{def:family-parametrised-classifiers}
\begin{split}
    \F^{M} &= \{f_\gamma^M: \gamma \in [0,+\infty]\} \\
    f^M_\gamma(z)& = \1_{\{\gamma c(\X,z) \geq c(\Y,z)\}} ,
\end{split}
\end{align}
where $c(\X,z)$ (resp. $c(\Y,z)$) denotes a matching score between some sample $z$ and samples from $\X$ (resp from $\Y$). 
The key idea is that a classifier will attribute a sample $z$ to distribution $P$ if it ``matches" better the samples from $P$ than it does those from $Q$.
Concrete instances of scores corresponding to standard classification rules will be given in an upcoming section.

Adapting the dual form \eqref{eq:dualPR} to this new family of classifiers $\F^M$, we obtain the following estimators of the Precision and Recall values based on Empirical Risk Minimization:
\begin{align}
    \label{eq:hatPRC}
        \hat \alpha^M_\lambda &:=\min_{\gamma \in \R^+}  \left\{\lambda\overline\fpr(f^M_\gamma) + \overline\fnr(f^M_\gamma) \right\} , %
\end{align}
where $\overline \fpr$ (resp. $\overline \fnr$) are the empirical estimates of the false positive (resp. negative) rate \eqref{eq:fpr}. 

In addition to these empirical estimators, we introduce for the theoretical analysis conducted in the next part another estimator for the Precision $\alpha_\lambda$. 
This estimator, denoted as $\hat \alpha_\lambda^\text{TV}$, is related to the TV distance between estimated distributions:
\begin{definition}[TV based plug-in estimator of Precision]
\label{def:alpha-TV}
\begin{equation}\label{eq:alpha-TV}
    \hat \alpha_\lambda^\mathrm{TV} :=  \alpha_\lambda(\hat P,\hat Q)
    = \frac{\lambda+1}{2} - \TV{\lambda \hat P - \hat Q} ,
\end{equation}
\end{definition}
where $\hat P,\hat Q$ are the estimated distributions based on samples $\X,\Y$ and the selected method $M$.
This estimator can be thought of as the "ground truth value" of Precision at $\lambda$ for the estimated distributions $\hat P,\hat Q$ and is in fact the plugin-in TV estimator of Precision at $\lambda$. While this estimator is not tractable in practice, it is useful in the statistical analysis that will be detailed in Section~\ref{sec:theoretical-consistency}. 
This notation emphasizes the fact that here the Precision is estimated directly from the TV distance between estimated measures, which is different from the proposed classifier-based approach $\hat \alpha^M_\lambda$ defined in~\eqref{eq:hatPRC}.

\subsection{Instantiating our Framework on kNN and KDE Families}
\label{sec:instantiating-framework}
In this section, we present effective families ---based on the $k$-nearest neighbors method (kNN) and KDE--- of classifiers $\F^M$ and their advantages. This section will be complemented in Section~\ref{sec:link-literature} with %
some extensions to encompass  
other metrics from the literature.

\paragraph{kNN family}
\label{sec:instantiate-knn}
An intuitive way of defining a binary classifier
which aims at discriminating between $P$ and $Q$ from sample points %
is to define a majority vote: for a given test sample, it simply consists in comparing the number of closest samples from $\X$ and from $\Y$.
We introduce the following family of kNN binary classifiers
\begin{equation*}
f^{\text{kNN}}_\gamma(z)=\1_{\gamma\#\{x\in \X / x\in\BKNN^{\X\cup\Y}(z)\}\geq \#\{y\in \Y / y\in\BKNN^{\X\cup\Y}(z)\}} \; ,
\end{equation*}
where the parameter $\gamma$ is related to the ratio of samples originating from $Q$. 
For $\gamma=1$, this binary classifier is an indicator function that outputs 1 when there are more samples from $\X$ falling in the kNN radius ball of $z$ than there are from $\Y$. This parameterized classifier will be discussed in depth in Sections~\ref{sec:knn-convergence} and \ref{sec:link-literature}. 

\paragraph{KDE family}
\label{sec:instantiate-kde}
Another family which naturally emerges when thinking of classification based on inter-sample distances is Kernel Density Estimation (KDE). We propose the simplest instance of KDE which uses fixed bandwidth. This method uses a uniform kernel and may be defined as follows:
\begin{equation}\label{eq:KDE_classifier}
    f_\gamma^{\text{KDE}}(z) 
    = \1_{\frac{\hat p(z)}{\hat q(z)}\geq \frac 1\gamma} ~ ,
\end{equation} 
with estimated density
$\hat p(z) \propto \sum_{x\in\X} \1_{\Vert x-z\Vert \leq \sigma}$ (similarly for $\hat q$)
using a positive bandwidth parameter $\sigma$.

\paragraph{Corresponding estimators}
With the parametrized families of classifiers for $\text{kNN}$ and $\text{KDE}$ we can now define the corresponding Precision metrics as:

\begin{align*}
    \hat\alpha^{\text{kNN}}_\lambda(P,Q)&=\min_{\gamma \in \R^{+}} \left\{ \lambda \fpr(f^\text{kNN}_\gamma) + \fnr(f^\text{kNN}_\gamma) \right\} , \\
    \hat\alpha^{\text{KDE}}_\lambda(P,Q)&=\min_{\gamma \in \R^{+}} \left\{ \lambda \fpr(f^\text{KDE}_\gamma) + \fnr(f^\text{KDE}_\gamma) \right\} .
\end{align*}

\subsection{Extending the Classification Framework}
\label{sec:extensions}
\paragraph{Introduction of splitting}
Based on the previous classification standpoint, it appears natural to split the dataset into two subsets: a ``training set'' ($\X \cup \Y)$ which is used to define the family of classifiers $f_\gamma^M$ and a ``test set'' ($\X' \cup \Y'$) used to evaluate the empirical error rates as defined above.
In this setting, the conventional empirical estimation of FPR and FNR for a classifier $f$ simply write:
\begin{equation*}
    \overline \fpr(f) := \frac{1}{|\X'|}
    \sum_{x \in \X'} (1-f(x))
    \qquad \text{ and } \qquad
 \overline \fnr(f) := \frac{1}{|\Y'|}
    \sum_{y \in \Y'}
    f(y).
\end{equation*}

By splitting the dataset, the law of large numbers applies which is crucial for the consistency of $\hat \alpha_\lambda^{M}$ which will be shown in Section~\ref{sec:theoretical-consistency}. 

Note however that because of splitting, it is possible that none of the classifier $f^M_\gamma$ ensures a null FPR. As a result, it is possible that $\hat\alpha^M_\lambda>1$. 
As a remedy, in our experiments, we always complement $\F^M$ with the trivial classifiers $f\equiv 1$ and $f\equiv 0$ that predict either $P$ or $Q$ uniformly.

From now on in the splitting setting $\X\cup\Y$ will denote the training set used to fit the classifiers and $\X'\cup\Y'$ the set on which the error rates are evaluated.

\paragraph{Hyper-parameter $k$}
In addition to the introduction of splitting, we consider modifying the hyper-parameter $k$ for the kNN balls approach. In the kNN literature (see e.g. \citet{devroye2013probabilistic}), it is known that as the number of samples $N$ gets bigger, $k$ must also increase but at a slower rate (this will be a key element to ensure the consistency of the kNN estimator in Theorem~\ref{thm:optimality}). We therefore consider for each approach, setting $k\propto
\sqrt N$, 
which is a standard choice of parameters in kNN classification and allows to benefit from the consistency Theorem~\ref{thm:optimality}.

\section{Theoretical Consistency Guarantees}
\label{sec:theoretical-consistency}
In this section, we will first prove asymptotic consistency of our kNN-based estimator. 
Then we will turn to establishing non-asymptotic consistency results for our KDE-based estimator. The second result highlights the curse of dimensionality as the upper-bound depends exponentially on the dimension.
These convergence analysis go way further than the ones in the relevant literature which includes examples where metrics are not even consistent.

\subsection{Asymptotic Consistency with kNN}
\label{sec:knn-convergence}

We first prove asymptotic statistical consistency of Precision and Recall estimators when using our kNN method defined in Section~\ref{sec:instantiating-framework}.
Once again, we will be focusing on $\alpha_\lambda$ as the results can be deduced for $\beta_\lambda$ using the relationship $\beta_\lambda=\alpha_\lambda / \lambda$.

\begin{theorem}
 \label{thm:optimality}
    Let $\lambda\in\bar\R^+$, $k\geq 3$ and $N=\#\X=\#\Y$. If $k\to\infty$ and $\tfrac{k}{N}\to 0$ as $N \to \infty$, and denoting
    $$
    \Gamma^*_\lambda=\argmin_{\gamma} \lim_{k\to\infty, \tfrac{k}{N}\to 0} \E[\lambda \fpr(f_\gamma^{\mathrm{kNN}})+\fnr(f_\gamma^{\mathrm{kNN}})] ,
    $$
    then, the following hold:
    \begin{enumerate}
        \item $\lambda\in\Gamma^*_\lambda$.
        \item $\E[\hat\alpha_\lambda^{\mathrm{kNN}}] \to \alpha_\lambda$ assuming that data split was used.
    \end{enumerate}
 \end{theorem}
The proof is provided in Appendix~\ref{app:proof-consistency-knn} and is similar to the standard Bayes consistency results of the kNN classifier (see, \emph{e.g.} \citep[Chapters~5 and~6]{devroye2013probabilistic}) .
     It is essentially adapted to the fact that the risk is class weighted \emph{i.e.} $R_\lambda(f)=\lambda \fpr(f) + \fnr(f)$ instead of the classical one $R(f)=\frac 12(\fpr(f)+\fnr(f))$.

In short, this theorem highlights that with increasing number of samples $N$, if the number of nearest neighbors $k$ is chosen to grow large but more slowly than $N$ so that it remains negligible in comparison, then our kNN-based estimator of Precision tends towards the ground-truth Precision in expectation.

\subsection{Non-Asymptotic Consistency with KDE: TV counterpart}
\label{sec:kde-convergence}

In this section, we derive an intermediate result about the optimal bound associated to the \emph{risk} of the estimated Precision based on the KDE. 
More precisely, we consider the equivalent formulation \eqref{eq:primalPR} of the Precision $\alpha_\lambda(P,Q)$ and the corresponding plugin-in estimator \eqref{eq:alpha-TV} based on the KDE of bandwidth $\sigma$. %

For simplicity, we use the same Parzen-window kernel $k_\sigma(z) \propto \1\{\Vert z \Vert \leq \sigma \}$ to estimate both densities, and we consider the estimator $\hat P := \bar P \ast k_\sigma$ (respectively $\hat Q := \bar Q \ast k_\sigma$),
where $\bar P$ is the empirical distribution of $N$ \emph{i.i.d.} samples $\cal X$ drawn from $P$ (resp. $\bar Q$ is the empirical distribution of $N$ samples $\cal Y$ from $Q$).\\

From the expression of $\alpha_\lambda$ in \eqref{eq:primalPR} and using standard statistical arguments, it is well-known that if one wants to establish precise non-asymptotic convergence rates, one has to trade on the generality w.r.t to the class of distributions.
Indeed, any estimator of $\alpha_\lambda(P,Q)$ can be trivially turned into an estimator of $\TV{\lambda P-Q}$ and vice versa. In that sense, the estimation difficulty of both functionals is the same. In fact, the difficulty of estimating $\TV{P-Q}$ (i.e $\lambda=1$) is not much studied in the literature but one can conjecture that it is related to the difficulty of estimating $P$ and $Q$ under the TV norm (more precise statements will be made later on in the paper).
It is important to know that in \citet{devroye_NoEmpiricalProbability_1990}, the authors exhibit a scenario in which for any number of samples $N$, and any estimation scheme, the estimation of a worst-case distribution $P$ is $1/2$ away from the actual target distribution in terms of TV norm. Such scenario relies on the potential existence of a singular component in the distribution (i.e. a component that is neither discrete nor absolutely continuous). Therefore in order to obtain any non-asymptotic consistency result, one usually requires some regularity and tail assumptions on the target distribution.
In our case, we will assume Sobolev-type regularity and subgaussianity (see Definition~\ref{def:subgaussianity}).

\begin{theorem}
\label{thm:minimax-upper-bound-via-plugin}
Assume that both $P$ and $Q$ are $K$-subgaussian and have densities w.r.t. Lebesgue measure that belong to the Sobolev ball of radius $R$. Set $\sigma\propto \frac 1 {N^{\frac 1{2+d}}}$. Then
\begin{equation}\label{eq:upper-bound-via-plugin}
    \E|\hat\alpha_\lambda^{\mathrm{TV}} -\alpha_\lambda(P,Q)|= O\left(\frac{\lambda+1}{N^\frac{1}{2+d}}\right) ,
\end{equation}
where the constant in the big O notation depends only on $K$, $R$ and the dimension $d$ (not on $P$ and $Q$ nor $\lambda$ and $N$).

Denoting $\mathcal M_{R,K}$ the set of probability distributions pairs $(P,Q)$ that are both $K$-subgaussian and whose densities are in the Sobolev ball of radius $R$, we obtain the  minimax upper-bound
\begin{equation}\label{eq:minimax-upper-bound-via-plugin}
    \inf_{\hat \alpha_\lambda} 
    \sup_{(P,Q)\in \mathcal M_{R,K}} 
        \E_{P^{\otimes N}\times Q^{\otimes N}}|\hat\alpha_\lambda -\alpha_\lambda(P,Q)|
    = O\left(\frac{\lambda+1}{N^\frac{1}{2+d}}\right) .
\end{equation}
\end{theorem}

In the remainder of this subsection, we will prove the theorem through a series of intermediate lemmas.

The following lemma highlights how the estimation error made on the Precision and Recall metrics can be transposed to a TV error estimation made on the distributions under study.
\begin{lemma}\label{prop:plugin-trivial-bound}%
\begin{align}\label{eq:plugin-trivial-bound}
    \E\left[\left|\hat \alpha^{\mathrm{TV}}_\lambda - \alpha_\lambda(P,Q)\right|\right]\leq&\lambda\E\left[\TV{\hat P -P}\right]+\E\left[\TV{\hat Q-Q}\right] .
\end{align}
\end{lemma}
\begin{proof}
Indeed
\begin{align*}
    \E\left[\left|\hat \alpha^{\mathrm{TV}}_\lambda - \alpha_\lambda(P,Q)\right|\right]=&\E\left[\left|\TV{\lambda\hat P -\hat Q} - \TV{\lambda P -Q}\right|\right] &\\
    \leq& \E\left[\TV{\lambda\hat P -\hat Q - (\lambda P -Q)}\right] & \text{(reverse triangle inequality)}\\
    \leq&\lambda\E\left[\TV{\hat P -P}\right]+\E\left[\TV{\hat Q-Q}\right] & \text{(direct triangle inequality)} .
\end{align*}
\end{proof}

Because of the previous bound, one can now consider the TV risk of the KDE estimator of $P$ (or $Q$). To do so, it is standard to use the bias-deviation decomposition which follows from the triangular inequality.
\begin{lemma}[Bias-Deviation decomposition]\label{prop:bias_deviation_bound}
\begin{align}
    \E\left[\TV{\hat P -P}\right]\leq \underbrace{\E\left[\TV{\hat P -\E[\hat P]}\right]}_{\mathrm{deviation}} + \underbrace{\E\left[\TV{\E[\hat P] -P}\right]}_{\mathrm{bias}} .
\end{align}
Besides, for the considered estimator $\hat P:=\bar P\ast k_\sigma$, we have
\begin{equation}
    \E[\hat P] = P\ast k_\sigma .
\end{equation}
\end{lemma}

\begin{lemma}[Deviation bound]\label{prop:risk_deviation}
For $P$ a $K$-subgaussian probability distribution defined in $\R^d$, $\bar P$ an empirical distribution based on $N$ samples of $P$ and a Parzen-window kernel $k_\sigma$ of width $\sigma > 0$ we have
\begin{equation}
        \E [\TV{\bar P \ast k_\sigma - P\ast k_\sigma}]  \leq \sqrt{\Gamma\left(\frac d2+1\right)}  8^{\frac d4} \left(\frac\eta d +\frac{K^2}{\sigma^2}\right)^{\frac d4}\frac 1{2\sqrt N} ,
\end{equation}
where $\eta$ is a universal constant independent from $d$ (see Proposition~\ref{prop:subgauss-props} for details) and $\Gamma$ is Euler's gamma function.
\end{lemma}

See \ref{sec:proof-smoothed-TV-bound-highSNR} for the proof.

\begin{remark}
Lemma~\ref{prop:risk_deviation} is stated for the Parzen-window kernel.
Nonetheless the proof can be easily extended to any kernel $k$ obeying the following growth condition:
$$
\exists c>0, \forall z\in\R^d, k^2(z)\leq c \tilde k(z) ,
$$
where $\tilde k$ is non-negative subgaussian kernel (i.e. corresponds to a density that is subgaussian).
In particular, this condition does not require the original kernel $k$ to be non-negative, which is an important flexibility in order to control the bias under high-order smoothness conditions. %
\end{remark}

In what follows, we consider probability distributions having densities which lie in a Sobolev space $W^{1,1}(\R^d)$. %

\begin{lemma}[Bias bound]\label{prop:risk_bias}
\label{prop:error-estimator-tv-density-bias-term}
Let $R>0$ and assume that $P$ has a density $p \in W^{1,1}(\R^d)$ w.r.t. Lebesgue measure   such that $\Vert p\Vert_{W^{1,1}(\R^d)}\leq R$. Then
$$
\TV{P-P\ast k_\sigma}\leq \frac R2 \sigma .
$$
\end{lemma}

See \ref{proof:tv-upperbound-bias} for the proof.

\begin{remark}\label{rem:risk_bias}
Again this lemma can be extended to more general kernels. For first-order smoothness, any  kernel $k$ (not necessarily non negative) is fine, as long as $\int \Vert z\Vert |k(z)|dz<+\infty$ (this quantity would of course appear in the upper bound).  For higher order smoothness $s$, the previous condition is replaced by $\int \Vert z\Vert^s|k(z)| dz<+\infty$ along with assumptions of vanishing signed moments up to order $s-1$. The proof is again similar, but rely on a higher order Taylor expansion. In this context, the bias term is controlled as $O(\sigma^s)$. This kind of result is standard in the approximation folklore and is mentioned merely for completeness.
\end{remark}

\begin{proof}[of Theorem~\ref{thm:minimax-upper-bound-via-plugin}]
First, combining Lemma~\ref{prop:risk_deviation} and Lemma~\ref{prop:risk_bias}, and optimizing for $\sigma$ in the resulting bound, the best bias-deviation trade of is achieved for the devised choice $\sigma\propto \frac 1 {N^{\frac 1{2+d}}}$. Plugging this into Lemma~\ref{prop:bias_deviation_bound} and then using Lemma~\ref{prop:plugin-trivial-bound} yields the first claim of Theorem~\ref{thm:minimax-upper-bound-via-plugin}. The second claim follows from the first one as the left hand side of \eqref{eq:minimax-upper-bound-via-plugin} is obviously upper bounded by the one of \eqref{eq:upper-bound-via-plugin}. 
\end{proof}

\begin{remark}
The result of Theorem~\ref{thm:minimax-upper-bound-via-plugin} is consistent with known results on density estimation under $L^1$ or $TV$ norm. Roughly speaking, for smoothness order $s$ (e.g in terms of Sobolev spaces or other regularity spaces), the minimax lower bound rate is known to be $\frac C{N^{\frac s{2s+d}}}$ for $C > 0$; see \cite{tsybakov2009nonparametric}. Following the discussion in Remark~\ref{rem:risk_bias}, to obtain an estimator for larger smoothness orders, it is well known that the Parzen-Window (constant) kernel is not sufficient. Instead one needs to use more regular kernels with vanishing signed moments up to order $s-1$, which implies that the kernel cannot be non-negative anymore.

That being said, one should realize from the proof of Lemma~\ref{prop:plugin-trivial-bound} that what really matters for our purpose is consistent  estimation of $\alpha_\lambda(P,Q)$, or equivalently, of $\TV{\lambda P-Q}$. Even for the simple case of $\lambda=1$, less is known about the consistency rate of $\TV{P-Q}$ except that it is necessarily faster than the rate for density estimation under $TV$ norm. This distinction between the convergence rate for an estimation under a functional vs the estimation of the functional itself is classical (see e.g. \citep{lepski1999estimation}). Usually the rates differ by a slow correcting factor (e.g. logarithmic or even sub-logarithmic in $N$). 

On the other hand,  the dependence we have obtained for the prefactor w.r.t the value of $\lambda$ is $\lambda+1$. This stems from the use of the triangular inequality. This approach is quite crude, since it would imply an infinite pre-factor when $\lambda\to\infty$, while the  consistency bound is necessarily lower than $\frac 12$ for all $\lambda \in [0,+\infty]$. Indeed, we know that the estimation target $\alpha_\lambda\in [0,1]$, so that trivial estimator $\hat \alpha_\lambda=\frac 12$ display the prescribed error of $\frac 12$ (see the discussion on the onset of Section~\ref{sec:kde-convergence}).

Intuitively, if one considers a common support hypothesis, then the  consistency error is zero for $\lambda\to\infty$ since we know for sure that $\alpha_\infty(P,Q)=1$ in that case. So one should lean towards thinking of $\lambda=1$ as the most difficult configuration. That being said, despite being widespread in density estimation, the common support assumption is not well adapted to a setting where one desires to estimate $\alpha_\infty$. Indeed  one would not seek to estimate a value that carries no information (as is the case for $\alpha_\infty$ under the common support assumption). Studying the  consistency of estimating $\alpha_\lambda(P,Q)$ under more general settings and its dependence on $\lambda$ is therefore a very interesting avenue that we leave to a future work.

Another interesting direction of future research is proving that the rate of Theorem~\ref{thm:minimax-upper-bound-via-plugin} is (nearly) optimal. We are currently working towards this direction by establishing a minimax lower bound that would reveal that the rate of Theorem~\ref{thm:minimax-upper-bound-via-plugin} is optimal up to a logarithmic factor.
\end{remark}

\subsection{Bias of the KDE method}
\label{sec:bias_analysis_KDE}

In the previous section we obtained an optimal bound on the risk associated to the estimated Precision using the plug-in estimator $\hat \alpha^{\text{TV}}_\lambda$ based on the TV distance using KDE. 
However, this analysis does not account for the practical way the Precision is estimated, based the evaluation of a classifier as proposed in \eqref{eq:hatPRC}, rather than a distance between densities. Furthermore, the TV norm $\TV{\lambda \hat P - \hat Q}$ is intractable in practice.
In this section, we show that we can still use this result to directly derive a bound on the bias associated to the KDE classifier~\eqref{eq:KDE_classifier} used to estimate
$\hat \alpha^{\mathrm{KDE}}_\lambda$.

Recall first that the proposed approach is based on a ``data splitting'' strategy introduced in \ref{sec:new-framework} where the classifier is evaluated on a different set than the training one.
Consequently, let us denote by $(\X, \Y)$ (resp. $(\X', \Y')$) random and independent samples from the distributions $(P,Q)$ which designate the training set (resp. the evaluation set). Namely, $\X, \X' \sim P^{\otimes N}$ and $\Y, \Y' \sim Q^{\otimes N}$. The distributions $\hat P, \hat Q$ are the estimated distributions based on the empirical samples $\X, \Y$. %

\begin{proposition}[Bound on the bias of the KDE-classifier estimator]
\begin{equation}\label{eq:bias_bound}
    \left|\E
    \left[ \hat \alpha^{\mathrm{KDE}}_\lambda - \alpha_\lambda \right]\right|
    \le
    2\lambda \E \TV{\hat P - P} + 2 \E \TV{\hat Q - Q} ,
\end{equation}
where once again $\hat P:=k_\sigma\ast\bar P$ and $\hat Q:=k_\sigma\ast\bar Q$ are the KDE estimators of $P$ and $Q$.
\end{proposition}

See \ref{sec:proof-bias-kde} for the proof.

\begin{remark}
We have seen in Theorem~\ref{thm:minimax-upper-bound-via-plugin}, that the plug-in estimator displays a fairly good convergence rate for its estimation risk (if the bandwidth $\sigma$ is set appropriately), namely $O\left(\frac{\lambda+1}{N^{\frac{1}{2+d}}}\right)$. The same goes for the classifier KDE estimate (which can be implemented in practice), but in that case, we only control the bias. To get a guarantee in terms of risk, one would need to also control the mean average deviation for this estimator, and show that it is $O\left(\frac{\lambda+1}{N^{\frac{1}{2+d}}}\right)$. This result is more involved due to the minimization over $\gamma$.
\end{remark}

\section{A Bridge Towards the Existing Literature}
\label{sec:link-literature}
In this section we establish a link between the proposed framework described in Section~\ref{sec:new-framework} and scalar Precision and Recall metrics from the state of the art. 
We show that these metrics can be 
expressed as the extreme values of Precision and Recall curves when carefully selecting a parameterized family $\F_M$.

\subsection{Previous Work Under our Setting}
\label{sec:description-common-pr-metrics}

In this section, we formally describe a few published metrics related to extreme Precision and Recall ($\alpha_\infty$ and $\beta_0$). Once again, we will focus on the estimate of $\alpha_\infty(P,Q)$ because swapping the role of $P$ and $Q$ entails the extreme Recall $\beta_0(P,Q)$. 

Recall that we consider a finite set $\X$ of examples from $P$ and others in $\Y$ sampled from $Q$ with $\# \X= \# \Y = N$. In the literature, no splitting occurs that is to say $\X'=\X$ and $\Y'=\Y$.

\paragraph{IPR} Proposed in \citet{kynkaanniemi_ImprovedPrecisionRecall_2019}, the Improved Precision and Recall metric is given by
\begin{equation}
    \label{eq:improved}
    \hat\alpha_\infty^{iPR} := \frac{1}{N}\sum_{y\in \mathcal{Y} } \1_{\exists x\in \mathcal X, y\in \BKNN^{\mathcal X}(x)} ,
\end{equation}
where $\mathcal X$ and $\mathcal Y$ are the observed samples from $P$ and $Q$ respectively, 
and $\BKNN^{\mathcal X}(x)$ represents the kNN ball around $x$ computed within the set $\mathcal X$.
This value can be interpreted as the empirical estimate of $Q(\supp(P))$ where samples from $\mathcal Y$ are used to estimate the $Q$ probability, and those from $\mathcal X$ are used to estimate the support of $P$ as the union of kNN balls. 

\paragraph{Coverage} First proposed as an estimate of $\beta_0$ in \citet{naeem_ReliableFidelityDiversity_2020}, it was also adapted to $\alpha_\infty$ in \citet{khayatkhoei_EmergentAsymmetryPrecision_2023}.
\begin{equation}
    \label{eq:coverage}
    \hat\alpha_\infty^{cov} := \frac{1}{N}\sum_{y\in \mathcal{Y} } \1_{\exists x\in \mathcal X, x\in \BKNN^{\mathcal Y}(y)} .
\end{equation}
Note that compared to Equation~\eqref{eq:improved}, the condition $y\in \BKNN^{\mathcal X}(x)$ is merely replaced by $x\in \BKNN^{\mathcal Y}(y)$. 
This metric was proposed to leverage the sensitivity to outliers of \ipr{}.

\paragraph{EAS} \citet{khayatkhoei_EmergentAsymmetryPrecision_2023} propose to combine both previous estimates by taking their minimum
\begin{equation}
    \label{eq:eas}
    \hat\alpha_\infty^{eas} := \min(\hat\alpha_\infty^{iPR}, \hat\alpha_\infty^{cov}) .
\end{equation}

\paragraph{PRC}
Proposed in \citet{cheema_PrecisionRecallCover_2023} Precision and Recall Cover is an extension of coverage:

\begin{equation}
    \label{eq:prc}
\hat\alpha_\infty^{PRC} = \frac{1}{N}\sum_{y\in \mathcal{Y} } \1_{\#\{x\in\mathcal X/ x\in \BKNN^\mathcal{Y}(y)\} \geq k'} \; ,
\end{equation}
where $k'\in\N^*$ is an additional hyper-parameter: setting $k'=1$ makes this estimator identical to  $\alpha_\infty^{cov}$.

\paragraph{PPR} Last, Probabilistic Precision and Recall was proposed in \citet{park_ProbabilisticPrecisionRecall_2023}:
\begin{equation}
    \label{eq:ppr}
    \hat\alpha_\infty^{PPR} := \frac{1}{N}\sum_{y\in \mathcal{Y} } \left( 1-\prod_{x\in\mathcal X} \left(1-\tau(\Vert y-x \Vert)\right) \right) ,
\end{equation}
where $\tau(d)=\max(0,1-\tfrac dR)$ is a fixed bandwidth tent kernel and $R$ a hyper-parameter defining the width of the kernel. This method allows to take into account the stochasticity of the samples by accounting for an uncertainty on the borders of the estimated support.

\subsection{Classification Interpretation}
\label{sec:classification-interpretation}

One may notice that every estimators mentioned in the previous section reads as
\begin{equation*}
    \hat\alpha^{M}_{\infty}=\overline\fnr(f^{M}_\infty) ,
\end{equation*}
where $\overline\fnr$ is the empirical FNR, $M$ denotes the selected method (e.g. \ipr{}, \cov{}) and $f^M_\infty$ is a classifier specific to the approach which takes as an input any $z \in \Omega$.
In particular, one has $f^{iPR}_\infty(z)=\1_{\#\{x\in \X / z\in\BKNN^{\X}(x)\}\geq 1}$ and $f^{cov}_\infty(z)=\1_{\#\{x\in \X / x\in\BKNN^{\Y}(z)\}\geq 1}$.

In most cases (all except PRC when $k'>1$), by design $\overline\fpr(f^{M}_\infty)=0$ because the classifier is equal to $1$ on training samples from $\X$ ---which are the same used to evaluate the error rates.
This is reminiscent of the form of $\alpha_\lambda$ in Equation~\ref{eq:dualPR} when $\lambda\to\infty$:
\begin{equation*}
    \alpha_\infty = \min_{f\in\F\text{ s.t. } \fpr(f)=0} \fnr(f) .
\end{equation*}
In fact, one can analyze each approach $M$, as the instantiation over a carefully selected family of classifiers $\F^M$ as defined in Definition~\ref{def:family-parametrised-classifiers}.

Note that there is no unique way of defining the class of functions $\F^M$ from the extreme classifiers $f_\infty^M$ and $f_0^M$. Yet, we shall now see that natural families emerge for both \ipr{} and \cov{}.

\paragraph{IPR} In that case, we set 
$$
f^{iPR}_\gamma(z)=\1_{\gamma\#\{x\in \X / z\in\BKNN^{\X}(x)\}\geq \#\{y\in \Y / z\in\BKNN^{\Y}(y)\}} .
$$
Note that this classifier is a Kernel Density Estimator (KDE) of the form 
$$f^{iPR}_\gamma(z)=\1_{\frac{\hat p(z)}{\hat q(z)}\geq \frac 1\gamma}$$
where $\hat p(z) \propto \sum_{x\in\X} \1_{\BKNN^{\X}(x)}(z)$ and similarly for $\hat q(z)$. It is therefore a KDE classifier with adaptive bandwidth.

\paragraph{Coverage} Here we set 
$$
f^{cov}_\gamma(z)=\1_{\gamma\#\{x\in \X / x\in\BKNN^{\Y}(z)\}\geq \#\{y\in \Y / y\in\BKNN^{\X}(z)\}} .
$$

This resembles to a classical kNN classifier up to a minor difference: it is more standard to use the same kNN structure for both classes, that is $\BKNN^{\X\cup \Y}$ rather than using a separate one per class.
Interestingly, one can verify that the condition $\exists x\in \X\text{ s.t. } x\in\BKNN^{\Y}(y)$ is in fact equivalent\footnote{One direction is trivial since $\BKNN^{\X\cup\Y}(y) \subset \BKNN^{\Y}(y)$, the other requires a bit more 
reasoning: assuming $\exists x\in\BKNN^{\Y}(y)$ one may consider in particular the $x$ closest to $y$ and conclude that it belongs to $\BKNN^{\X\cup\Y}(y)$.} to $\exists x\in  \X\text{ s.t. } x\in\BKNN^{\X\cup\Y}(y)$. In other words, without split our \knn{} method matches the extreme values of Coverage.

\begin{remark}[On symmetry]
For symmetry reasons between Precision and Recall, we use the previous definitions of $f^M_\gamma$  for $\gamma\geq 1$  and favor a strict inequality over a loose one for $\gamma<1$. 
\end{remark}

\begin{remark}[Fixed vs adaptive bandwidth]
Adaptive bandwidth KDE estimators based on kNN distances---such as those used in \ipr{}---may seem like a natural choice. However, in such settings, outliers can significantly inflate the local bandwidth, distorting the estimator's output by attributing excessive support to regions that should lie outside the true distribution.
Using a fixed bandwidth as in our \kde{} estimator (see Section~\ref{sec:instantiate-kde}) helps in leveraging this issue.
\end{remark}

\begin{remark}[Splitting]
\label{remark:splitting}
While the metrics from the literature were defined out of the splitting context, extending them is straightforward.
This is done by computing the empirical error rates on ``test set'' $\X' \cup \Y'$ different from the ``training set'' ($\X \cup \Y)$ used to the define the classifiers.
\end{remark}

\subsection{The Correct Version of Extreme Precision}

In the literature, the accepted expression of the extreme Precision is 
\begin{equation}
    \label{eq:link-sota-pr}
    \alpha_\infty(P,Q):=\lim_{\lambda\to\infty} \alpha_\lambda(P,Q) = Q(\supp(P)).
\end{equation}
While this formula is used as a theoretical basis in the literature, it is flimsy and requires to be amended mainly because the support of a distribution is defined up to null sets for that distribution. 
In the incriminated identity, the issue stems from the fact that adding a $P$-null set can change the $Q$-mass of the set, and therefore the right-hand side is not well characterized.
Correcting identity \eqref{eq:link-sota-pr} requires to clarify a few notions beforehand. 
\begin{definition}[support and co-support]
\label{def:support-cossuport}
Let $A$ be a measurable subset of $\measSpace$ and $A^c$ its complement.
We say that $A$ is a
\begin{itemize}
    \item support of $P$, denoted\footnote{This is a slight abuse of notations.}  $A=\supp(P)$, iff  $P(A^c)=0$. 
    \item co-support of $P$ and $Q$  denoted $A=\cosupp(P,Q)$ iff 
    (
    $(P\wedge Q)(A^c)=0$ 
    and $\forall B\subset A$, $P(B)=0\Leftrightarrow Q(B)=0$).
\end{itemize}
\end{definition}
As the reader may notice, the second notion is characterized up to sets that are simultaneously $P$ and $Q$ null. More precisely, we have the following result.
\begin{proposition}
    \label{thm:cosupp}
    All co-supports of $P$ and $Q$ have the same $Q$-mass and 
    $$
    \alpha_\infty(P,Q)=Q(\cosupp(P,Q)) .
    $$
\end{proposition}
\begin{proof}
See proof in Appendix~\ref{proof:cosupp}.
\end{proof}

\begin{example}
The following example illustrates the difference between the non-corrected and the corrected variants. Let $P = N(0,\mathrm{Id})$ and let $Q$ be a singular distribution, e.g., $\delta_x$ for some $x \in \R^d$. Then the entire PR curve is trivial and $\alpha_\infty(P,Q) = 0$. However, $\supp(P) = \R^d$, so $Q(\supp(P)) = 1$. In contrast, $\cosupp(P,Q) = \emptyset$, hence $Q(\cosupp(P,Q)) = 0$.

\end{example}

A natural follow-up question is whether either variant (based on support or co-support) provides a meaningful Precision metric. We believe the answer is negative for both. The issue is that these quantities tend to saturate at $1$ even in situations where Precision should be considered low. Beyond the classical example of shifted Gaussians (which yields perfect Precision regardless of the shift), the example above—and slight variations of it—illustrate this limitation.

Consider the following interpretation. Let $P$ be a white noise distribution over images, and let $Q$ be a Dirac distribution concentrated on a single image (e.g., a dog). A meaningful Precision metric should reflect whether samples from $Q$ resemble typical samples from $P$. Clearly, a dog image is not representative of white noise, so assigning a Precision score of $1$ would be incongruous.

Moreover, this issue persists (albeit to a lesser extent) with the corrected formulation. If $Q$ is instead a narrow Gaussian centered at the dog image (i.e., the dog image slightly perturbed by a zero-mean Gaussian noise), both $Q(\supp(P))$ and $Q(\cosupp(P,Q))$ still evaluate to $1$, despite the lack of alignment with the notion of Precision. For this reason, we argue that extreme metrics such as $\alpha_\infty$ are not suitable estimation targets. Consequently, studying their convergence rates is of limited interest. Instead, we advocate evaluating full PR curves and, if needed, summarizing them using appropriate statistics.

 More specifically, the previously mentioned caveats reflect the influence of the tails of $P$ and $Q$ on the extreme metrics, even when those tails decay very fast. The previous example of white-noise ($P=\mathcal N(0,\mathrm{Id})$) vs noisy-dog images ($Q=\mathcal N(x,\sigma\mathrm{Id})$) is illuminating. Whatever the noise scale $\sigma$ (be it extremely small), it remains that $P$ and $Q$ have full co-support and therefore $\alpha_\infty(P,Q)=1$.
The first negative impact of this observation is that both $\alpha_\infty$ and $\beta_0$ provide a very weak characterization of the relation between $P$ and $Q$.
The second negative impact concerns the estimation of $\alpha_\infty$: namely, the tails of $P$ and $Q$ are elusive based on empirical samples, making this extreme Precision the most challenging to evaluate.
Both of these observations have gone unnoticed by the previous approaches starting from \citet{kynkaanniemi_ImprovedPrecisionRecall_2019} as they purposely focused on the extreme values.
In addition to those two arguments, we would like to highlight that estimating the mass of a support is not a standard topic in machine learning. As a result, taking the binary classification dual standpoint brings much more useful hindsight to design estimators correctly.

\subsection{Distilling Curves with Scalar Metrics}
\label{section:distilling-curve}

While we have just highlighted the importance using Precision and Recall curves over scalar Precision and Recall metrics, practitioners may find it useful to summarize the PR curve using two metrics that respectively reflect Precision and Recall, thereby trading exhaustiveness for conciseness.

This can be particularly helpful for simplifying model comparisons. In this context, we previously discussed how focusing solely on extreme Precision or Recall values is often far from ideal. We therefore discuss four alternative approaches and examine their behavior in a scenario combining mode dropping, mode invention, and mode re-weighting (e.g., Figure~\ref{fig:mode-diffs}).

An empirical evaluation of the advantages and limitations of all four alternatives would be a valuable future direction. For example, one could study how these metrics behave with respect to hyperparameters in state-of-the-art generative models, such as the truncation trick in GANs, the temperature in auto-regressive model and normalizing flows, or the guidance scale factor in diffusion models. 
Such experiments are presented in Sections~\ref{section:stylegan-exp}~and~\ref{section:stable-diffusion-experiments}.

\paragraph{AuC} Area under the curve (AuC) is a standard metric which synthesizes a curve by simply outputting its area. This metric is particularly used in binary classification. Such a metric is proposed in \citet{pillutla_MAUVEMeasuringGap_2021} where the MAUVE metric is the area under the curve which is defined by type I and type II errors on generated and real distributions. In our case the AuC simply denotes as $\int \alpha_\lambda(\beta_\lambda) d\beta_\lambda$. The values of AuC range from 0 to 1. An AuC of 0 means that the PR curve boils down to the point $(0,0)$ and is synonymous of the two distributions under study having disjoint supports. On the contrary, AuC=1 means that the two distributions under study are the same. This score behaves as an averaging metric which tends to soften extreme behaviors for example giving the same score to curves with very large Recall and very small Precision and conversely meaning it fails to distinguish high Recall and high Precision because a symmetrized version of a curve would give exactly the same AuC value. However, AuC complements  the extreme values as it  synthesizes what happens between those two extreme values.

\paragraph{F-scores} Proposed by \citet{sajjadi_AssessingGenerativeModels_2018}, the $F_b$ score is defined as:
$$F_b = \max_{\lambda \in [0,+\infty]} \frac{b^2+1}{\frac{b^2}{\alpha_\lambda}+\frac 1{\beta_\lambda}}.$$
As $b \to \infty$, $F_b \nearrow \alpha_\infty$, and conversely, as $b \to 0$, $F_b \searrow \beta_0$. Based on this, \citet{sajjadi_AssessingGenerativeModels_2018} proposed using both $F_b$ and $F_{1/b}$ with a large value of $b$ (specifically, $b=8$). Although not explicitly discussed in their original work, it can be understood that these metrics will not be sensitive to rapidly decaying infinite tails. However, note that the $F_b$ score is a weighted harmonic mean computed from a single (optimal) point on the PR curve, meaning it can remain unchanged even when the curve is modified elsewhere. As a result, when $b$ is large, $F_b$ and $F_{1/b}$ mainly capture mode dropping or invention but are only superficially informative about mode re-weighting.

\paragraph{PR median} Another alternative is to consider $(\alpha_{\bar \lambda}, \beta_{\bar \lambda})$, where $\bar\lambda$ is chosen such that the line $\alpha = \bar\lambda \beta$ divides the area under the PR curve into two equal parts. In cases of pure mode dropping or invention, we have $\alpha_{\bar\lambda} = \alpha_\infty$ and $\beta_{\bar\lambda} = \beta_0$. By contrast, when infinite tails induce sharp transitions to $\alpha_\infty = 1$ (as seen in Figure~\ref{fig:mode-diffs}), $\alpha_{\bar\lambda} < 1$ will be only mildly affected by the fast-decaying tails, similar to $F_b$. Unlike $F_b$, however, $\alpha_{\bar\lambda}$ is significantly influenced by the presence of transitions due to mode re-weighting (purple transition in Figure~\ref{fig:mode-diffs}). As a result, these metrics may be preferable when one wishes to account for both mode re-weighting and mode dropping/invention.

\paragraph{PR\,@\,$\epsilon$}
When $P$ and $Q$ share the same support, the extreme values will always equal 1, providing little useful information about the distributions under study. However, in the case of Gaussian distributions, as the means of $P$ and $Q$ diverge, the PR curve tends toward the point $(0,0)$. Capturing the behavior “just before” reaching these extremes can provide more detailed insights. This can be done by computing Precision at a fixed value of Recall, denoted $\alpha@\epsilon$, or vice versa. In our case, we set $\epsilon = 0.05$. This idea is inspired by recommender systems or information retrieval evaluation, where one might measure how many of the top-$k$ selected items are relevant. In short, $\alpha @ \epsilon$ (or equivalently $\beta @ \epsilon$) refers to the maximum $\alpha$ such that $(\alpha, \epsilon) \in \PRD(P, Q)$ if such a point exists and $0$ otherwise.

\section{Experiments}
\label{sec:experiments}
In this section, we first revisit several experiments on toy datasets previously proposed in the literature where the ground truth is available and results are easy to interpret (Section~\ref{section:toy-experiments}. We then apply our metrics on data generated by StyleGAN-V2 model reproducing a key experiment from \cite{kynkaanniemi_ImprovedPrecisionRecall_2019} and computing PR curves based on embeddings (Section~\ref{section:stylegan-exp}).
However, analyzing experiments conducted solely on real-world data is inherently challenging due to the absence of clear ground truth, making it difficult to rigorously evaluate the performance of the metrics. To address this, we introduce a hybrid experiment: starting from real samples, we fit Gaussian distributions and generate synthetic samples following the corresponding Gaussian laws.
This hybrid setup provides a controlled reference against which we can meaningfully compare our metrics (Section~\ref{sec:hybrid-exp}).
We believe that some encouraging results in the literature are the effect of a combination of negative bias---due to correlation as train and test samples are the same---which could be compensated by the positive bias due to the sub-optimal selected classifier.
These experiments showcase the pros and cons of the various extensions, yet without declaring a clear winner.

In short, in the toy experiments, \cov{} and \knn{} perform far better than \ipr{} which fails to recover the middle part of the ground truth curve. 
In the StyleGAN experiment, \ipr{} exhibits a very different behavior than \cov{} and \knn{} which in any event is hard to analyze. Finally, the proposed hybrid setting illustrates the advantages of each method in different situations.
Additional experiments are presented in Appendix~\ref{app:add_exp} including the specific case $P=Q$ advocating  the use of data splitting.

In all the experiments, we denote $n$ as the total number of samples per distribution. Therefore, considering the previous notations, $n=2N$ when splitting occurs otherwise $n=N$.

\subsection{Toy Examples}
\label{section:toy-experiments}
In all toy experiments, we use $n = 10$K samples in $\mathbb{R}^{64}$ for the various estimators under study, matching the experimental settings of previous works. In each case, the distributions $P$ and $Q$ are known analytically, and the ground truth PR curve can be easily estimated since the Bayes classifier is the likelihood ratio classifier $f_\lambda^*(z) = \1_{\frac{dQ}{dP}(z) \leq \lambda}$ \citep{simon_RevisitingPrecisionRecall_2019}. To achieve high-accuracy ground truth curves, we apply a large-sample Monte Carlo estimation ($n^{GT} = 100$K) and estimate:
$$\hat\alpha^{GT}_\lambda = \lambda \overline{\fpr}(f_\lambda^*) + \overline{\fnr}(f_\lambda^*).$$

Based on this PR curve, we can either visually assess the quality of an estimator or use a scalar indicator to summarize performance. In particular, we propose using the Intersection over Union (IoU) score between the ground truth curve and the estimator under review. This index ranges from $0$ to $1$, with values closer to $1$ indicating stronger similarity to the ground truth.

\subsubsection{Shifted Gaussian} %
\label{sec:exp_shift}

\setlength{\myheight}{0.3\textwidth}

\begin{figure*}[t]
    \centering
    \begin{tabular}{cc}
        \multicolumn{2}{c}{\emph{50\% split validation/train}} \\
        \includegraphics[height=\myheight]{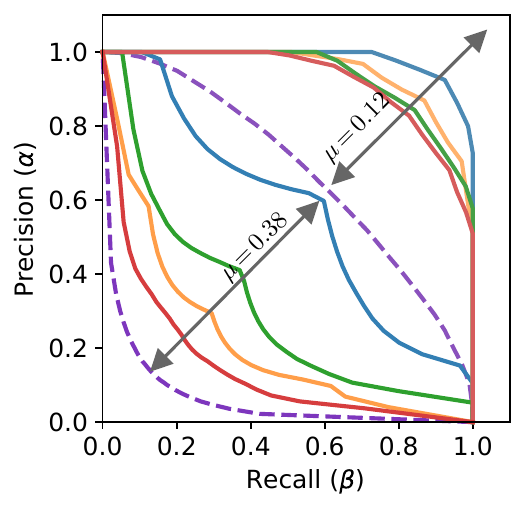} &
        \includegraphics[height=\myheight]{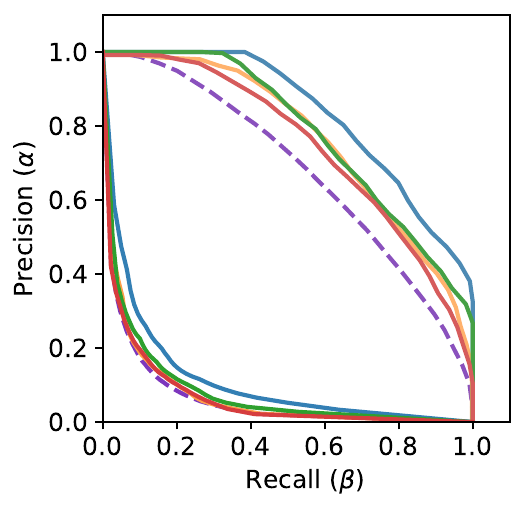} \\
        $k=4$ & $k = \sqrt n$ \\
        \multicolumn{2}{c}{\emph{without split}} \\
        \includegraphics[height=\myheight]{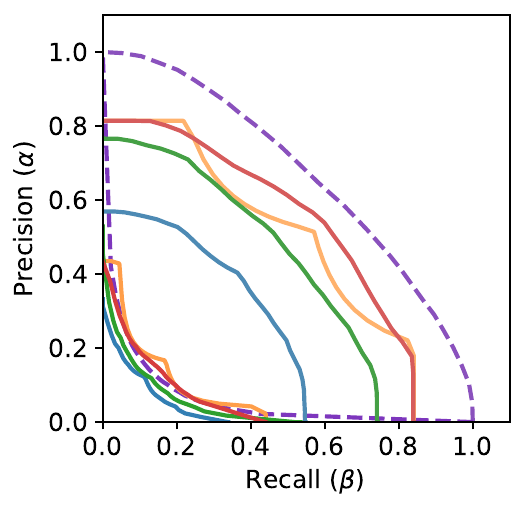} &
        \includegraphics[height=\myheight]{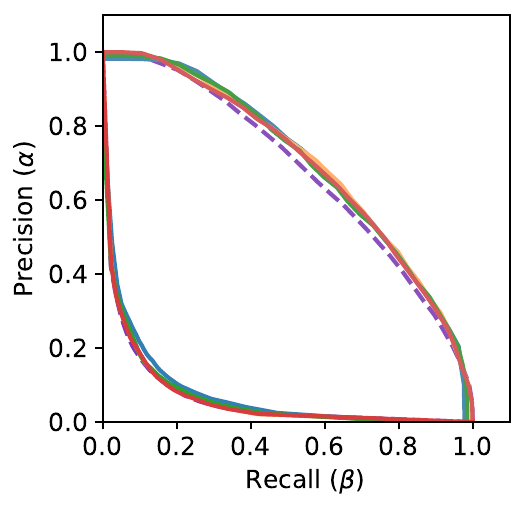} \\
        $k=4$ & $k = \sqrt n$
    \end{tabular}
        \caption{\textbf{Comparing two shifted Gaussians.} 
    The ground truth PR curve (
    {\color{color_GT} \bfseries -~-\textsc{GT}})
    is compared to empirical estimates from various NN-classifiers:
    {\color{color_IPR} \bfseries --\ipr{}},
    {\color{color_KNN} \bfseries --\knn{}},
    {\color{color_PARZ} \bfseries --\parzen{}}, and 
    {\color{color_COV} \bfseries --\cov{}}.
    Here $P \sim \mathcal{N}(0,\mathbb{I}_{d})$ and $ Q \sim \mathcal{N}(\mu \mathbf{1}_{d},\mathbb{I}_{d})$ with $d=64$ dimensions and $\mu=\frac{1}{\sqrt d}\approx.12$ or $\mu=\frac{3}{\sqrt d}\approx.38$. 
    $n=10$K points are sampled using $k=4$ or $k=\sqrt n$ for NN comparison, with or without dataset validation/train split.
    (Each curve is obtained by averaging 10 PR curves from different sets of random samples.)}.
    \label{fig:shift-gauss}
\end{figure*}

Inspired by experiments in related works, we consider the case where $P$ and $Q$ are two Gaussian distributions with increasing shifts. 
Although we sample a large number of points, we run each experiment $10$ times to obtain robust interpretations of the metrics. 
Unlike \citet{kynkaanniemi_ImprovedPrecisionRecall_2019, naeem_ReliableFidelityDiversity_2020, park_ProbabilisticPrecisionRecall_2023}, we once again emphasize that the extreme values of the curves should always equal $1$, as the two Gaussian distributions have a full co-support.

We evaluate four different methods across four different shifts of the fake Gaussian. The resulting curves are presented for two shifts in Figure.~\ref{fig:shift-gauss} and summarize the full experiment results and the corresponding IoU scores in Appendix~\ref{app:add_exp} Table~\ref{tab:mean_iou_shift_k4}. In these settings, the standard deviation across experiments is below $10^{-2}$ on IoU scores. We vary the nearest neighbor parameter $k \in \{4, \sqrt{10K}\}$ and we test our metrics with and without split through the experiment.
As described in Remark~\ref{remark:splitting}, we can apply the splitting framework to the methods from the literature which initially didn't encompass it. In this split setting, for each distribution $P, Q$ with $n$ sampled points, half of them are selected for defining the parameterized family of functions and the other half is used to compute the error rates which yield the Precision and Recall curve.

As it was already highlighted in \citet{naeem_ReliableFidelityDiversity_2020}, we can observe how iPR noticeably underestimates the ground truth in the no-split and k=4 setting (it's initial definition by the original authors).
In some settings, the PR curves differ significantly from the ground truth.
In prior works, the number of nearest neighbors $k$ used for manifold estimation was set to $3$ \citep{kynkaanniemi_ImprovedPrecisionRecall_2019} or $5$ \citep{naeem_ReliableFidelityDiversity_2020}. 
Here, we first use $k = 4$ and then, motivated by Th.~\ref{thm:optimality}, set $k = \sqrt{n}$. 
Interestingly, we observe in the former case that, as expected, estimated PR curves are more pessimistic in the split setting, meaning they underestimate the similarity between the two distributions compared to the validation/train split. This effect is so pronounced that the extreme values of the curves do not reach one, as they theoretically should.
We observe that results are more robust with $k = \sqrt{n}$, and that differences between split and no-split scenarios become marginal.

Given the theoretical guarantees in the split case, we retain the combination of split and $k = \sqrt{n}$ for the subsequent toy experiment (additional results are provided in the Appendix~\ref{ap:gaussian_shifts}).

\setlength{\mywidthhere}{.4\textwidth}
\begin{figure}[!ht]
     \centering
     \begin{subfigure}[b]{\linewidth}  
        \centering
        \includegraphics[width=\mywidthhere,trim={42, 0, 42, 0}, clip]{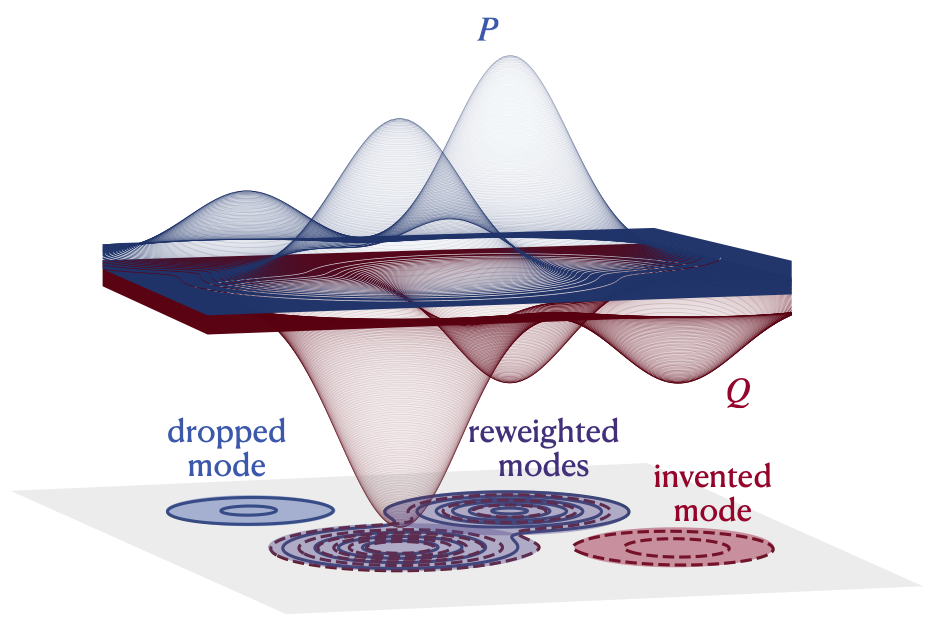}
        \caption{Gaussian Mixture Models $P$ and $Q$ showing mode dropping (only in $P$), mode inventing (only $Q$), and mode re-weighting (in both but distributed differently).
        }
        \label{fig:mode-diffs}
    \end{subfigure}
 \begin{subfigure}[b]{\linewidth}
        \centering
        \includegraphics[width=\mywidthhere,trim={8, 0, 30, 0}, clip]{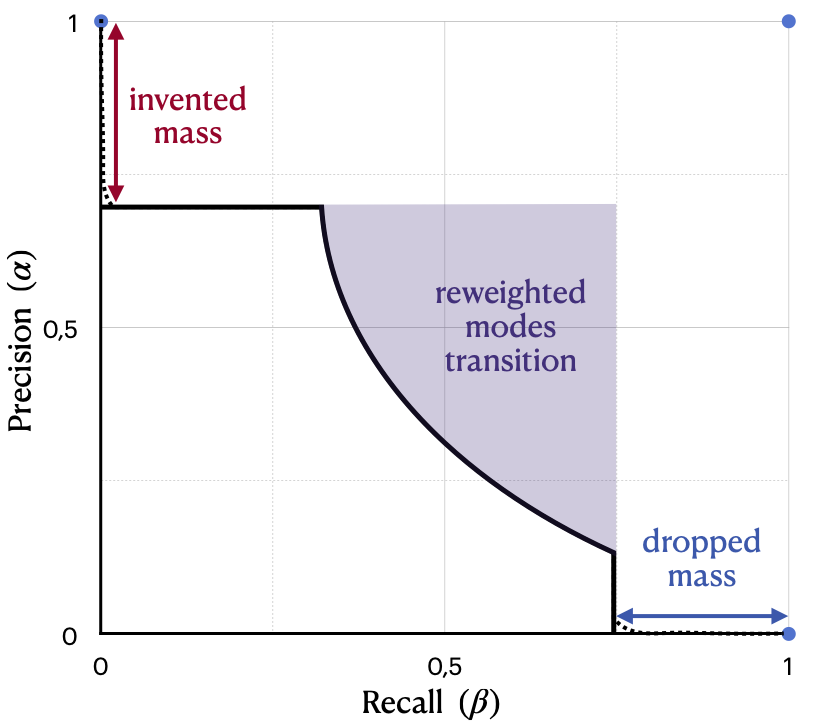}
        \caption{Expected coarse shape of the PR-curve (solid black).  
        Due to the infinite tails of the Gaussian modes, the vertical and horizontal transitions are theoretically smooth and reach $1$ (dashed curve).
        }
        \label{fig:pr-mode-diffs}
    \end{subfigure}

    \begin{subfigure}[b]{\linewidth}
        \centering
        \includegraphics[width=\mywidthhere]{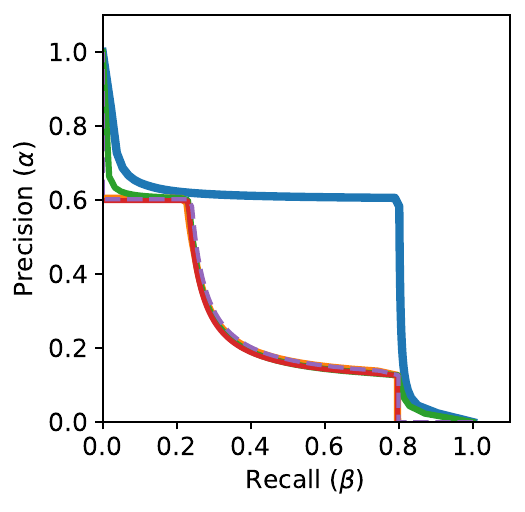}
        \caption{\textbf{Comparing two Gaussian mixtures}. %
        The ground truth PR curve (
        {\color{color_GT} \bfseries \normalfont -~-\textsc{GT}})
        is compared to empirical estimates from various NN-classifiers:
        {\color{color_IPR} \bfseries --\ipr{}},
        {\color{color_KNN} \bfseries --\knn{}},
        {\color{color_PARZ} \bfseries --\parzen{}}, and 
        {\color{color_COV} \bfseries --\cov{}}. Notice that the method \cov{} and \knn{} overlap (and are very close to the ground truth).
        }
        \label{fig:GMM-dim64}
    \end{subfigure}
    \caption{Comparison of Gaussian Mixture Models.}
    \label{fig:gmm_exp}
\end{figure}

\subsubsection{Gaussian mixture models} %
\label{sec:exp_GMM}

As argued by \citet{luzi_EvaluatingGenerativeNetworks_2023}, neural representations---such as Inception features---are better approximated by Gaussian mixture models (GMMs) than by pure Gaussians as stated in the original paper introducing FID. 
This more general modeling also allows us to illustrate key phenomena: mode \emph{dropping} (a mode present in $P$ but not in $Q$), mode \emph{invention} (a mode present in $Q$ but not in $P$), and mode \emph{re-weighting} (shared modes between $P$ and $Q$ with different weights), as shown in Figure~\ref{fig:mode-diffs}. 
The corresponding theoretical PR curves are depicted in Figure~\ref{fig:pr-mode-diffs}. 
The distributions $P$ and $Q$ used for experiments are here defined as a mixture of 4 modes: $P = \sum_{\ell} p_\ell\, \mathcal{N}(\mu_\ell \mathbf{1}_{d}, \mathbb{I}_{d})$ and $Q = \sum_{\ell} q_\ell\, \mathcal{N}(\mu_\ell \mathbf{1}_{d}, \mathbb{I}_{d})$, with $d = 64$ and $\mu_\ell \in \{0, -5, 3, 5\}$. 
The weights differ between $P$ and $Q$: $p_\ell \in \{0.2, 0.2, 0.6, 0\}$ and $q_\ell \in \{0, 0.5, 0.1, 0.4\}$. 
In this synthetic experiment, we sample points from two GMMs and apply splitting with $k = \sqrt{n}$ (see Figure~\ref{fig:GMM-dim64}). 

The \knn{}, \kde{}, and \cov{} methods perform well relative to the ground truth, while \ipr{} tends to overestimate the PR curve, particularly failing to capture the re-weighting transitions. 
Transitions to the extreme values are noticeably not well estimated by the \kde{} approach. 
Indeed, the \kde{} classifier, with non-adapting kernel width, is less discriminative on Gaussian tail densities, while being more robust than \ipr{}.
The tails effects are amplified for \ipr{} because low density samples produce large kNN balls.
Additional experiments described in \ref{ap:gaussian_mixture} confirm these observations.

\subsubsection{Variability study}\label{sec:exp_variability}
As the access to a large number of samples might be a bottleneck, the question of the variability of the metrics' output might be risen. 
In Figure~\ref{fig:varibility-nbpoints}, the impact of the size sample $n$ on the variability of the evaluation curves is shown.
These figures illustrate the consistency of the proposed method based on robust classifiers when increasing the number of samples (Theorem~\ref{thm:optimality}). Empirically, using $10K$ points reduces sufficiently the variability to make comparison between curves reliable, 
which is in line with the standard usage for generative model evaluation \citep{heusel_GANsTrainedTwo_2017,sajjadi_AssessingGenerativeModels_2018}.

Figure~\ref{fig:varibility-nbpoints} complements Section~\ref{sec:exp_variability} about variability.
Average curves are obtained by computing the empirical mean of $N = 100$ PR curves obtained different random $n$-samples (with $n=10^4$), \emph{i.e.} 
 $$
    (\bar \alpha(\lambda) , \bar \beta(\lambda)) 
    = \frac{1}{n}\sum_{i=1}^n  (\hat \alpha_i(\lambda) , \hat \beta_i(\lambda)) ,
 $$
 where $(\hat \alpha_i(\lambda), \hat\beta_i(\lambda))$ are the independent estimates (associated with the separate runs).
Deviation from average curves are materialized with two curves
$$
 (\alpha_{\pm\sigma}(\lambda) , \beta_{\pm\sigma}(\lambda))
 =
    (\bar \alpha(\lambda)\pm \sigma_\alpha(\lambda)  , \bar \beta(\lambda)\pm \sigma_\beta(\lambda)  ) ,
$$
 with empirical estimator
$$
    \sigma_\alpha(\lambda)^2 
    = \frac{1}{n}\sum_{i=1}^n  (\hat \alpha_i(\lambda)-\bar \alpha(\lambda))^2 .
$$
and similarly for the Recall deviation $\sigma_\beta$.
\textbf{NB:} By construction, for a fixed $\lambda$, $\hat\alpha_i(\lambda)=\lambda\hat\beta_i(\lambda)$. This ratio is automatically inherited for the average curves $(\bar\alpha(\lambda),\bar\beta(\lambda))$ and the standard deviation curves $(\alpha_{\pm\sigma}(\lambda) , \beta_{\pm\sigma}(\lambda))$. Whereas the methods' Precision and Recall have different biases, their deviation are very similar in their behavior as shown in Figure~\ref{fig:varibility-nbpoints}.
    
\setlength\myheight{0.4\textwidth}
\begin{figure}[htb]
    \centering
    \begin{tabular}{cc}
        \includegraphics[height=\myheight]{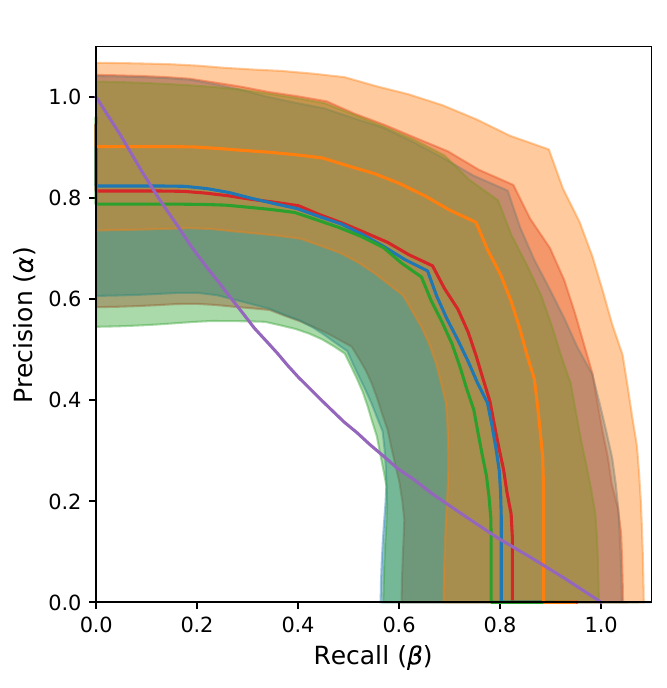}
        &
        \includegraphics[height=\myheight]{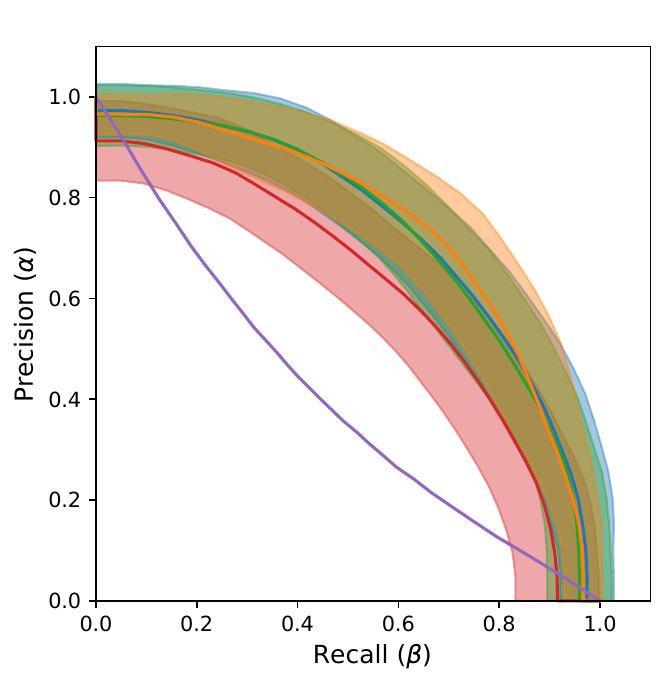}
        \\
        $n=10$ & $n=100$ \\
        \includegraphics[height=\myheight]{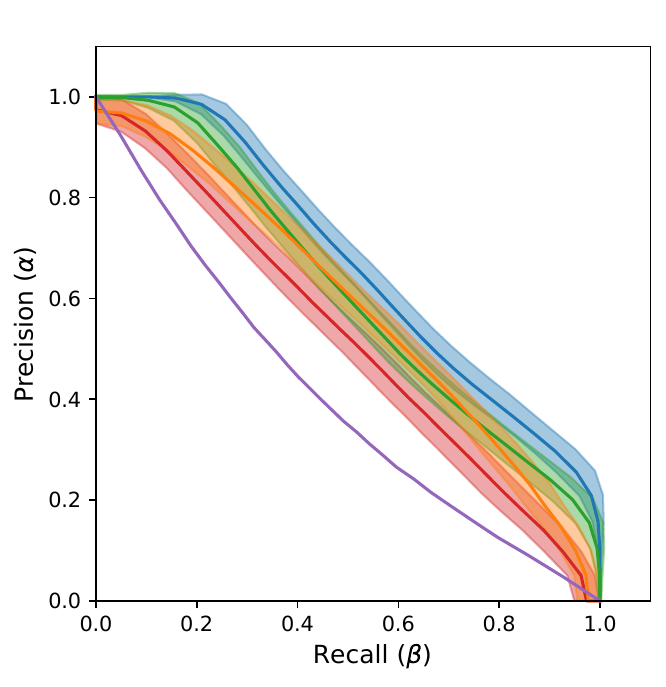}
        &
        \includegraphics[height=\myheight]{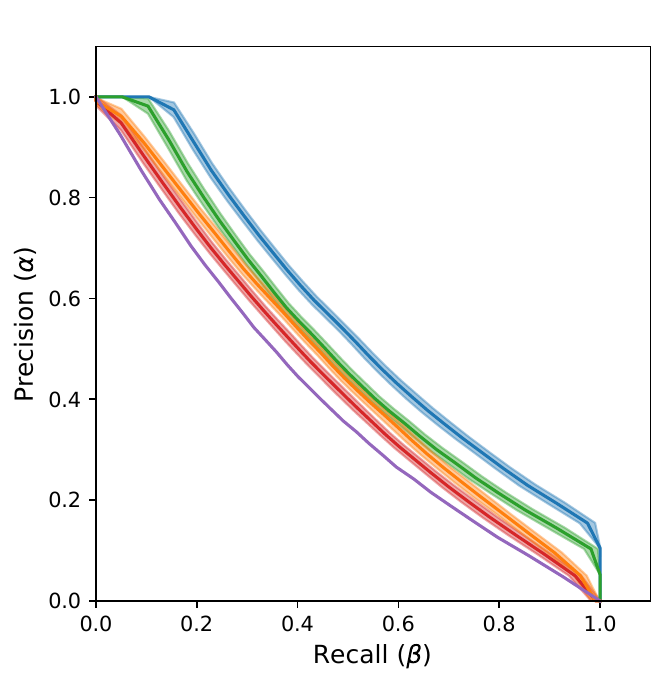}
         \\
        $n=1,000$ & $n=10,000$ 
    \end{tabular}
    \caption{\textbf{Influence of sample size $n$.}
    The setting is the same as Figure~\ref{fig:shift-gauss}
    for a translation of $\mu=.21$ between two Gaussian in dimension $d=64$ (with splitting and $k=\sqrt n$).
    Solid (respectively transparent) curves correspond to the empirical average (resp. deviations) of $100$ PR curves computed from random samples. %
    In this experiment, we use splitting with a factor $0.5$.
    }
    \label{fig:varibility-nbpoints}
\end{figure}

\subsection{Experiment with StyleGAN}
\label{section:stylegan-exp}

In this section, we test our metrics in the scenario of a generative model trained on some ``real-world'' data. To do so, we use the same approach as in \citet{kynkaanniemi_ImprovedPrecisionRecall_2019} where the model StyleGAN\footnote{In our case, we use \href{https://github.com/NVlabs/stylegan2-ada-pytorch/blob/main/generate.py}{StyleGAN v2 implementation} \citep{styleganV2}.} is trained on the Flickr Face High-Quality dataset (FFHQ).

\paragraph{Image Embedding} %
In the previous experiments on synthetic data, we noticed how high dimension made the estimation of Precision and Recall metrics challenging.
However, images from the FFHQ dataset live in very high dimension (\emph{i.e.} $d\approx 3.10^6$).
As it is standard in the literature, image samples are compared via neural representations from pre-trained deep networks.
In addition to mitigating the curse of dimensionality by mapping images to a new space of smaller dimension, it allows to define relevant metrics that can be aligned with human perception (such as LPIPS).
In this experimental section, we make use of several popular feature embedding (supervised or non-suspervised): InceptionV3 \citep{szegedy_RethinkingInceptionArchitecture_2016} ($d=2048$), VGG16 \citep{simonyan_VeryDeepConvolutional_2015} ($d=4096$), VGG16 with random weights as proposed in \cite{naeem_ReliableFidelityDiversity_2020}, and DINOv2 \citep{oquab_DINOv2LearningRobust_2024} ($d=384$).

\paragraph{Filtering of generated samples \emph{via} truncation}
Recall that StyleGAN is a latent generative model built upon a specific feed-forward neural network. Starting from a constant input, convolutional layers are modulated by a latent variable that is obtained in two stages: a Gaussian random variable $\mathbf z \sim \mathcal{N}(\mathbf{0}_{d}, \mathbb{I}_{d})$ is first sampled (Z space) that is then encoded using a MLP mapping (W space): $\mathbf w = M( \mathbf z)$.
Note that in StyleGAN-V1 \citep{styleganV1} additional random noise is also employed for modulation.
In \citet{kynkaanniemi_ImprovedPrecisionRecall_2019}, several truncation methods parametrized by a scalar $\Psi$ are investigated to improve automatically the quality of generated samples.
The truncation parameter $\Psi$ indicates how close to the average the latent variables of the model are linearly remapped after being sampled from a normal distribution. 
The most efficient one in terms of a trade-off between simplicity and efficiency consists in using an interpolation to the mean:
$
\mathbf w \mapsto \Psi (\mathbf w -  \mathbf{\bar w}) + \mathbf{\bar w}
$, where $\mathbf{\bar w}=\mathbb{E} M(\mathbf z)$.
In the Improved Precision and Recall paper, the authors' intuition was that as $\Psi \longrightarrow 0$, the generated distribution collapses to a single point (the point with highest probability density) thus always generating---slight variations of---the same image but which is the most probable. Conversely, when $\Psi \longrightarrow 1$, the model is less constrained and its latent variables are varied, yielding images with high diversity yet with more artifacts.

We reproduced the results of the truncation experiment on StyleGAN from \citet{kynkaanniemi_ImprovedPrecisionRecall_2019} using our metrics.
As described in Section~\ref{section:distilling-curve}, practitioners might want to have Precision and Recall curves summarized to scalar metrics (Area Under Curve or AuC, $F_b$, $\alpha@\varepsilon$). We computed those synthetized metrics including the extreme value $\alpha_\infty$ as done in the StyleGAN experiment of \citet{kynkaanniemi_ImprovedPrecisionRecall_2019}).  %
These metrics are shown in Figures~\ref{fig:real-exp-ipr-stylegan}, ~\ref{fig:real-exp-coverage-stylegan} and~\ref{fig:real-exp-knn-stylegan} as a function of the truncation parameter $\Psi$, for respectively \ipr{}, \cov{} and \knn{}.

\
\begin{figure}[htbp]
    \centering
    \begin{subfigure}[b]{0.49\textwidth}
      \includegraphics[width=\textwidth,trim={0 0 0 20},clip]{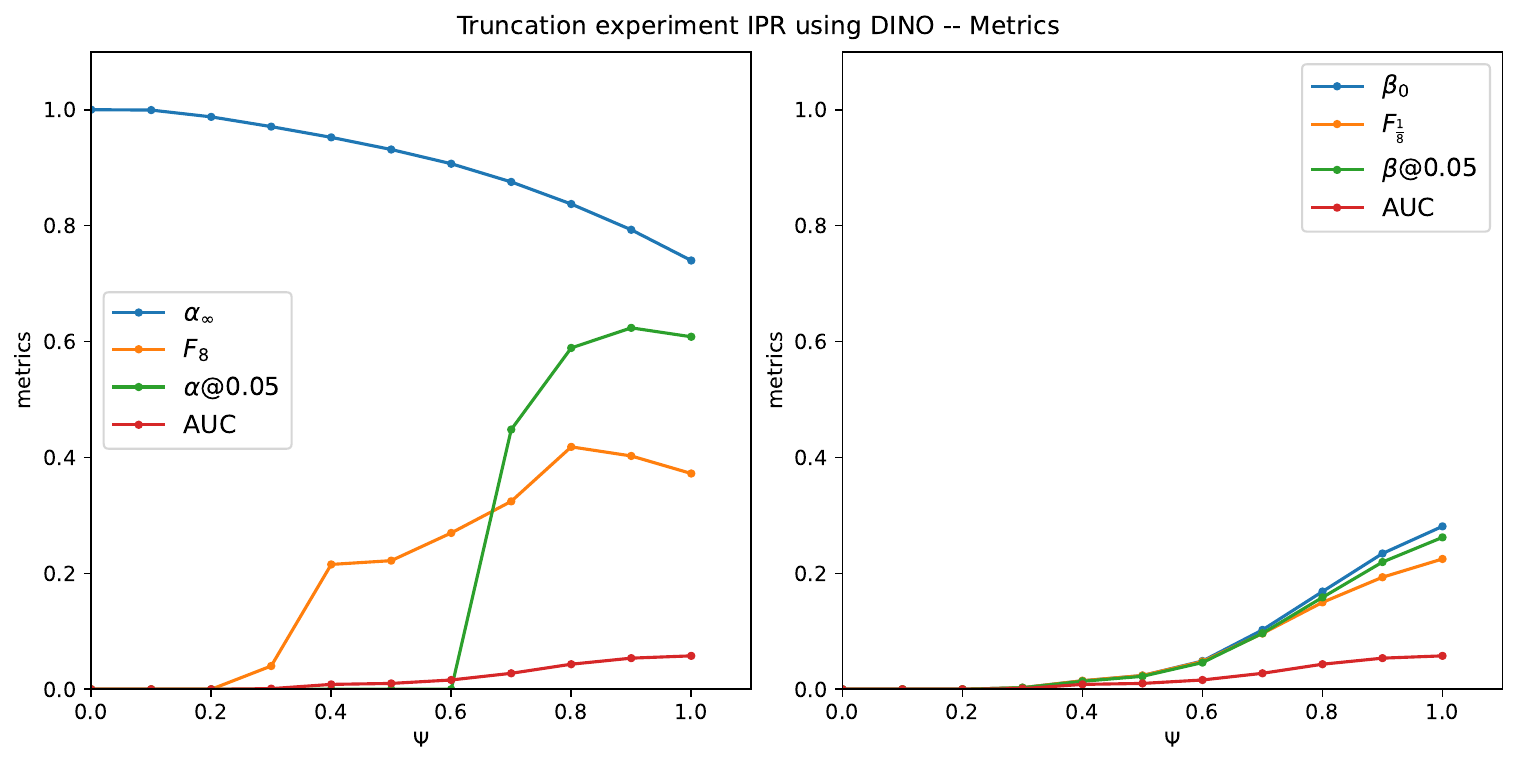}
        \caption{DINOv2}
    \end{subfigure}
    \hfill
    \begin{subfigure}[b]{0.49\textwidth}
\includegraphics[width=\textwidth,trim={0 0 0 20},clip]{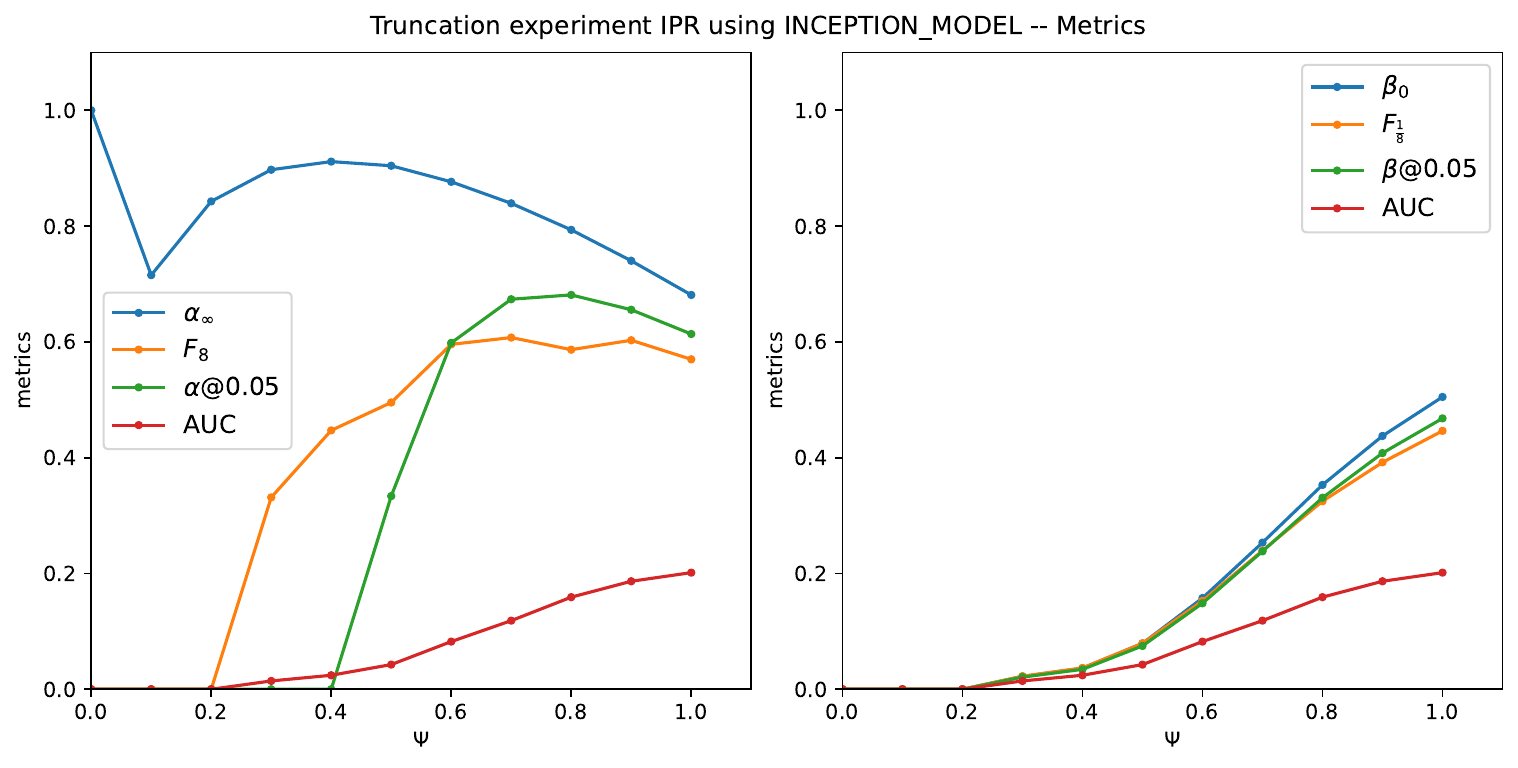}
        \caption{Inception}
    \end{subfigure}
    \vskip\baselineskip
    \begin{subfigure}[b]{0.49\textwidth}
 \includegraphics[width=\textwidth,trim={0 0 0 20},clip]{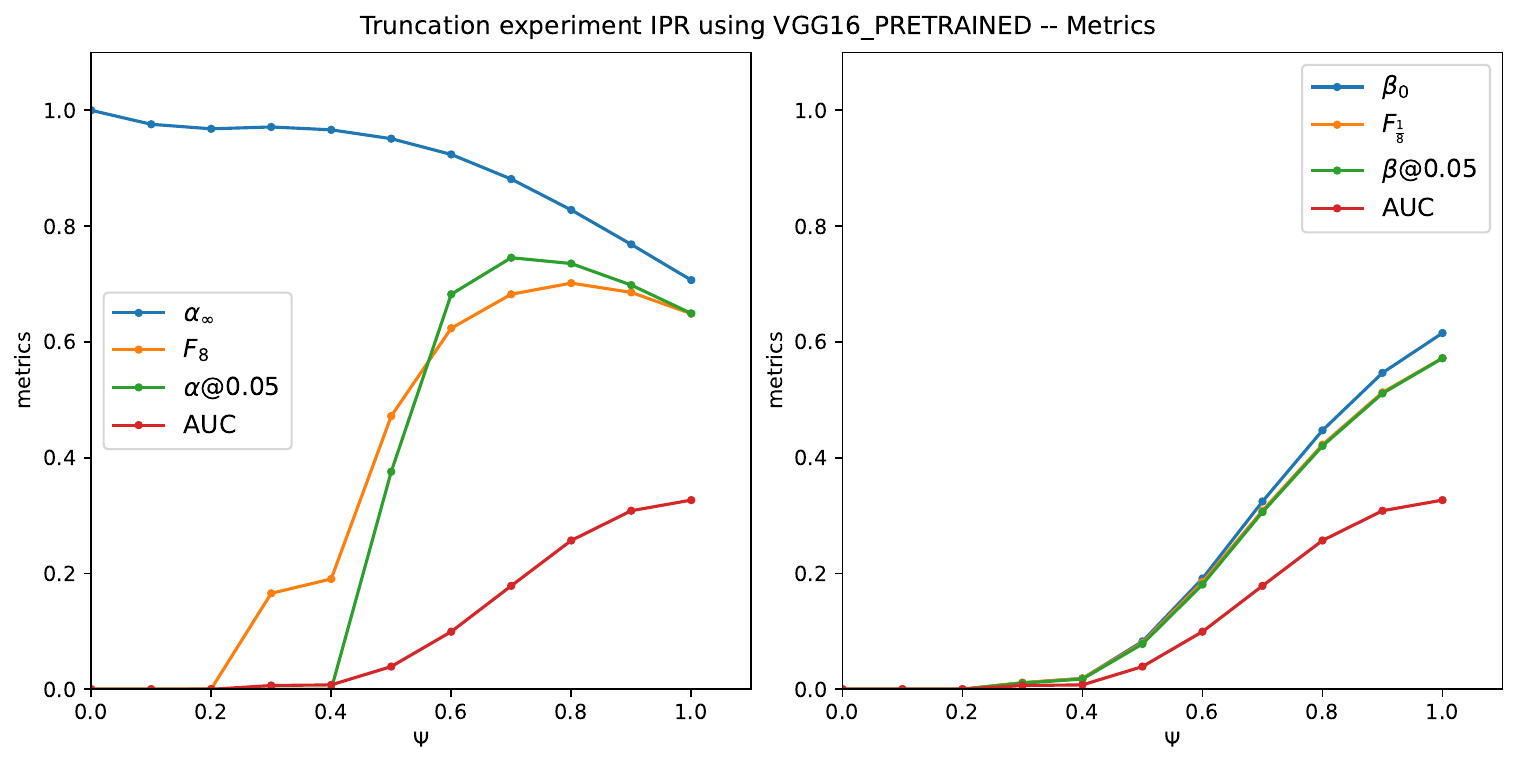}
        \caption{VGG pretrained}
    \end{subfigure}
    \hfill
    \begin{subfigure}[b]{0.49\textwidth}
        \includegraphics[width=\textwidth,trim={0 0 0 20},clip]{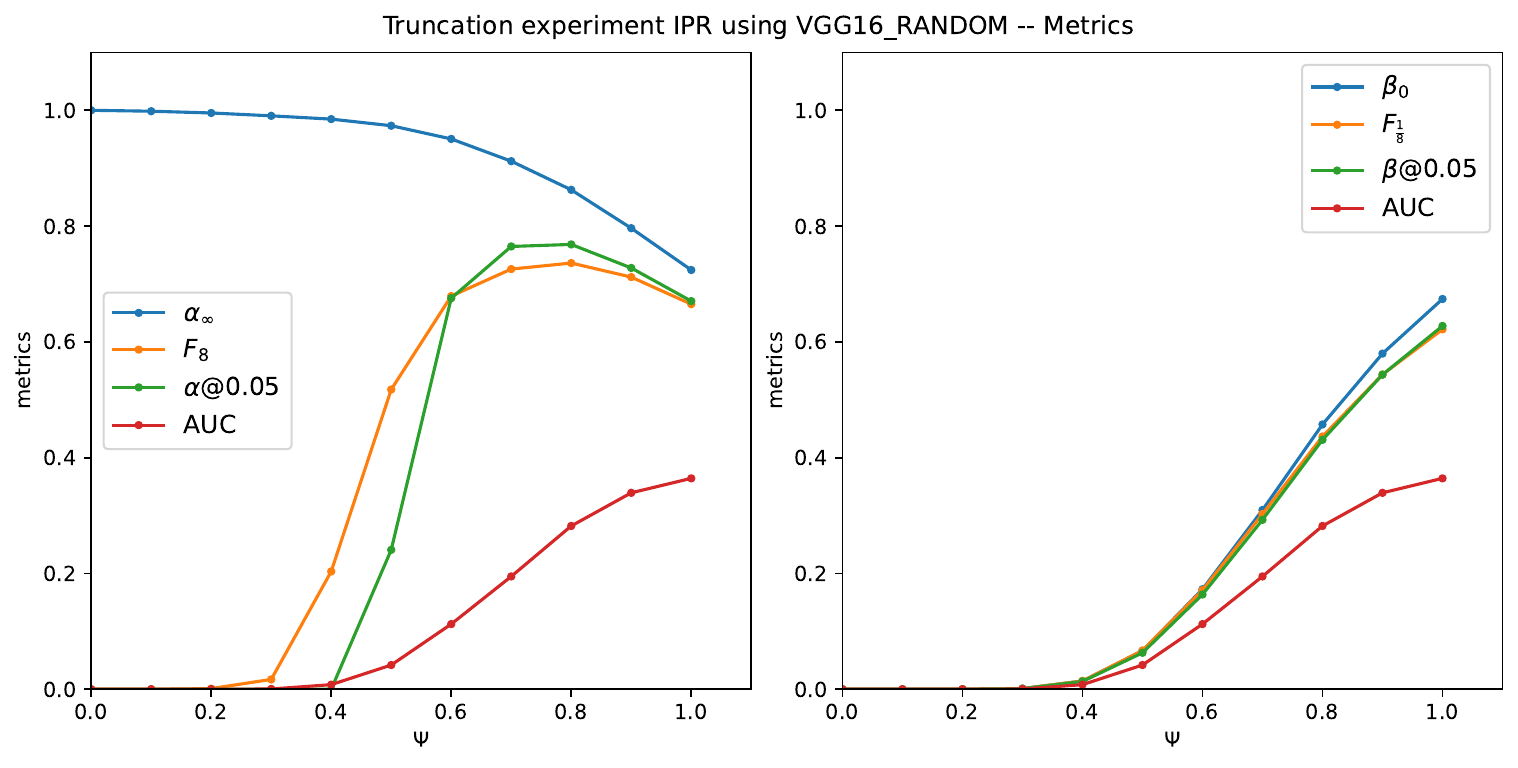}
        \caption{VGG random}
    \end{subfigure}
    \caption{Truncation experiment with the \textbf{IPR} metric, using the same parameters as in the original article.}   \label{fig:real-exp-ipr-stylegan}
\end{figure}

\begin{figure}[htbp]
    \centering
    \begin{subfigure}[b]{0.49\textwidth}
        \includegraphics[width=\textwidth,trim={0 0 0 20},clip]{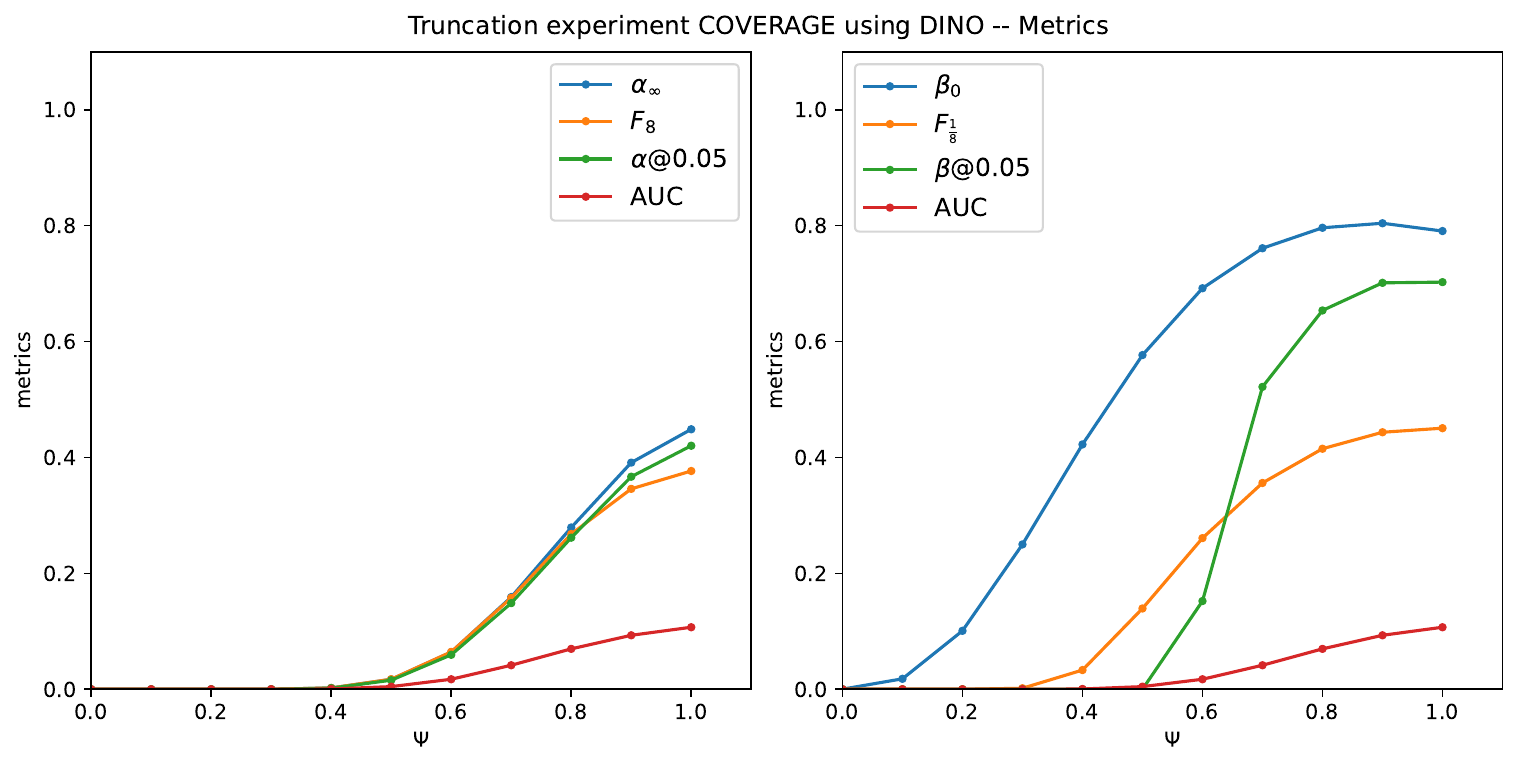}
        \caption{DINOv2}
    \end{subfigure}
    \hfill
    \begin{subfigure}[b]{0.49\textwidth}
        \includegraphics[width=\textwidth,trim={0 0 0 20},clip]{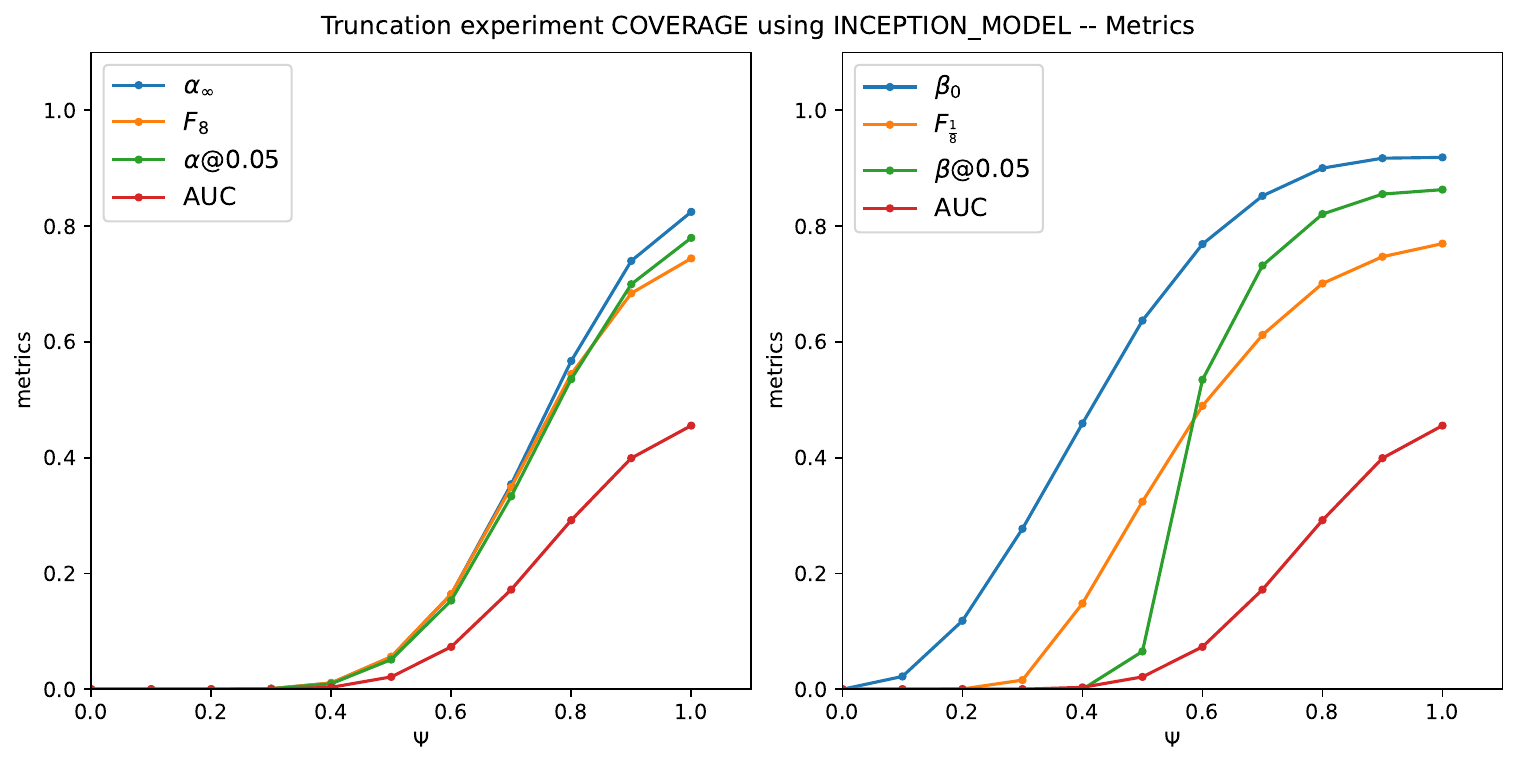}
        \caption{Inception}
    \end{subfigure}
    \vskip\baselineskip
    \begin{subfigure}[b]{0.49\textwidth}
        \includegraphics[width=\textwidth,trim={0 0 0 20},clip]{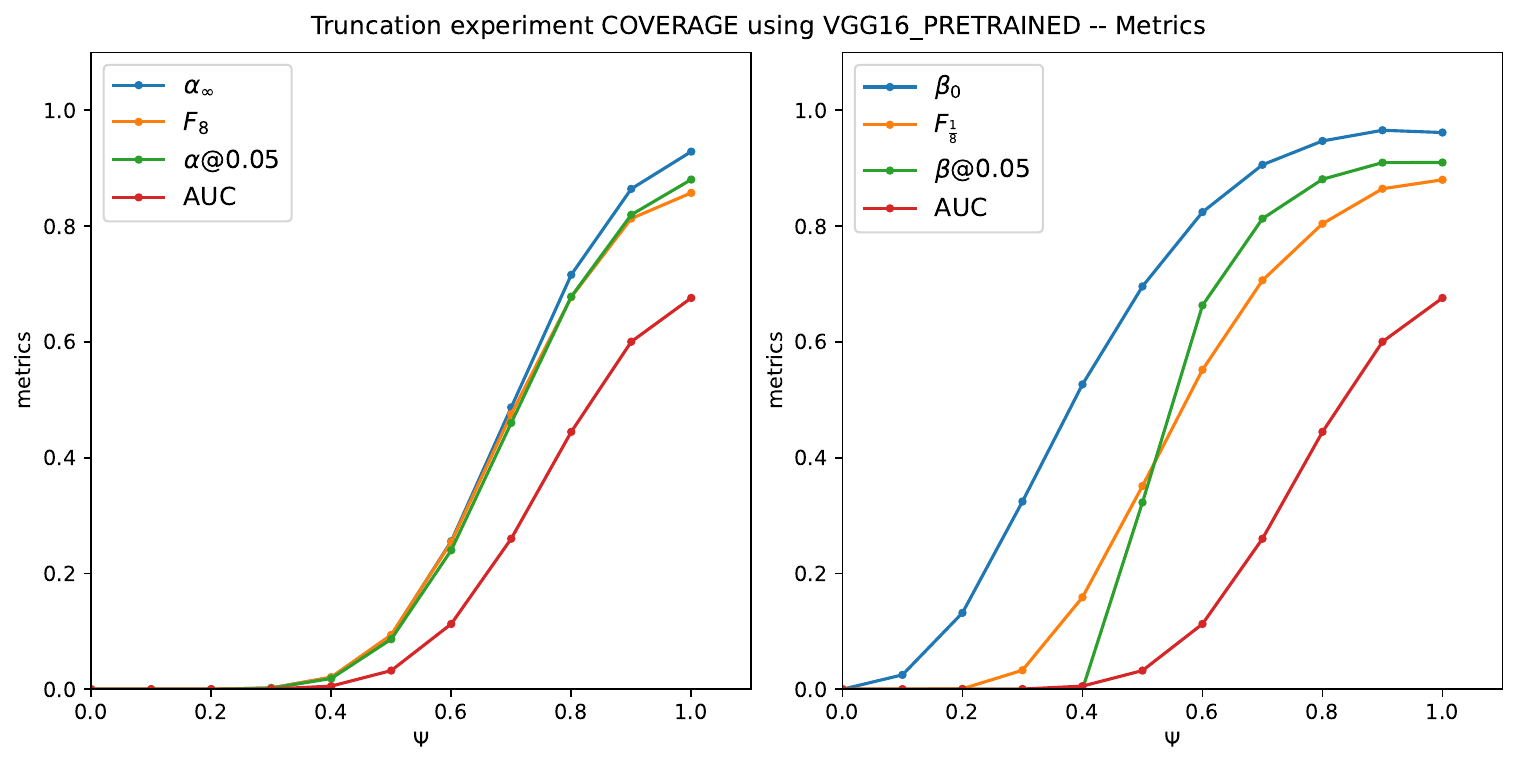}
        \caption{VGG pretrained}
    \end{subfigure}
    \hfill
    \begin{subfigure}[b]{0.49\textwidth}
  \includegraphics[width=\textwidth,trim={0 0 0 20},clip]{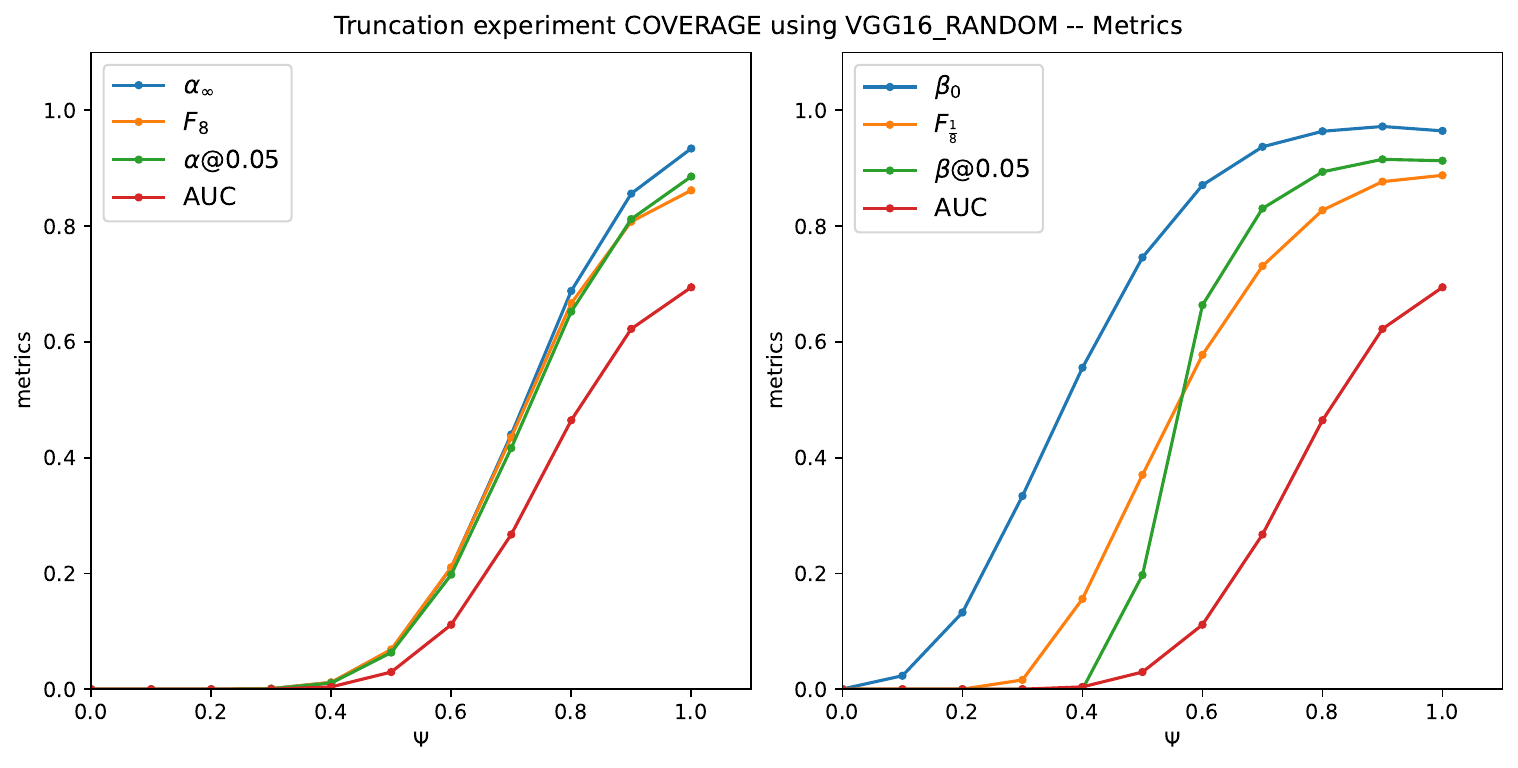}
        \caption{VGG random}
    \end{subfigure}
    \caption{Truncation experiment with the \textbf{Coverage} metric, using the same parameters as in the original article.}    \label{fig:real-exp-coverage-stylegan}
\end{figure}

\begin{figure}[htbp]
    \centering
    \begin{subfigure}[b]{0.49\textwidth}
        \includegraphics[width=\textwidth,trim={0 0 0 20},clip]{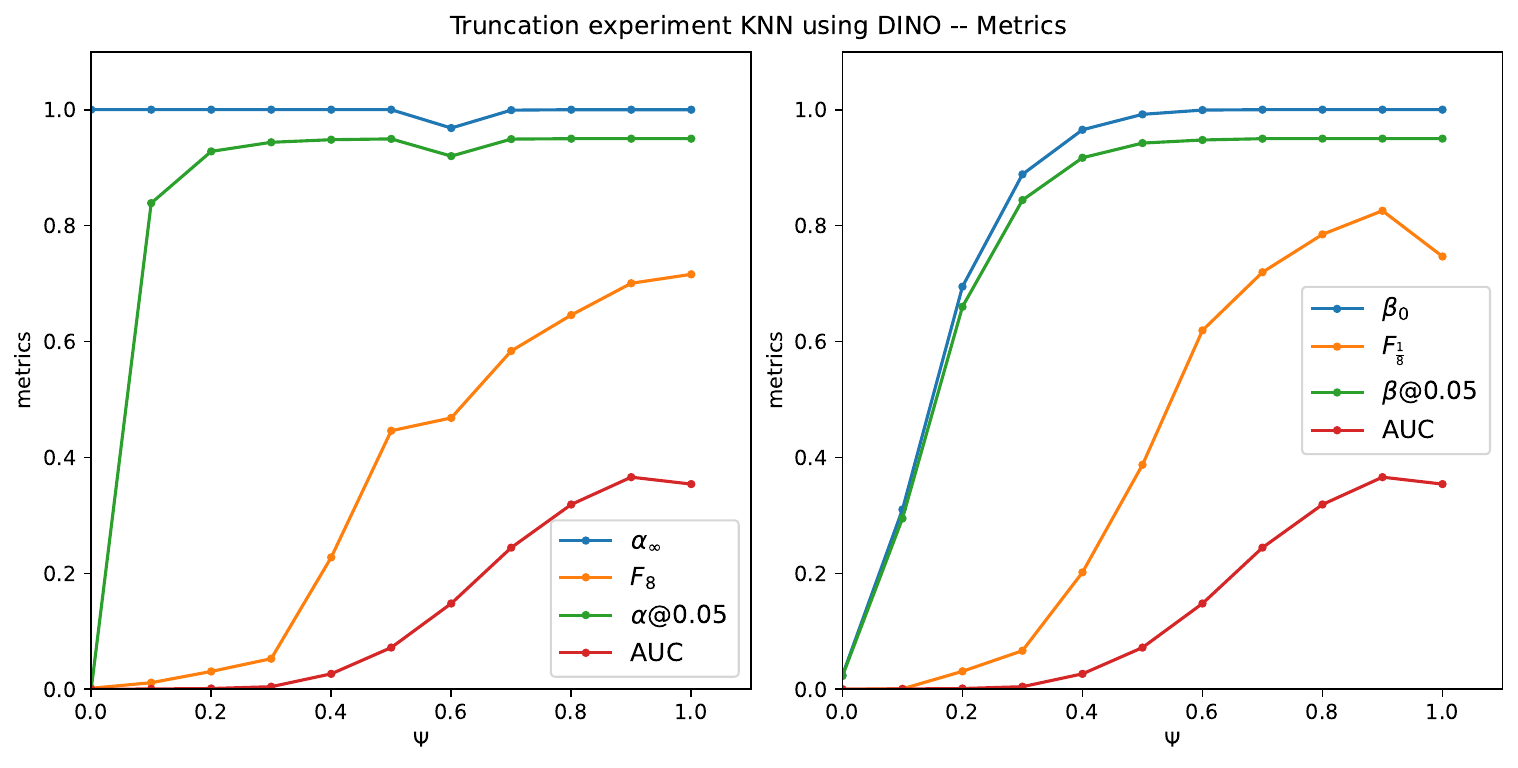}
        \caption{DINOv2}
    \end{subfigure}
    \hfill
    \begin{subfigure}[b]{0.49\textwidth}
        \includegraphics[width=\textwidth,trim={0 0 0 20},clip]{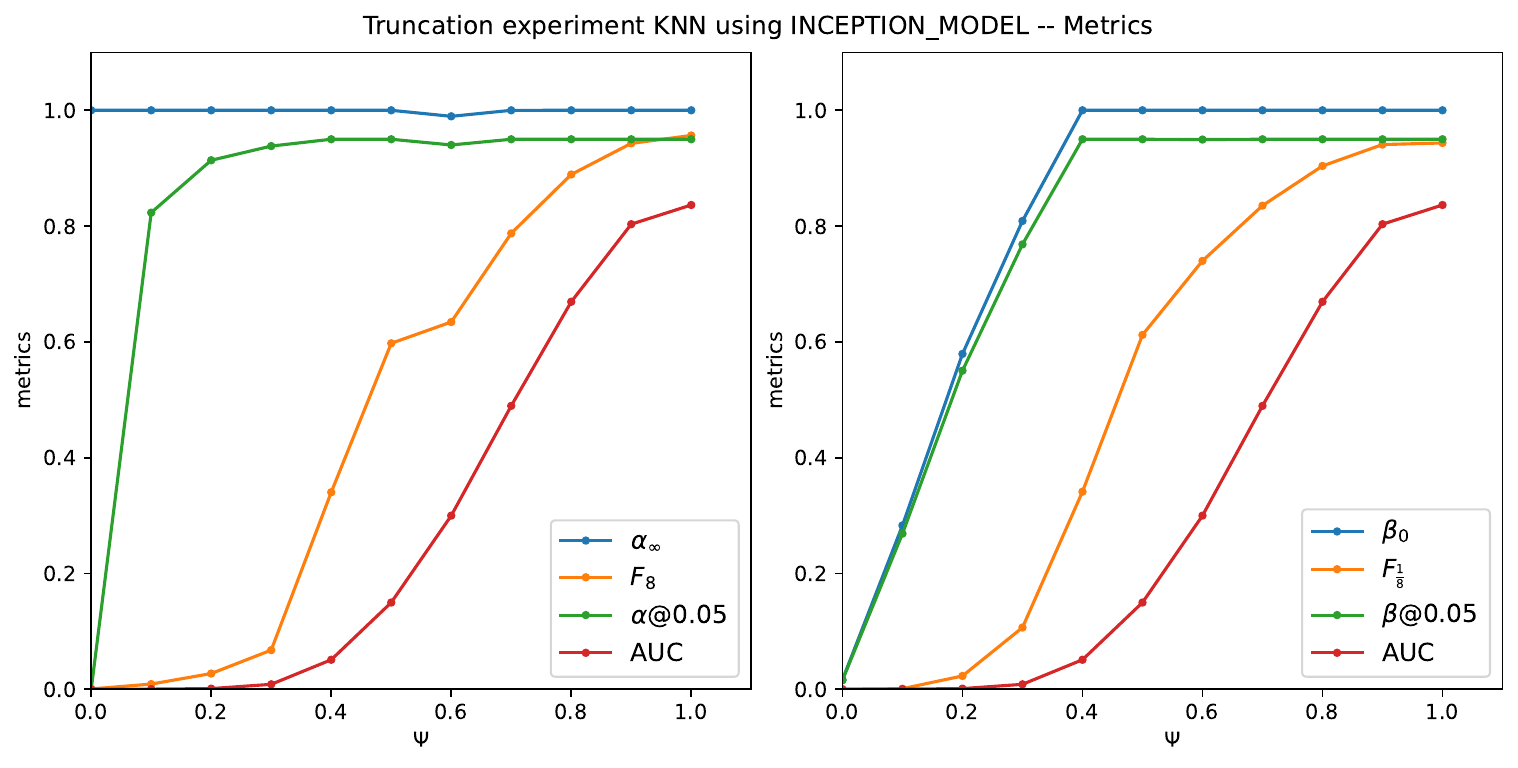}
        \caption{Inception}
    \end{subfigure}
    \vskip\baselineskip
    \begin{subfigure}[b]{0.49\textwidth}
        \includegraphics[width=\textwidth,trim={0 0 0 20},clip]{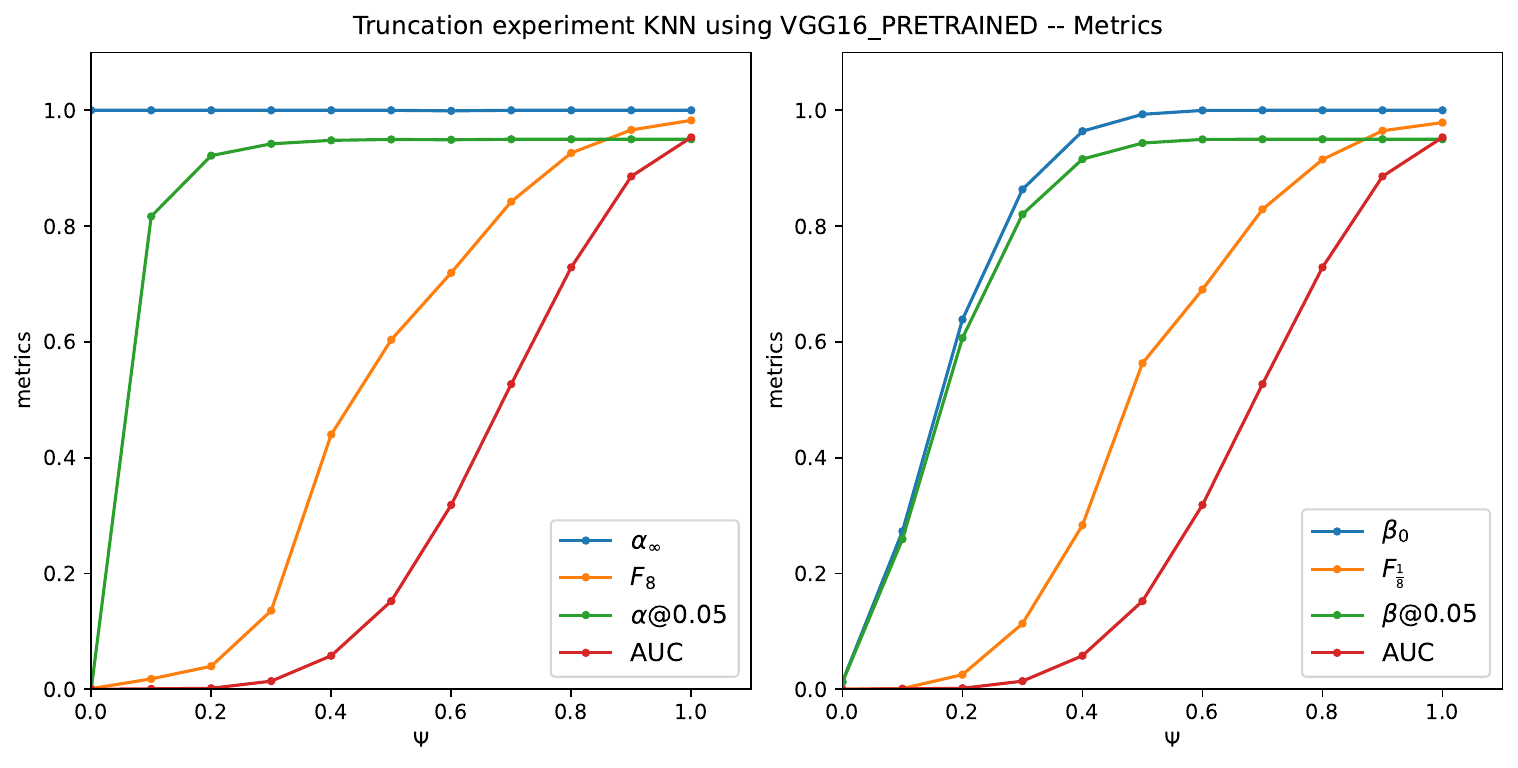}
        \caption{VGG pretrained}
    \end{subfigure}
    \hfill
    \begin{subfigure}[b]{0.49\textwidth}
        \includegraphics[width=\textwidth,trim={0 0 0 20},clip]{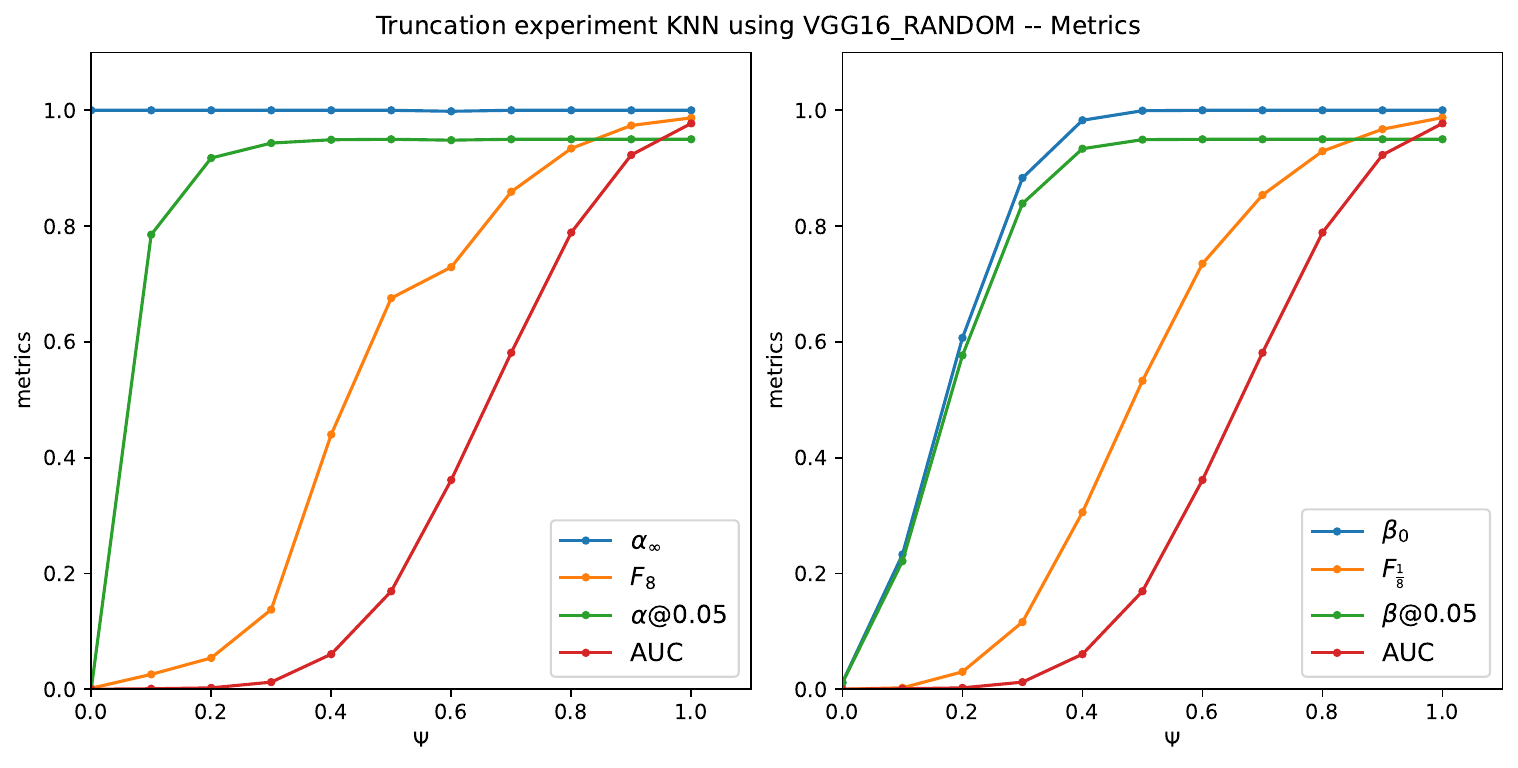}
        \caption{VGG random}
    \end{subfigure}
    \caption{Truncation experiment with the \textbf{\knn{}} method (split=0.5, k=$\sqrt{n}$).}
    \label{fig:real-exp-knn-stylegan}
\end{figure}

Regarding the extreme Precision metric $\alpha_\infty$ for the \ipr{} method, as expected we obtain the same global behavior (mostly monotonic decreasing) as in \citet{kynkaanniemi_ImprovedPrecisionRecall_2019}.
Notice that the only slight difference is for the Inception embedding, which is probably due to the fact that 
we use the deterministic StyleGAN-V2 which makes generated samples collapse to a single image when $\Psi\rightarrow 0$, making estimation of Precision more challenging.

The comparison of \cov{} with \knn{} and \ipr{} highlights a different behavior for $\alpha_\infty$ when $\Psi\rightarrow 0$ which we argue to be the desired one, as done previously by \citet{djolonga_PrecisionRecallCurvesUsing_2020}.
In this setting, we find ourselves comparing the real distribution $P$, to the distribution $Q_{\Psi}$, which collapses to a Dirac.
Using the cosupport concept introduced in Definition.~\ref{def:support-cossuport}, the extreme Precision writes $\alpha_\infty(P,Q)=Q(\cosupp(P,Q))$. 
In this case $\cosupp(P,Q)=\emptyset$ yielding $\alpha_\infty(P,Q)=0$, which is not the value yielded by \ipr{}. On the other hand, the the other metrics all yield $\alpha_\infty=0$ when $\Psi=0$ which indeed is the expected behavior.

Note that other metrics (AuC, $F_b$ and $\alpha @ \varepsilon$) share the same behavior across all embeddings (DINOv2, VGG, Inception), and all methods, including \ipr{}.

\paragraph{Impact of the embedding model}
Interestingly, a randomly initialized VGG embedding yields very similar result to what is obtained with DINOv2 or VGG networks trained on ImageNet, as previously reported by \citet{naeem_ReliableFidelityDiversity_2020}. 
In the mean time, we might question the relevance of using models pretrained on simple and cropped datasets such as ImageNet to evaluate models on complex datasets. Using a random model seems relevant as it does not induce a bias from any given dataset.
To go further, we might wonder how many output dimensions are necessary to capture all the relevant information from the images. Consistently removing a large number of dimensions from the output of a randomly initialized neural network might help yielding relevant results while not suffering from the curse of dimensionality.

\subsection{Hybrid Experiment: Sampling Gaussians from StyleGAN Statistics}
\label{sec:hybrid-exp}

In this section we present a hybrid experiment standing between the ``real'' StyleGAN truncation experiment and the ``toy experiment'' where we made two gaussian distributions evolve with respect to one another.

Due to the fact that the DINOv2 model produces embeddings in lower dimension than the other models---thus making computations less intensive---we will use DINOv2 embeddings for the remainder of this experiment.
In this experimental setting, we select the 50 000 samples embeddings that come from the FFHQ dataset and are ran through our embedding model. We will project all the features in the same subspace which is yielded by the first $d$ principal components of FFHQ embeddings. Our experiment focuses on varying the number of selected dimensions $d$. For DINOv2 dimension $d$ lies in $[1, 384]$.

For each given truncation value of $\Psi \in \{0.0, 0.1,..., 1.0\}$, we select the embeddings of the 50 000 samples generated on StyleGAN (discussed in \ref{section:stylegan-exp}), trained on FFHQ and truncated with parameter $\Psi$. We project the given embeddings in the same $d$ dimension subspace determined by FFHQ. 
In this setting, we therefore first reduce the complexity of the problem and transpose it in a controlled setting where we associate each distribution to a Gaussian distribution determined by the mean and covariance of the projected embeddings. This experimental setting makes it possible to compare the PR curves we compute to the ground truth which is estimated by a Monte Carlo simulation applied on the Bayes likelihood ratio classifier---which is explicitly known.

For each generated curve which is defined by a subspace dimension $d$, a truncation value $\Psi$ and a method $M$, we compare the method's curve to the ground truth estimation using the intersection over union (IoU) score. Given a value of $\Psi$, we propose to denote $G^d_\Psi$ the multivariate and non isotropic Gaussian distribution yielded by the projected embeddings in dimension $d$. Similarly, $G^d_{\text{FFHQ}}$ describes the pending FFHQ distribution in dimension $d$.

We propose three different experimental settings, detailed in the following table, each representing a particular scenario designed to reveal how certain approaches are favored or disadvantaged under specific conditions.

\begin{center}
\begin{tabular}{|c|c|c|c|c|}
\hline
\textbf{Description}&\textbf{Section} & $P^d$ & $Q^d$ & \textbf{Favorable Bias} \\
\hline
FFHQ vs truncated distribution & \ref{sub-sec:exp-DINO-stylegan} & $G^d_{FFHQ}$ & $G^d_\Psi$ & underestimation \\
Distribution vs itself &\ref{sub-sec:exp-DINO-self} & $G^d_\Psi$ & $G^d_\Psi$ & overestimation \\
Differently weighted GMMs &\ref{sub-sec:exp-DINO-gaussian-mixture}  & $GMM_{\Psi,\Psi'}$ & $GMM_{\Psi,\Psi''}$ & unknown \\
\hline
\end{tabular}
\end{center}
\

\subsubsection{Comparing projected FFHQ to projected truncated distribution}
\label{sub-sec:exp-DINO-stylegan}

In this setting, we transpose exactly the previous StyleGAN experiment to a controlled setting where, for a given method we compare $G^d_\Psi$ to $G^d_{FFHQ}$. As we now know, in high dimension, the curves are more complex to estimate. Besides, because Gaussians in high dimensions tend to be singular with respect to one another, the ground truth curve tends towards 0. Therefore this experiment \textbf{favors methods that tend to underestimate the PR curve} as it is the case for \citet{kynkaanniemi_ImprovedPrecisionRecall_2019}. 

Similarly to the previous workflows we presented, we compute the IoU score between the generated curves and the estimated ground truth and average the various computed IoU over different $\Psi$ values.

The results of this experiment are shown in Figure \ref{fig:dino-gaussian-samples}. We can observe how indeed the dimension of the ambient space makes it in practice hard to estimate Precision and Recall, especially for the extreme values. As we select more and more eigenvalues, the IoU score effectively decreases until reaching approximately 0. For example if we focus on the \textit{\knn{} with split} method we introduced earlier on, we see that when selecting the 2 largest eigenvalues, the IoU score is around 0.9 and this value decreases to as low as 0.3 when selecting the largest 50 eigenvalues. 
\ipr{} turns out to be a clear winner in this setting as it is known to underestimate the ground truth PR.

\begin{figure}[htbp!]
    \centering
    \includegraphics[width=\mywidth,trim={0 0 0 20},clip]{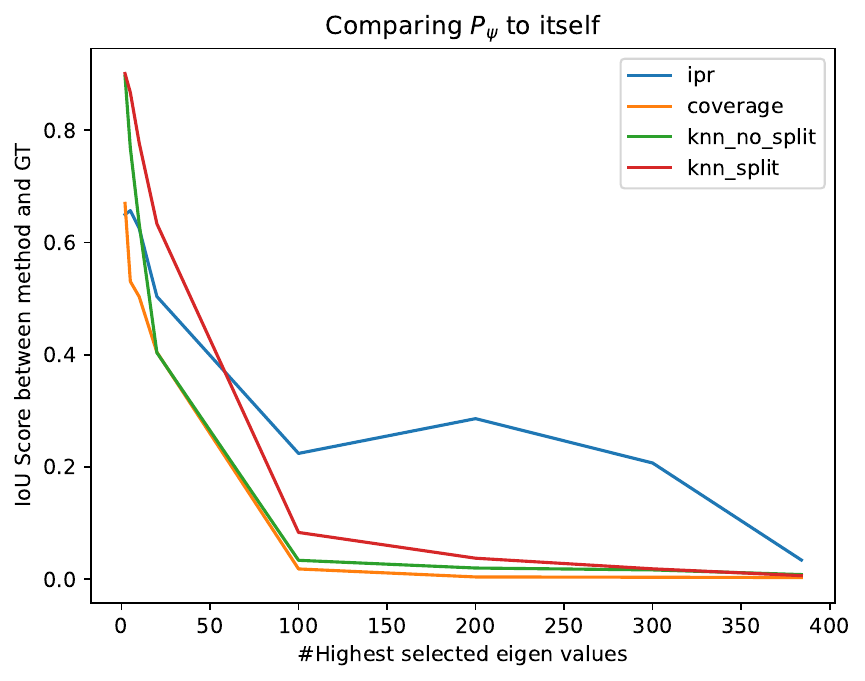}
    \caption{Evolution of IoU score between ground truth and various Precision and Recall metrics based on Gaussian samples of DINOv2.}
    \label{fig:dino-gaussian-samples}
\end{figure}

\subsubsection{Comparing a truncated and projected distribution to itself}
\label{sub-sec:exp-DINO-self}
In this setting, for each truncation $\Psi$ value, the generated distribution $G^d_\Psi$ is compared to itself. We know that in this setting, \ipr{} tends to considerably underestimate the PR curve which should be the entire going through $(1,1)$. 
This experiment should leave an advantage to methods that tend to overestimate the true PR curve. 

The results of this experiment are shown in Figure \ref{fig:dino-self}. We indeed observe that \ipr{} decreases very fast and that the other methods give more appropriate results. In this setting, we illustrate the consistency of the kNN based estimators, \knn{} and \cov{}. Interestingly, in this setting, \knn{} without split performs globally better than coverage (note that the number $k$ of neighbors differs between these two approaches so that their estimates  are  not strictly identical even for extreme values $\lambda=0$ or $\lambda=\infty$).

\begin{figure}[htbp!]
    \centering
    \includegraphics[width=\mywidth,trim={0 0 0 20},clip]{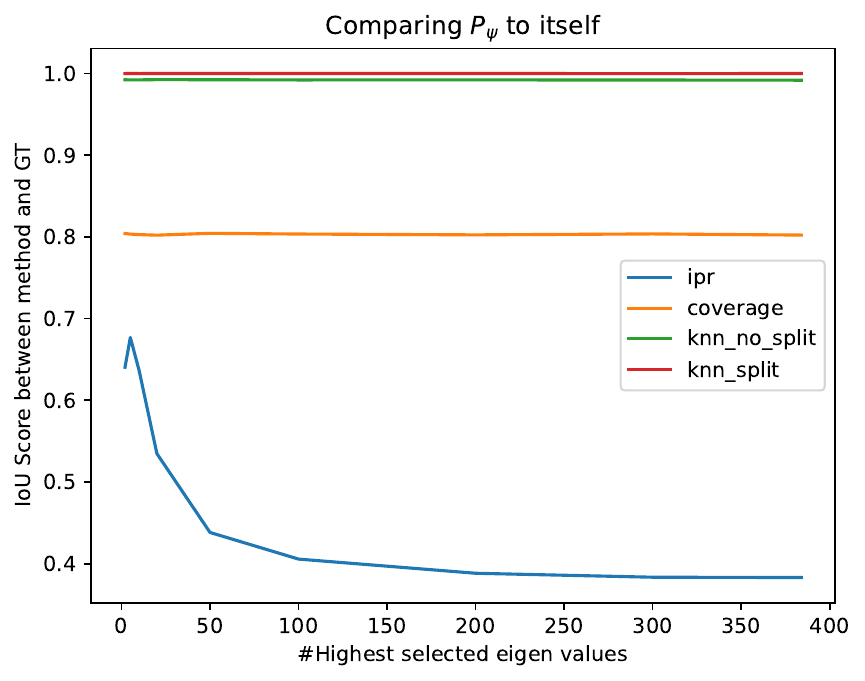}
    \caption{Evolution of IoU score when comparing a distribution to itself with various Precision and Recall metrics.}
    \label{fig:dino-self}
\end{figure}

\subsubsection{Truncated and projected Gaussian Mixture}
    \label{sub-sec:exp-DINO-gaussian-mixture}
In this experiment, we turn back to the GMM experiment described in Section~\ref{sec:exp_GMM}.
Now, we consider the case where $P^d \propto G^d_{\Psi=0.3} + G^d_{\Psi=0.5}$ and $Q^d \propto G^d_{\Psi=0.3} + G^d_{\Psi=0.9}$ are two Gaussian mixtures sharing a mode.
As the modes have all equal weight, the theoretical PR curve is approximately be a square going through $(0.5,0.5)$. This setting neither favors overestimation nor underestimation.

The results of this experiment are shown in Figure \ref{fig:dino-mixt-gaussian} where kNN based methods have very similar results and have an IoU score that tends very fast towards $0$. 
This time the comparison between \knn{} without split and \cov{} is reversed as the latter now has the upper hand by a small margin.
On the other hand, the \ipr{} method even though having a kink at the start gives a constant and reasonable IoU score with the ground truth curve.

\begin{figure}[!htpb]
    \centering
    \includegraphics[width=\mywidth,trim={0 0 0 20},clip]{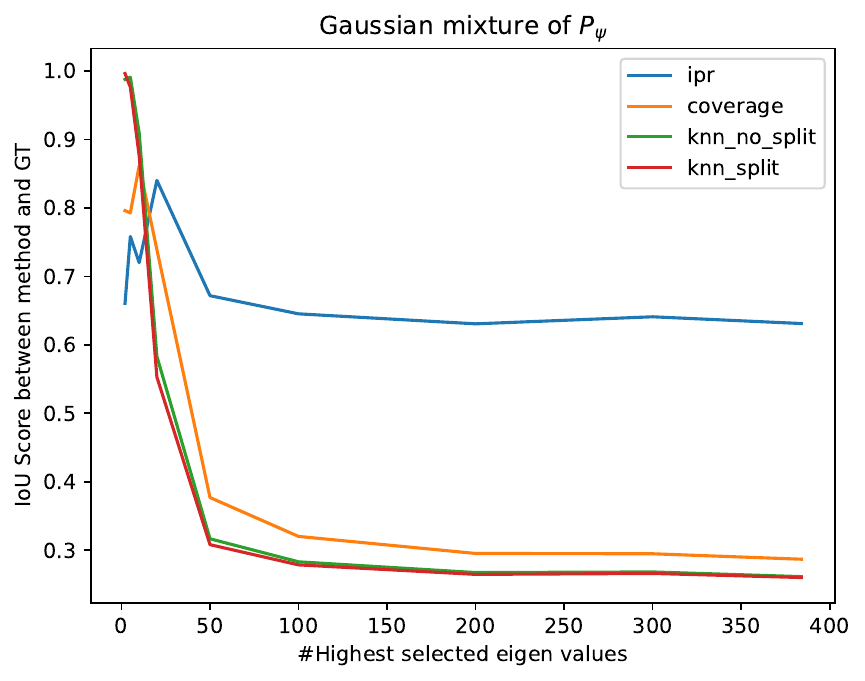}
    \caption{Evolution of IoU score between ground truth and various Precision and Recall metrics in the Gaussian mixture of StyleGAN truncated embeddings setting.}
    \label{fig:dino-mixt-gaussian}
\end{figure}

As the two distributions are Gaussian Mixtures, the extreme values should be equal to 1. Yet as the dimension increases, the two distributions become singular with respect to one another and the ground truth curve tends towards a square passing through $(0.5, 0.5)$.

\definecolor{color_KNN}{RGB}{255,136,32}

\subsection{Experiments with a conditional diffusion model}
\label{section:stable-diffusion-experiments}
This section is dedicated to experiments with a state-of-the-art generative flow-matching model.
We use Stable Diffusion \citep{esser_ScalingRectifiedFlow_2024}, a conditional text-to-image generative model based on rectified flows. 
More specifically, we use the public model \texttt{SD-3.5-medium} (with approximately three billion trainable parameters plus CLIP encoders for text conditioning) from the HuggingFace repository \citep{huggingface_StabilityaiStablediffusion35mediumHugging_2025}. 
During inference, for a given text prompt, an ordinary differential equation solver is used to synthesize iteratively a new image based on the velocity field predicted by the trained model conditioned on the text prompt.
To improve the image quality and prompt adherence, this diffusion process may be enhanced by extrapolating the estimated velocity using the classifier-free guidance technique of \cite{ho_ClassifierFreeDiffusionGuidance_2022}.
This effect is controlled by the so-called “guidance scale” factor.

In this context, we experiment two different configurations which allow to showcase the proposed metrics for a conditional generative model, when either changing the text conditioning (\S~\ref{sec:exp_prompt}) or the guidance scale (\S~\ref{sec:exp_guidance}), by
comparing samples from a generic vs specific guidance.
To do so, in both case, two datasets of 40k images are created at ${512 \times 512}$ pixels resolution with 25 inference steps. 
Similarly to Section~\ref{section:stylegan-exp}, we then use DINOv2 \citep{oquab_DINOv2LearningRobust_2024} to extract embeddings from the images.

\subsubsection{Generic vs. Specific Prompting}\label{sec:exp_prompt}

In this setting, we first generate a reference dataset with the generic prompt ``\texttt{butterfly}'' corresponding to distribution $P$, and a second one (distribution $Q$) from the more specific prompt ``\texttt{blue butterfly}''. 
This setting allows us to evaluate how restricting the prompt affects diversity (Recall). In order to affect the Precision, we also fixed the guidance scale to $7.5$ (instead of the default value of $5$) in order to amplify the adherence to the prompt, in this case the ‘‘blue’’ aspect of the generated image.

The experimental performance curves associated with this first experiment are presented in Figure~\ref{fig:stable_diffusion_varying_prompt}. 
As the prompt of distribution $Q$ describes a subset of distribution $P$, we expect in this setting a drop in Recall. Now, because of the high guidance during synthesis, we may also expect an impact on Precision.
All the curves behave accordingly, but in noticeable different ways.
In accordance with the previous experiments on StyleGAN, the \knn{} curve tends to overestimate extreme Precision and Recall (estimated with split datasets), but with a sharp transition along the axes, also suggesting a drop of Precision and Recall. The effect of drop in Precision and Recall is highlighted by the curves summary $F_8=0.44$ and $F_{1/8}=0.52$. %
In particular the \ipr{} curve shows a significant drop in extreme Recall to $\beta^{\ipr{}}_0 =  0.21$ and a moderate for Precision to $\alpha^{\ipr{}}_\infty = 0.65$ while, conversely, the \cov{} curve exhibits a more important drop of extreme Precision than Recall ($\alpha^{\cov{}}_\infty =0.14$ vs $\beta^{\cov{}}_0 = 0.42$).

\begin{figure}[htbp!]
    \centering
    \includegraphics{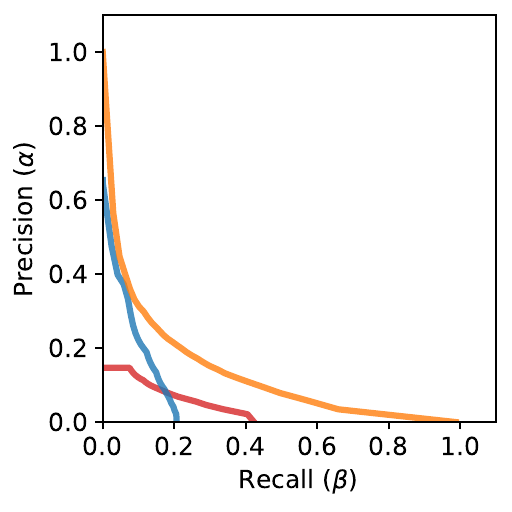}
    \caption{
        A comparison of the Precision and Recall metrics estimated with various classifiers:
        {\color{color_IPR} \bfseries --\ipr{}},
        {\color{color_KNN} \bfseries --\knn{}} and
        {\color{color_COV} \bfseries --\cov{}}.
    }
    \label{fig:stable_diffusion_varying_prompt}
\end{figure}

It is important to underline that, even in this controlled setting, we do not have access to the ground-truth to assess which method is the most relevant.
As a tentative to do so, it is interesting to visually inspect the images, some of which are showed in Figure~\ref{fig:stable_diffusion_butterfly_datasets}, with the following observations.
As expected, the variability of the ``\texttt{blue butterfly}'' images is much less in terms of color (even the blue colors of butterflies fall within a very narrow portion of the hue space compared to the ``\texttt{butterfly}'' dataset) and pose compared to those of the general class ``\texttt{butterfly}' %
thus affecting the Recall.
However, blue butterflies in the general class are mostly displayed on a white background, while samples from the restricted dataset seems to appear with a more diverse background (such as blue color gradient) %
which may also explain the simultaneous drop in Precision.
To illustrate this observation, Figure~\ref{fig:blue_butterflies_from_unrestricted} displayed the first 40 samples of blue butterflies from the unrestricted ``\texttt{butterfly}'' dataset, selected manually. 
Of course, this analysis strongly depends on the feature extractor ability to retain information about color, pose and background.

\begin{figure}[htbp]
    \centering
    \begin{subfigure}[b]{\textwidth}
        \centering
        \includegraphics[width=\textwidth]{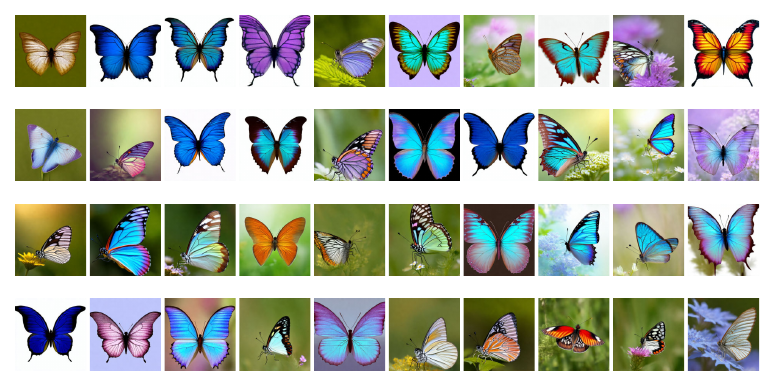}
        \caption{Samples from the ``\texttt{butterfly}'' dataset.}
        \label{subfig:butterfly}
    \end{subfigure}
    \vskip\baselineskip
    \begin{subfigure}[b]{\textwidth}
        \centering
        \includegraphics[width=\textwidth]{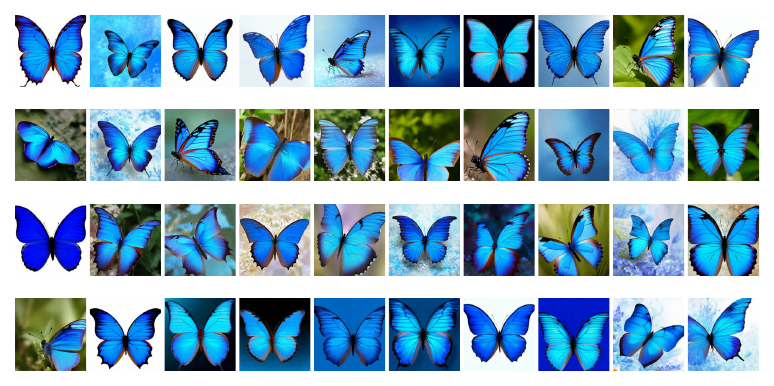}
        \caption{Samples from the ``\texttt{blue butterfly}'' dataset.}
        \label{subfig:blue_butterfly}
    \end{subfigure}
    \caption{Random images from each generated datasets, with guidance scale set to 7.5.}
    \label{fig:stable_diffusion_butterfly_datasets}
\end{figure}

\begin{figure}[htbp]
    \centering
    \includegraphics[width=\textwidth]{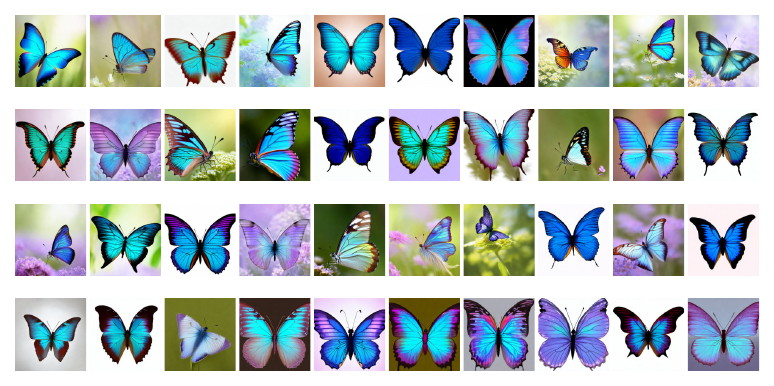}
    \caption{First 40 images of blue butterflies from the unrestricted ``\texttt{butterfly}'' dataset, selected by a human.}
    \label{fig:blue_butterflies_from_unrestricted}
\end{figure}

\subsubsection{Varying guidance scale}\label{sec:exp_guidance}

In this setting, we generate several datasets using different guidance scales and analyze their impact on the resulting Precision and Recall curves. The guidance scale parameter can be interpreted as an inverse temperature: a higher guidance scale reduces image variability while increasing adherence to the text prompt.

In the following experiment, 40k images are generated for each dataset, using guidance scales ranging from $1$ to $15$, with the same prompt : ``\texttt{a beautiful landscape with mountains and rivers}''.
Note that in order to make sure the samples are independent, a different random seed is used for the generation of each dataset.
Random samples are shown in Figure~\ref{fig:visualisation_images_guidance_scales}.

All datasets are then compared with the reference dataset corresponding to a standard guidance scale of $5$.
The summary metrics (AuC, $\alpha_\infty$, $\alpha$@$\varepsilon$, $F_b$, $\beta_0$, $\beta$@$\varepsilon$, $F_{1/b}$) 
of the corresponding PR curves are displayed in Figure~\ref{fig:stable_diffusion_varying_guidance_scale}.

\begin{figure}[htbp]
    \centering
    
    \raisebox{7mm}{\rotatebox[origin=c]{90}{\small $1.0$}}
    \hfill
    \includegraphics[width=.97\textwidth]{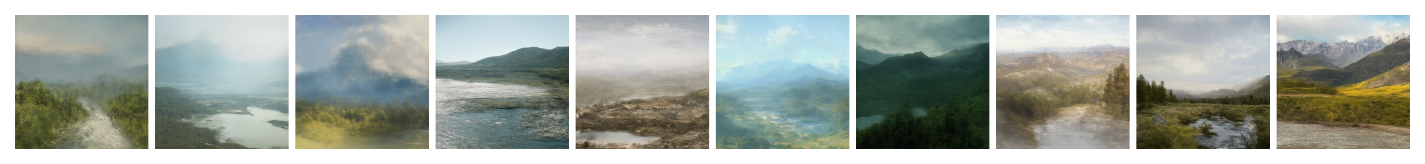}
    
    \raisebox{7mm}{\rotatebox[origin=c]{90}{\small $3.0$}}
    \hfill
    \includegraphics[width=.97\textwidth]{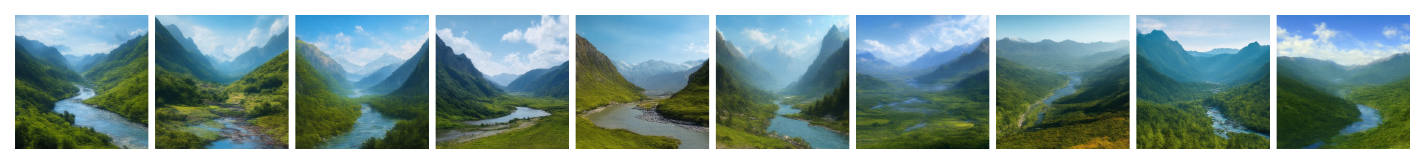}
    
    \raisebox{7mm}{\rotatebox[origin=c]{90}{\small $5.0$}}
    \hfill
    \includegraphics[width=.97\textwidth]{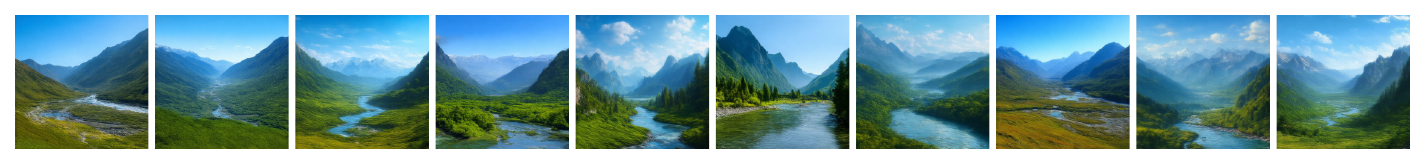}
    
    \raisebox{7mm}{\rotatebox[origin=c]{90}{\small $7.5$}}
    \hfill
    \includegraphics[width=.97\textwidth]{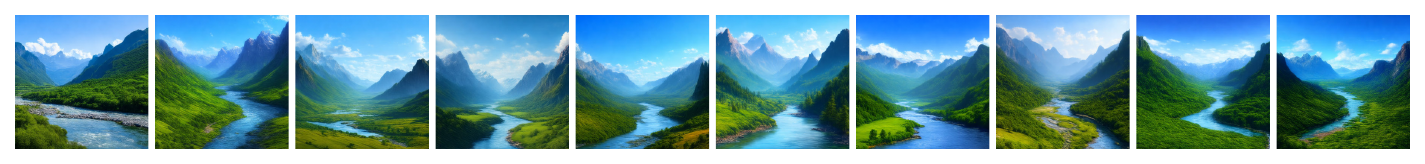}
    
    \raisebox{7mm}{\rotatebox[origin=c]{90}{\small $15$}}
    \hfill
    \includegraphics[width=.97\textwidth]{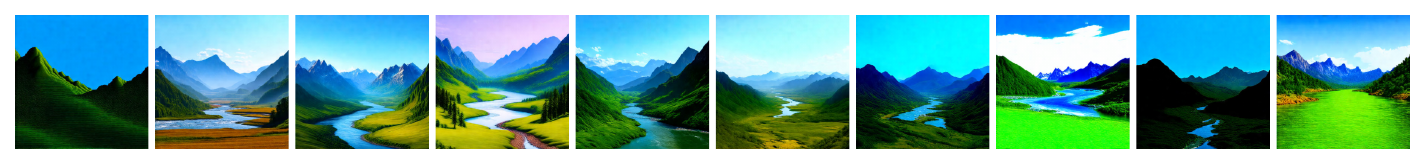}

    \caption{Generated images of ``\texttt{a beautiful landscape with mountains and rivers}'' for varying guidance scales (indicated on the left).}
    \label{fig:visualisation_images_guidance_scales}
\end{figure}

\begin{figure}[htb]
    \centering
        \begin{tabular}{c@{\hspace{2mm}}c@{\hspace{1mm}}c@{\hspace{1mm}}c}
        \multirow{2}{*}[25mm]{\rotatebox[origin=c]{90}{Precision}}
        &
        \includegraphics[width=.3\textwidth, trim={0 0 13cm 0},clip]{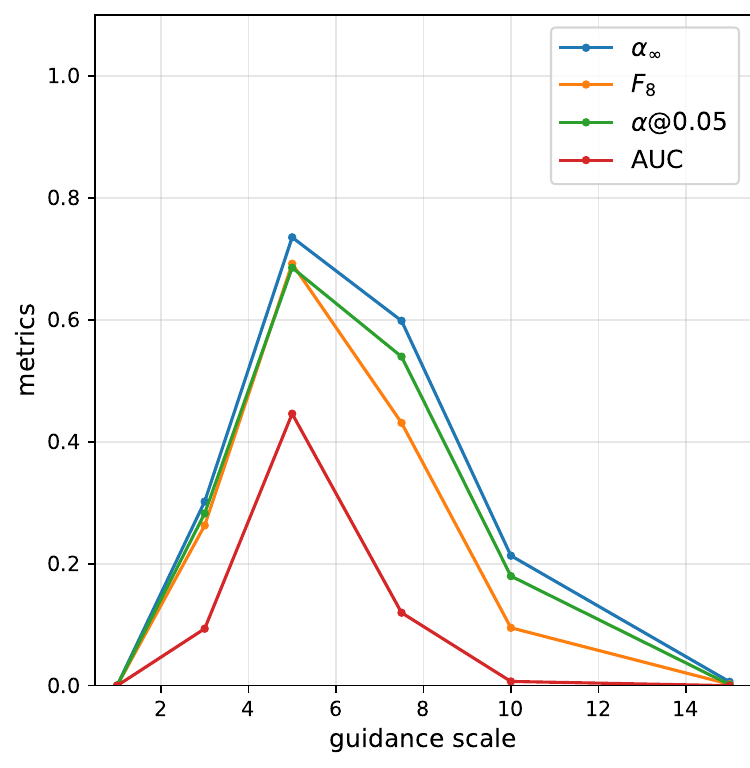}
        & 
        \includegraphics[width=.3\textwidth, trim={0 0 13cm 0},clip]{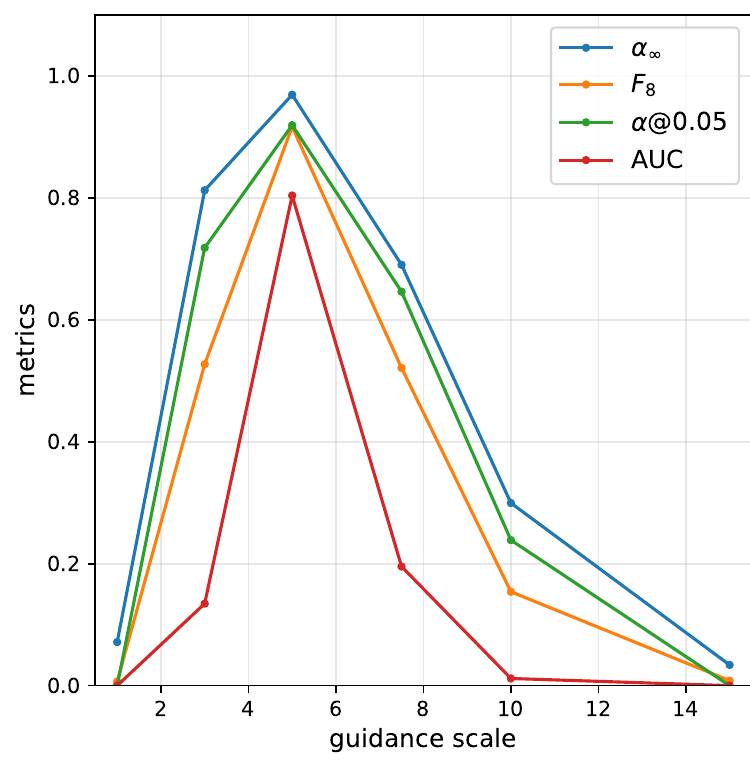}
        &
        \includegraphics[width=.3\textwidth, trim={0 0 13cm 0},clip]{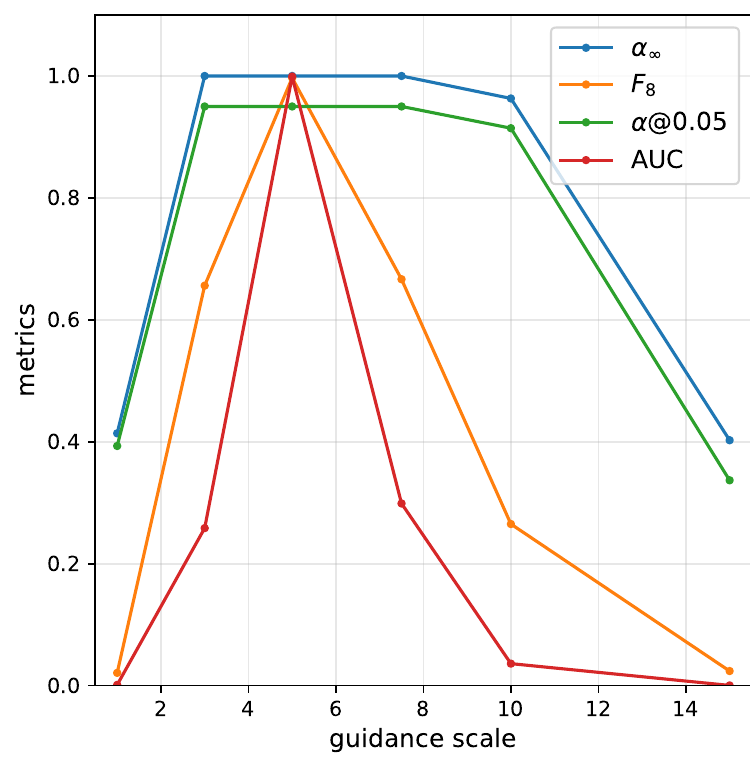}
        \\
        & \ipr{} & \cov{} & \knn{} 
        \\[2mm]
        \multirow{2}{*}[25mm]{\rotatebox[origin=c]{90}{Recall}}
        &
        \includegraphics[width=.3\textwidth, trim={13cm 0 0 0},clip]{figures_NEW/experiments/stable_diffusion/varying_scale_with_5/metrics_vs_scales_5.0_m_ipr_k_4_s_0.pdf}
        &
        \includegraphics[width=.3\textwidth, trim={13cm 0 0 0},clip]{figures_NEW/experiments/stable_diffusion/varying_scale_with_5/metrics_vs_scales_5.0_m_coverage_k_6_s_0.pdf}
        &
        \includegraphics[width=.3\textwidth, trim={13cm 0 0 0},clip]{figures_NEW/experiments/stable_diffusion/varying_scale_with_5/metrics_vs_scales_5.0_m_knn_k_200_s_0.5.pdf}
        \end{tabular}

    \caption{Experiment with generated images using the same conditioning prompt (``\texttt{a beautiful landscape with mountains and rivers}'') with different guidance scales ($1$, $2.5$, $5$, $7.5$, $10$, $15$). 
    Generated datasets are compared with a reference dataset using a fixed guidance scale of $5$. 
    Summarizing metrics for Precision and Recall curves w.r.t to the guidance scale. 
    The embeddings are based on DINOv2.}
    \label{fig:stable_diffusion_varying_guidance_scale}
\end{figure}

As expected, when spanning guidance factors, metrics of Precision and Recall increase until a maximum ($1$ in in theory) when they reach the reference guidance scale of 5, and then decrease.
When the guidance scale is low, the drop in Precision and Recall is due to a lower adherence to the prompt, resulting here in a larger diversity of scenery and a lower image quality.
For larger guidance scale, the decrease in Precision and Recall is likely due to an excessive match of the prompt (\emph{e.g.} resulting in images with a more vibrant color palette).

Interestingly, this experiment also highlights a limitation of the \ipr{} method first reported in \cite[Section~4.1.1]{naeem_ReliableFidelityDiversity_2020}: when comparing two identical Gaussian distributions, \ipr{} strongly underestimates the extreme Precision and Recall (\emph{i.e.} $\alpha^{\ipr{}}_\infty = 0.66$ and $\beta^{\ipr{}}_0 = 0.67$ instead of the perfect score $1$). 
In Section~\ref{sec:exp_shift}, while we focused the discussion on the core idea of evaluating how Precision and Recall curves are affected by two Gaussian distributions shifting away from one another, we also had observed how extreme Precision and Recall were severely underestimated with the \ipr{} method when the shift was absent.
This behavior was then observed in the StyleGAN Truncation experiment in section \ref{sub-sec:exp-DINO-self}, where Gaussian densities were fitted on the dataset of generated StyleGAN images, with an IoU of $0.66$ %
instead of 1.
Here in this setting, for a guidance scale of $5$, the two random datasets are directly drawn from the same non-Gaussian distribution, and yet we observe again the same phenomenon with $\alpha^{\ipr{}}_\infty = 0.73$ and $\beta^{\ipr{}}_0 = 0.74$ instead of $1$.
Accordingly to previous experiments, \cov{} tends to give a better estimate in this setting ($\alpha^{\cov{}}_\infty = 0.97$ and $\beta^{\cov{}}_0 = 0.97$), while \knn{} with data-splitting (which tends to overestimate extreme values) provides perfect estimates ($\alpha^{\knn{}}_\infty = 1 = \beta^{\knn{}}_0$).

\subsection{Summary of the experiments}

These experiments show in turn that different methods have an advantage over each other in some specific settings. Because of the lack of information in the general open setting, we prefer comparing metrics in controlled environments such as \ref{section:toy-experiments} or \ref{sec:hybrid-exp} in which we can estimate the ground truth and compare our estimators to it. It turns out that the extension of \ipr{}, even though not theoretically established, display some good empirical features in high dimension. 
While this observation diverges with the current literature, it should be understood as referring to the overall shape of the extended curve. In contrast, prior studies focused exclusively on the extreme value, due to inherent limitations in the PR metrics under investigation.

\section{Conclusion}
\label{sec:conclusion}
In this work, we have described a fresh perspective on the evaluation of generative models by alleviating the links between PR curves, the binary classification standpoint and total variation distance. 
We have proposed a new framework for evaluating generative models that describes an entire curve, rather than being limited to scalar metrics, as is often done in the recent literature.
This classification standpoint opened new practical perspectives (splitting, choice of hyperparameter $k$). 
Benefiting from the splitting setting, we statistically analyzed the estimators of PR we introduced. This analysis allowed to highlight the crucial role played by the dimension in the upper bounds of the estimation errors. 
We then showed how to bridge common scalar PR metrics from the literature to our framework, thus stemming an entire curve from initially 2 scalar metrics. This extension can be viewed as a dual approach to those common metrics as the extreme values of our extended curves match rigorously the scalar metrics.
Finally, we have studied the empirical behavior of the obtained variants in the light of several toy datasets and real world experiments with complex datasets. We consolidated the experiments by introducing a hybrid setting between the two previous kinds of dataset which allowed to display a complex case in which the ground truth metrics are known.

Our main messages are the following. First, computing non extreme PR values is crucial because of essential issues in the extreme values which are related to their sensitivity to the distribution tails.
Then, the curves themselves allow to describe more finely how the masses of the two distributions under comparison differ on their modes. This is useful in practice in order to tackle the case where a model generates data from the target support but with re-weighted masses.
We highlighted the negative impact of dimension in the quality of our estimators. Yet, as the extreme values of our extensions match some of the methods we extended, this approximation upper bound also holds for the initial metrics.
On the experimental side, there is no method clearly winning over the others as different behaviors emerge in varying scenarios (high dimension, several modes in the distributions ...)
If employing a data split is theoretically appealing and allowed our statistical analyses, its empirical impact is less marked since the negative bias resulting from the lack of split can sometimes advantageously compensate the positive bias caused by the restricted hypothesis class. However, this benefit is not consistent over all experiments.

This work offers several interesting perspectives and could benefit from some extensions.
A first step would be to focus on an asymptotic analysis for kNN methods in order to elucidate the convergence rate of the estimator.
In addition, the proposed KDE asymptotic analysis derived an upper bound for the estimation error which yielded a minimax upper bound. An extension of this analysis could be developed to establish a minimax lower bound.

\clearpage
\appendix

\section{Proof of \knn{} consistency}
\label{app:proof-consistency-knn}
\begin{proof}
     To establish the proof, we need only show that $R_\lambda(f^{kNN}_{\lambda})\to \alpha_\lambda$ as $k\to\infty$ and $\tfrac{k}{n}\to 0$.
     This will effectively imply both items in the theorem since $\alpha_\lambda$ is the associated Bayes risk \citep{simon_RevisitingPrecisionRecall_2019}. 
     To establish this limit, the first step is to show that for fixed $k$ then $\lim_{n\to\infty} R_\lambda(f^{kNN}_{\gamma})$ is equal to
     \begin{equation}
     \label{eq:rlambda}
         \begin{split}
     2\lambda \E[\eta(Z)\P\{\mathrm{Binom}(k,\eta(Z))<\tfrac{k}{\gamma+1}|Z\}] +\\
     2\E[(1-\eta(Z))\P\{\mathrm{Binom}(k,\eta(Z))>\tfrac{k}{\gamma+1}|Z\}] ,
         \end{split}
     \end{equation}
     where $Z=UX+(1-U)Y$ with $X\sim P$, $Y\sim Q$  and $U$ is a fair coin random variable so that $Z\sim \tfrac{P+Q}{2}$ and $\eta(Z):=\P(U=1|Z) = \tfrac{dP}{d(P+Q)}(Z)$ (respectively $1-\eta(Z):=\P(U=0|Z) = \tfrac{dQ}{d(P+Q)}(Z)$). The demonstration of Equation~\eqref{eq:rlambda} follows the same argument as in \citet{devroye2013probabilistic}[Theorem~5.2] (up to the occurrence of $\lambda$ and $\gamma$ weights) and is not repeated here for the sake of conciseness.

    Now, taking $\gamma=\lambda$ we want to show that the previous expression tends to $\alpha_\lambda$ which for the Recall equals $(\lambda P\wedge Q)(\measSpace)$ or expressed otherwise as $2\E[\lambda \eta(Z)\wedge(1-\eta(Z))]$.

     Equation~\eqref{eq:rlambda} can be reformulated as $\lim_{n\to\infty} R_\lambda(f^{kNN}_{\lambda}) = 2\E[\mu_\lambda(\eta(Z))]$ with
     \begin{equation}
         \label{eq:mu}
         \begin{split}
         \mu_\lambda(p) =& \lambda p \P\{\mathrm{Binom}(k,p) < \tfrac{k}{\lambda+1}\} 
                  + (1-p) \P\{\mathrm{Binom}(k,p)>\tfrac{k}{\lambda+1}\} .
         \end{split}
     \end{equation}
     So that it suffices to show that $\forall p\in[0,1]$, $\mu_\lambda(p)\to\lambda p \wedge(1-p)$.
     
     Let's proceed by cases, starting by considering $\lambda p < (1-p)$ which is also equivalent to $p<\tfrac{1}{\lambda+1}$. In that case we need to show that $2\mu_\lambda(p)\to \lambda p$.
     Denoting $q_\lambda(p)  = \P\{\mathrm{Binom}(k,p)>\tfrac{k}{\lambda+1}\}$, we have
     \begin{equation}
     \begin{split}
        \mu_\lambda(p) =&  \lambda p (1-q_\lambda(p)) + (1-p)q_\lambda(p) \\
                =& \lambda p + q_\lambda(p)(1-(\lambda+1)p) .
     \end{split}
     \end{equation}
     Using Hoeffding's inequality (i.e. $\forall t>0, \P\{\mathrm{Binom}(k,p)-kp> t\}\leq \exp (-2k t^2)$  and obtain
     \begin{equation}
         \begin{split}
     q_\lambda(p) =& \P\{\mathrm{Binom}(k,p)-kp> k (\tfrac{1}{\lambda+1}-p)\} \\
     \leq& \exp \left(-2k \left(\tfrac{1}{\lambda+1}-p\right)^{2}\right) .
         \end{split}
     \end{equation}
     Note that the assumption $p < \tfrac 1{\lambda+1}$ is crucial to apply Hoeffding's inequality (because $t$ needs to be positive).
     The right hand side converges to $0$ as $k\to\infty$ because by assumption $p\neq \tfrac{1}{\lambda+1}$.

     The case where $\lambda p > (1-p)$ (or $p>\tfrac{1}{\lambda+1}$) is similar and is left to the reader. In that case, we obtain $\mu_\lambda(p)\to (1-p)$.
     There remains the case of equality, that is $\lambda p = 1-p = \tfrac{1}{\lambda +1}$. In that case, even without taking the limit, one can check that $\mu_\lambda (p) = \lambda p$, which concludes the proof.
\end{proof}

\section{Proof of KDE results}

\subsection{Subgaussian Random Variables}
\begin{proposition}[Properties on subgaussian random variables]
\label{prop:subgauss-props}
\ 
\begin{enumerate}
    \item\label{itm:subgauss-sum} If $X$ and $Y$ are independent random vectors in $\R^d$, then $\Var_G(X+Y)\leq\Var_G(X)+\Var_G(Y)$.
    \item\label{itm:subgauss-ball} If $X$ is uniformly distributed on the unit ball of $\R^d$, then $\Var_G(X)\leq  \frac \eta{d}$ where $\eta>0$ is a universal constant independent from $d$ 
\end{enumerate}
\end{proposition}
\begin{proof}
The first item is easily obtained from the standard factorization of the exponential combined with the independence assumption.
For the second item, see \citep[Theorem~3.4.5]{vershynin2018high} which establishes the result for the uniform distribution on the sphere and henceforth on the ball as well.
\end{proof}

\begin{lemma}
\label{lemma:gauss-rep}
Let $X\in\R^d$ be a $K$-subgaussian and centered random vector, and $W\sim \mathcal N(0,I_d)$, then
$$\E[\exp(a\Vert X\Vert^2)] \leq \E[\exp(K^2a\Vert W\Vert^2)] = \frac 1{\left(1-2 K^2a\right)^{\frac d2}}$$
(the rightmost identity being valid as long as $1-2 K^2a> 0$ i.e. $a<\frac 1{2 K^2}$).

In particular, for $a=\frac1{4 K^2}$ one obtains 
$$
\E[\exp(a\Vert X\Vert^2)] \leq 2^{\frac d2} .
$$
\end{lemma}
\begin{proof}
NB: this lemma is a sharper version\footnote{In their work, for $a=\frac1{4 K^2}$ Goldfeld et al obtain an upper bound of the form $\left(\exp\left(\frac 34\right)\right)^{\frac d2}$ which is slightly loose in comparison to our bound since $\exp\left(\frac 34\right)\approx 2.117$. Actually our bound is tight, because it is an identity for isotropic Gaussians.} of \citet[Equation~7]{goldfeld_ConvergenceSmoothedEmpirical_2020} with a direct and basic proof.
The first part is referred as the ``Gaussian replacement technique'' in \citet{vershynin2018high}).
First notice that $\forall x\in\R^d$, $\E[\exp(W^T x)] = \exp\left(\frac{\Vert x\Vert^2}2\right)$, and apply it to $x:=\sqrt{2a}X$ (and condition on $X$ so that it can be considered constant):
\begin{align*}
    \E[\exp(a\Vert X\Vert^2)]=&  \E[\E[\exp(W^T \sqrt{2a}X)|X]] = \E[\exp(W^T \sqrt{2a}X)]\\
    =& \E[\E[\exp(W^T \sqrt{2a}X)|W]]\leq \E[\exp(\frac 12 K^2 2a\Vert W\Vert^2)] . \qquad\text{(subgaussianity of $X$)}
\end{align*}
This ends the proof of the first part (i.e. the inequality).

Let us turn to the second part (i.e. the identity),
\begin{align*}
    \E[\exp(aK^2\Vert W\Vert^2)]=& \prod_{i=1}^d \underbrace{\E[\exp(a K^2 W_i^2)]}_{=\frac 1{\sqrt{2\pi}}\int \exp(-\frac{1-2a K^2}2 w^2)dw}\\
    =& \frac 1{\left(1-2aK^2\right)^{\frac d2}}\qquad\text{if $a<\frac 1{2 K^2}$} .
\end{align*}

\end{proof}

\subsection{TV upper bound on bias}
\label{proof:tv-upperbound-bias}
\begin{proof}
We have 
\begin{align*}
    p(x) - (p\ast k_\sigma)(x)&=\int_{y} \left(p(x)k_\sigma(x-y)-p(y)k_\sigma(x-y)\right)dy\\
    &=\int_{y} (p(x)-p(y))\tfrac{k(\frac{x-y}{\sigma})}{\sigma^d}dy\\
    &=\int_{z}  (p(x)-p(x-\sigma z))k(z)dz\\
    &=\int_{z} k(z) \int_{t=0}^1 \frac{d}{dt}p(x-\sigma zt) dz dt\\
    &=-\sigma \int_{z} k(z) \int_{t=0}^1 \nabla p(x-\sigma zt)^T z dt dz .
\end{align*}

Hence,
\begin{align*}
    2\TV{P-P\ast k_\sigma}=\int | p(x) - (p\ast k)(x)|dx & \leq \sigma \int_{z} |k(z)| \int_{t=0}^1  \underbrace{\int |\nabla p(x-\sigma zt)^T z|dx}_{\leq \int_{x \in \R^d} \|\nabla p(x-\sigma zt)\| \Vert z\Vert dx = \Vert z\Vert \int \| \nabla p(v)|\|dv}dt dz \\
    & \leq \sigma \int \Vert z \Vert  |k(z)| dz \int \|\nabla p(v)\| dv \\
    & \leq \sigma R .
\end{align*}
    Last, we can notice that for the constant kernel that is both non-negative and supported on the unit ball $\int \Vert z\Vert  |k(z)| dz \leq \int k(z) dz = 1$.
\end{proof}

\subsection{Proof of Lemma~\ref{prop:risk_deviation}}
\label{sec:proof-smoothed-TV-bound-highSNR}
\begin{proof} We have
\begin{align*}
    \E [2\TV{\bar P \ast k_\sigma - P\ast k_\sigma}]
    =& \int \E [ |\bar P\ast k_\sigma(z)-P\ast  k_\sigma(z)|]dz\\
    =& \E \left[ \frac{|\bar P\ast k_\sigma(Z)-P\ast  k_\sigma(Z)|}{f_a(Z)}\right]\\
    \leq& \E \left[ \left(\frac{\bar P\ast k_\sigma(Z)-P\ast k_\sigma(Z)}{f_a(Z)}\right)^2\right]^{\frac 12} ,
\end{align*}
where $Z\sim f_a(z)dz$ is an artificially introduced random vector in $\R^d$ independent of the samples and $f_a$ is a density that can be chosen arbitrarily as long as it is never null. 
In our case, we use
$$
f_a(z):=\mathcal N(z; 0,\frac 1{2a}I)=\left(\frac a\pi\right)^{\frac d2}\exp\left(-a\Vert z\Vert^2\right) .
$$
 Note that in the last inequality is merely the Cauchy-Schwarz one.

Now using the law of total expectation
\begin{align*}
     \E \left[ \left(\frac{\bar P\ast k_\sigma(Z)-P\ast k_\sigma(Z)}{f_a(Z)}\right)^2\right]=& \E \left[ \E\left[\left(\frac{\bar P\ast k_\sigma(Z)-P\ast  k_\sigma(Z)}{f_a(Z)}\right)^2\middle|Z\right]\right]\\
     =& \E\left[ \underbrace{\Var\left[\frac{\bar P\ast k_\sigma(Z)}{f_a(Z)}\middle|Z\right]}_{:=(*)} \right] .
\end{align*}
Since the samples $X_1,\cdots,X_N$ are independent and $Z$ is independent from them, the $X_k$'s remain independent given $Z$ is observed. Therefore, we have
\begin{align*}
    (*)=&\frac 1N \Var\left[\frac{\delta_{X_1}\ast k_\sigma(Z)}{f_a(Z)}\middle| Z\right]= \frac 1{N} \frac 1{f_a(Z)^2} \Var\left[k_\sigma(Z-X_1)\middle| Z\right] ,
\end{align*}
which, plugged into the total second moment above, shows that
\begin{align*}
       \E [2\TV{\bar P \ast k_\sigma - P\ast k_\sigma}]^2 \leq & \frac 1{N} \int \frac 1{f_a(z)} \Var\left[k_\sigma(z-X_1)\right] dz\\
      \leq& \frac 1{N} \int \frac 1{f_a(z)} \E\left[k_\sigma^2(z-X_1)\right] dz .
\end{align*}
In our case, $k_\sigma(y)= \frac 1{\text{vol}_{B_1}(d)\sigma^d}\1_{\Vert y\Vert\leq \sigma}$ and therefore $k_\sigma^2(y)=\frac 1{\text{vol}_{B_1}(d)\sigma^d} k_\sigma(y)$. 
Thus, 
\begin{align*}
       \E [2\TV{\bar P \ast k_\sigma - P\ast k_\sigma}]^2 \leq & \frac 1{N \text{vol}_{B_1}(d) \sigma^d} \E\left[\int \frac1{f_a(z)} k_\sigma(z-X_1)dz\right]\\
      =& \frac 1{N \text{vol}_{B_1}(d) \sigma^d} \E\left[\frac 1{f_a(Y+X_1)}\right] .
\end{align*}
where $Y\sim k_\sigma(y)dy$ is independent from $X_1$. 

Given that the constant kernel corresponds to a uniform distribution on a ball of radius $\sigma$, Prop~\ref{prop:subgauss-props}.\ref{itm:subgauss-ball} implies it has a  subgaussian variance smaller than $\frac{\eta\sigma^2}d$. 
In turn, Prop~\ref{prop:subgauss-props}.\ref{itm:subgauss-sum} implies that 
$S=Y+X_1$ is also sub-gaussian with constant $K_S$ following $K_S^2 = K^2+\frac{\eta}d \sigma^2$ where the constant $\eta$ does not depend on the ambient dimension $d$.

Then we will use Lemma~\ref{lemma:gauss-rep}  with $a:=\frac 1{4 K_S^2}$ (which turns out to be the optimal choice here since it minimizes $\frac{1}{a(1-2K_S^2 a)}$):

\begin{align*}
       \E [2\TV{\bar P \ast k_\sigma - P\ast k_\sigma}]^2 \leq & \frac {\left(\frac \pi a\right)^{\frac d2}}{N \text{vol}_{B_1}(d) \sigma^d} \E\left[\exp(a\Vert S\Vert^2)\right]&\\
       \leq& {\left(\frac \pi {a(1-2K_S^2 a)}\right)^{\frac d2}}\frac 1{N \text{vol}_{B_1}(d) \sigma^d}&\text{(Lemma~\ref{lemma:gauss-rep})}\\
       \leq& \left(8K_S^2\pi\right)^{\frac d2} \frac 1{N \text{vol}_{B_1}(d) \sigma^d} .
\end{align*}
Last, we take the square root of this inequality and use both $\text{vol}_{B_1}(d)=\frac{\pi^{\frac d2}}{\Gamma\left(\frac d2 +1\right)}$ and $K_S^2=K^2+\frac \eta{d} \sigma^2$ to obtain

\begin{align*}
       \E [2\TV{\bar P \ast k_\sigma - P\ast k_\sigma}]\leq& \sqrt{\Gamma\left(\frac d2+1\right)} 8^{\frac d4}\left(\frac
       \eta d +\frac{K^2}{\sigma^2}\right)^{\frac d4}\frac 1{\sqrt N} .
\end{align*}
\end{proof}

\subsection{Proof KDE bias term}
\label{sec:proof-bias-kde}
\begin{proof}
Before starting the proper proof, let us mention that we will somewhat involve a relative bias between the classifier KDE estimator $\hat \alpha^{\text{KDE}}_\lambda$ and the plug-in KDE estimator $\hat \alpha^{\text{TV}}_\lambda$. To bound this relative bias we will need to set the latter in a form similar to the former.
Recall that from the definition \eqref{eq:hatPRC}, using empirical false positive (resp. negative) rate $\overline \fpr$ (resp. $\overline \fnr$), we have
\begin{equation*}
\begin{aligned}
    \hat \alpha^{\text{KDE}}_\lambda
    &= \min_{\gamma} \lambda\overline\fpr(\hat f^{\text{KDE}}_\gamma) +\overline\fnr(\hat f^{\text{KDE}}_\gamma) ,
\end{aligned}
\end{equation*}
where the empirical likelihood-ratio classifier parameterized by $\gamma$ is estimated from training data $\{\X, \Y\}$ using
\begin{equation*}
    \hat f^{\text{KDE}}_\gamma(z)= \1_{ \gamma \hat P(z) \geq \hat Q(z)} .
\end{equation*}
Now, we recall that the empirical FNR and FPR are estimated on the evaluation data $\{\X', \Y'\}$
\begin{equation*}
    \overline \fpr(f) = \frac{1}{N}
    \sum_{x \in \X'} (1-f(x))
    \qquad \text{ and } \qquad
    \overline \fnr(f) = \frac{1}{N}
    \sum_{y \in \Y'}
    f(y).
\end{equation*}
On the other hand, the plug-in estimator is 
\begin{align*}
    \hat \alpha^{\text{TV}}_\lambda
    &= \tfrac{1}{2}(\lambda + 1) - \TV{\lambda \hat P-\hat Q}
    =\lambda  \widehat \fpr (\hat f^{\text{KDE}}_\lambda) + \widehat \fnr (\hat f^{\text{KDE}}_\lambda) .
\end{align*}
where the estimated FNR and FPR depend on the training data $\{\X, \Y\}$ are defined as
\begin{equation*}
    \widehat \fpr(f) = \int_{x} (1-f(x)) d\hat P(x)
    \qquad \text{ and } \qquad
    \widehat \fnr(f) = \int_{y} f(y) d\hat Q(y) .
\end{equation*}
Let us now proceed with the proof. First, due to the splitting between train/test samples, the bias of the considered estimator is positive. Indeed,
\begin{align*}
\E \,\hat\alpha^{\text{KDE}}_\lambda - \alpha_\lambda=& \E \left[\E \left[ \hat \alpha^{\text{KDE}}_\lambda | \X,\Y \right] - \alpha_\lambda\right]\\
=&\E \left[\E \left[ \lambda\overline{\fpr}(f_{\gamma*}^{\text{KDE}}) + \overline{\fnr}(f_{\gamma*}^{\text{KDE}}) | \X,\Y \right] - \alpha_\lambda\right]\\
=&\E \left[\underbrace{\lambda\fpr(f_{\gamma*}^{\text{KDE}}) + \fnr(f_{\gamma*}^{\text{KDE}}) - \alpha_\lambda}_{\geq 0}\right]\geq 0 ,
\end{align*}
where we have used that for any classifier $f$, conditionally to the training set, the expectations of $\overline{\fpr}(f)$ and $\overline{\fpr}(f)$ are just the population $\fpr(f)$ and $\fnr(f)$, as well as the fact that the population excess-risk is non-negative.

Therefore
\begin{align*}%
    \left| \E \,\hat\alpha^{\text{KDE}}_\lambda - \alpha_\lambda  \right|
    &= \E \,\hat\alpha^{\text{KDE}}_\lambda - \alpha_\lambda  
    \\
    &= \E \left[  \E \left[ \hat \alpha^{\text{KDE}}_\lambda | \X,\Y \right] - \alpha_\lambda \right]
    \\
    & = \E \left[ \E[\hat \alpha^{\text{KDE}}_\lambda| \mathcal X,  \mathcal Y] -   \hat \alpha^{\text{TV}}_\lambda +  \hat \alpha^{\text{TV}}_\lambda - \alpha_\lambda \right]
    \quad \text{ using $\hat \alpha^{\text{TV}}_\lambda$ defined in \eqref{def:alpha-TV}}
    \\
    &\leq \underbrace{\E  \left[\E[\hat \alpha^{\text{KDE}}_\lambda| \mathcal X,  \mathcal Y] -  \hat \alpha^{\text{TV}}_\lambda \right]}_{\text{relative bias w.r.t to plug-in}}
    + 
    \underbrace{ \E \left| \hat \alpha^{\text{TV}}_\lambda - \alpha_\lambda \right|}_{\text{TV plug-in risk}}
    .
\end{align*}

Regarding the second term, we have already seen in Proposition~\ref{prop:plugin-trivial-bound}   that 
\begin{align*}%
     \E \left| \hat \alpha^{\text{TV}}_\lambda - \alpha_\lambda \right| 
    & \le \lambda  \E \TV{\hat P-P} +  \E \TV{\hat Q - Q} 
    .
\end{align*}

It remains to show that the relative bias is also subject to the same upper bound.
Using once again the fact that the expectation of the empirical FPR (resp. FNR) on an independent evaluation set is the FPR (resp. FNR), we can upper bound the expected estimated Precision
\begin{align*}
    \E \left[\hat \alpha^{\text{KDE}}_\lambda \,|\, \mathcal X,  \mathcal Y \right]
    &= \E \left[ \min_\gamma \lambda  \overline \fpr (\hat f^{\text{KDE}}_\gamma) + \overline \fnr (\hat f^{\text{KDE}}_\gamma) \,|\, \mathcal X,  \mathcal Y \right]
    \\
    & \le \lambda \, \E \left[ \overline \fpr (\hat f^{\text{KDE}}_\lambda) \,|\, \mathcal X,  \mathcal Y \right]
    + \E  \left[ \overline \fnr (\hat f^{\text{KDE}}_\lambda) \,|\, \mathcal X,  \mathcal Y \right]
    \quad \text{(picking $\gamma = \lambda$ instead of $\gamma^*$)}
    \\
    &= \lambda \fpr \, (\hat f^{\text{KDE}}_\lambda) + \fnr \, (\hat f^{\text{KDE}}_\lambda) .
\end{align*}
So we can bound the classifier relative bias using
\begin{align*}
    \E \left[\hat \alpha^{\text{KDE}}_\lambda| \mathcal X,  \mathcal Y] -  \hat \alpha^{\text{TV}}_\lambda  \right ]
    & \le  
        \lambda \int_{} (1 - \hat f^{\text{KDE}}_\lambda ) d(P-\hat P)
        + \int_{} \hat f^{\text{KDE}}_\lambda d(Q-\hat Q)
    \\
    & = 
        \lambda \int_{} \hat f^{\text{KDE}}_\lambda  d(P-\hat P)
        + \int_{} \hat f^{\text{KDE}}_\lambda  d(Q-\hat Q)
    \\
    & = \lambda 
        \int_{} \1_{ \lambda \hat P \geq \hat Q}  d(P-\hat P)
        + 
         \int_{} \1_{ \lambda \hat P \geq \hat Q} d(Q-\hat Q) 
    \\
    & \le \lambda 
        \TV{P-\hat P}
        + 
        \TV{ Q-\hat Q } .
\end{align*}
where we have used the definition of $\TV{P-\hat P} :=\sup_{A} P(A)-\hat P(A) = \sup_{A} \int \1_A d(P-\hat P)$.

\end{proof}

\section{Proof of Theorem~\ref{thm:cosupp}: co-support}
\label{proof:cosupp}
\begin{proof}
First, if $A, A'$ are two co-supports. Then $Q(A)=Q(A\cap A') = Q(A')$. Indeed, if $Q(A)> Q(A\cap A')$ then letting $B=A\setminus A'\subset A$ one obtains $Q(B)>0$. Yet $B\subset A'^c$ so that $P(B)< P(A'^c)=0$ yielding a contradiction.

Second, let us exhibit a co-support $C$ that verifies $Q(C)=\alpha_\infty$.
In \citet{simon_RevisitingPrecisionRecall_2019}, it is shown that for $\lambda=+\infty$, Equation~\ref{eq:dualPR} can be restated as $\alpha_\infty = \min_{A\text{ s.t. } P(A^c)=0} Q(A)$. 
Let $A^*$ one of the minimizers.
Without further care on the minimizer, $A^*$ should be merely a particular support of $P$ but could still not be a co-support. 
We therefore need to filter out any part of the space that charges $P$ but not $Q$ (which will make it a co-support without affecting its $Q$-mass).
To do so, we consider $C = A^*\setminus \cap_{\lambda>0} \{\lambda P > Q\}$. 
First, the monotone convergence theorem implies that $Q(C)=Q(A^*)=\alpha_\infty$. 

It remains to show that $C$ is indeed a co-support.
Notice that $P\wedge Q(C^c)=(P\wedge Q)({A^*}^c\bigcup \cap_{\lambda>0} \{\lambda P > Q\})\leq P({A^*}^c) + Q(\cap_{\lambda>0} \{\lambda P > Q\})=0$ (because of the constraint on $A^*$ for the first summand, and by the monotone convergence theorem again for the other summand).

Besides let $B\subset C \subset (\cap_{\lambda>0} \{\lambda P > Q\})^c=\cup_{\lambda>0} \{\lambda P \leq Q\}$ so that $Q(B)=0 \implies P(B)=0$.
Conversely, if $P(B)=0$ let us show that $Q(B)=0$. 
To do so let us reason by contradiction, by assuming that $Q(B)>0$. Then $A=A^*\setminus B$ verifies the constraint $P(A)=0$ and $Q(A)=Q(A^*)-Q(B)$ (because $B\subset C\subset A^*$). Then $Q(A)<Q(A^*)$ would contradict the definition of $A^*$.
\end{proof}

\section{Additional experimental results}\label{app:add_exp}

\subsection{Gaussians shifts}
\label{ap:gaussian_shifts}
Table.~\ref{tab:mean_iou_shift_k4} complements Section~\ref{sec:exp_shift} and Figure \ref{fig:shift-gauss} by comparing various estimated PR curves with respect to the ground truth, using average IoU scores.
While this observation contrasts with findings reported in the literature, it should be understood as referring to the overall shape of the curve. In contrast, prior studies focused exclusively on the extremes, due to inherent limitations in the PR metrics under investigation.

\begin{table}[htb]
\small
    \centering
    \begin{tabular}{c c c||c|c|c|c}
&&shift $\mu$ %
   & {\color{color_IPR}  \ipr{}}
   & {\color{color_KNN}  \knn{}}
   & {\color{color_PARZ}  \parzen{}} 
   & {\color{color_COV}  \cov{}}
\\
\hline\hline
\multirow{8}{*}{\rotatebox[origin=c]{90}{with 50 \% split}}&\multirow{4}{*}{\rotatebox[origin=c]{90}{$k=4$}}
&$0.12$  &  0.69  & 0.71  & 0.72  & 0.73  \\
&&0.21  &  0.42  & 0.49  & 0.49  & 0.55  \\
&&0.29  &  0.24  & 0.38  & 0.34  & 0.48  \\
&&0.38  &  0.13  & 0.33  & 0.24  & 0.48  \\
\cline{2-7}
&\multirow{4}{*}{\rotatebox[origin=c]{90}{$k=\sqrt n$}}
&0.12  &  0.81 & 0.87 & 0.84 & 0.92 \\
&&0.21  &  0.69 & 0.84 & 0.78 & 0.90 \\
&&0.29  &  0.65 & 0.84 & 0.75 & 0.90 \\
&&0.38  &  0.63 & 0.84 & 0.75 & 0.93 \\
\cline{1-7}
\multirow{8}{*}{\rotatebox[origin=c]{90}{without split}}&\multirow{4}{*}{\rotatebox[origin=c]{90}{$k=4$}}
&0.12  &  0.43 & 0.7 & 0.62 & 0.76 \\
&&0.21  &  0.55 & 0.81 & 0.68 & 0.84 \\
&&0.29  &  0.62 & 0.79 & 0.68 & 0.77 \\
&&0.38  &  0.55 & 0.61 & 0.62 & 0.63 \\
\cline{2-7}
&\multirow{4}{*}{\rotatebox[origin=c]{90}{$k=\sqrt n$}}
&0.12  &  0.91 & 0.93 & 0.94 & 0.96 \\
&&0.21  &  0.88 & 0.93 & 0.92 & 0.97 \\
&&0.29  &  0.84 & 0.92 & 0.90 & 0.95 \\
&&0.38  &  0.83 & 0.91 & 0.90 & 0.96
\end{tabular}
    \caption{\textbf{Mean IoU scores} for shifted Gaussians.
    Standard deviations are $< 10^{-2}$ with $n=10$K.}
    \label{tab:mean_iou_shift_k4}
\end{table}

\subsection{Gaussian mixture comparison}
\label{ap:gaussian_mixture}
Figure~\ref{fig:GMM-dim64_suite} complements Section~\ref{sec:exp_GMM} and Figure \ref{fig:GMM-dim64} with additional curves for different setting (w/ and w/o splitting, $k=4$ or $k=100$).

\setlength\mywidthhere{0.4\textwidth}
\begin{figure}[htb]
    \centering
    \begin{tabular}{cc}
        \multicolumn{2}{c}{\emph{50\% split validation/train}} \\
        \includegraphics[width=\mywidthhere]{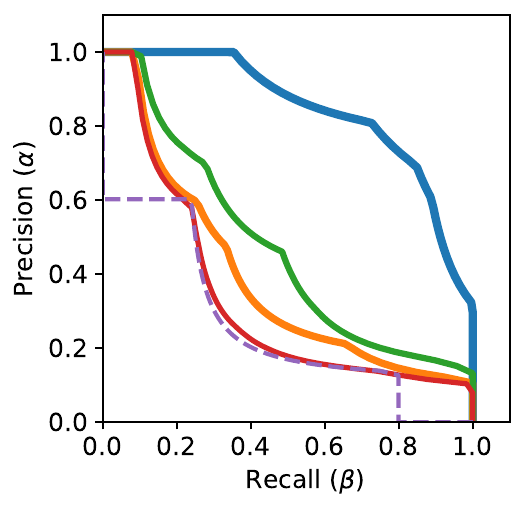} &
        \includegraphics[width=\mywidthhere]{figures_NEW/gmm/gmm_split_k_n.pdf} \\
        $k=4$ & $k = \sqrt n$ \\
        \multicolumn{2}{c}{\emph{without split}} \\
        \includegraphics[width=\mywidthhere]{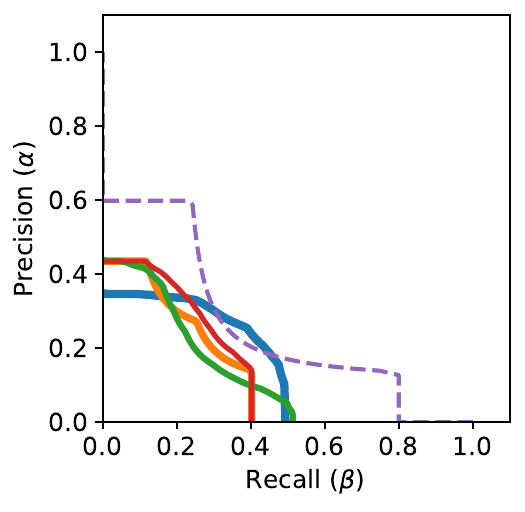} &
        \includegraphics[width=\mywidthhere]{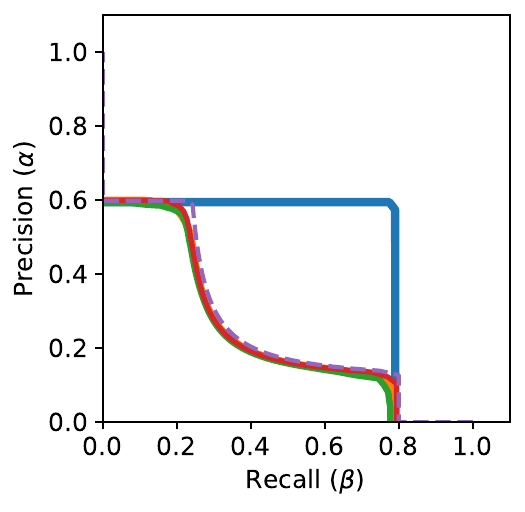} \\
        $k=4$ & $k = \sqrt n$
    \end{tabular}
    \caption{\textbf{Comparing two Gaussian mixtures}. This figure complements Figure~\ref{fig:GMM-dim64}.
    The ground truth PR curve (
    {\color{color_GT} \bfseries -~-\textsc{GT}})
    is compared to empirical estimates from various NN-classifiers:
    {\color{color_IPR} \bfseries --\ipr{}},
    {\color{color_KNN} \bfseries --\knn{}},
    {\color{color_PARZ} \bfseries --\parzen{}}, and 
    {\color{color_COV} \bfseries --\cov{}}.
    Here $P$ and $Q$ are two GMMs sharing the same modes (centered at $\mu_k$):
    $P = \sum_{\ell} p_\ell \mathcal{N}(\mu_\ell \mathbf 1_{d}, \mathbb{I}_{d})$ 
    and 
    $ Q =  \sum_{\ell} q_\ell \mathcal{N}(\mu_k \mathbf 1_{d}, \mathbb{I}_{d})$ with $d=64$ dimensions and 
    $\mu_\ell \in \{0, -5, 3, 5\}$.
    However, $P$ and $Q$ have different weights ($p_\ell$ and $q_\ell$):
    $p_\ell  \in \{0.3, 0.2, 0.5, 0\}$ vs
    $q_\ell  \in \{0, 0.5, 0.2, 0.3\}$.
    $n=1$k points are sampled and split in half between validation and train, and $k=\sqrt n$. 
    }
    \label{fig:GMM-dim64_suite}
\end{figure}

\end{document}